\documentclass{article}
\usepackage[preprint]{neurips_2026}

\usepackage[utf8]{inputenc} % allow utf-8 input
\usepackage[T1]{fontenc}    % use 8-bit T1 fonts
\usepackage{hyperref}       % hyperlinks
\usepackage{url}            % simple URL typesetting
\usepackage{booktabs}       % professional-quality tables
\usepackage{amsfonts}       % blackboard math symbols
\usepackage{nicefrac}       % compact symbols for 1/2, etc.
\usepackage{microtype}      % microtypography
\usepackage{xcolor}         % colors
\usepackage{graphicx}
\usepackage{amsmath,amssymb,amsthm}
\DeclareMathOperator{\Var}{Var}
\theoremstyle{remark}
\newtheorem{remark}{Remark}
\usepackage{adjustbox}
\usepackage{cleveref}
\newtheorem{theorem}{Theorem}
\usepackage{float}
\usepackage{tocloft}
\usepackage{xcolor}
\usepackage[table]{xcolor}
\usepackage{tcolorbox} % in preamble
\usepackage{subcaption}
\usepackage{caption}
\usepackage{subcaption}
\usepackage{enumitem}
\definecolor{lightblue}{RGB}{173,216,230}
\definecolor{darkgreen}{RGB}{0,100,0}

\title{Formatting Instructions For NeurIPS 2026}

\title{When Marginals Match but Structure Fails: Covariance Fidelity in Generative Models}

\author{
  Nazia Riasat \\
  North Dakota State University \\
  \texttt{nazia.riasat@ndsu.edu}
}

\begin{document}

\maketitle

\begin{abstract}\vspace{-0.7em}
Generative models are increasingly deployed as 
substitutes for real data in downstream scientific 
workflows, yet standard evaluation criteria remain 
focused on marginal distribution matching. We argue 
that this represents a fundamental gap: downstream 
inference is rarely a marginal operation, and a model 
that passes every univariate diagnostic can still 
produce structurally unreliable synthetic data. We introduce \emph{covariance-level dependence 
fidelity}, measured by $D_\Sigma(P,Q) = 
\|\Sigma_P - \Sigma_Q\|_F$, as a principled, 
computable criterion for evaluating whether a 
generative model preserves the joint structure of 
data beyond its univariate marginals. Three results 
formalise this criterion. First, marginal fidelity 
provides no constraint on dependence structure: 
$D_\Sigma$ can be made arbitrarily large while all 
univariate marginals match exactly. Second, 
covariance divergence induces quantifiable downstream 
instability, including sign reversals in population 
regression coefficients. Third, bounding $D_\Sigma$ 
provides positive stability guarantees for 
dependence-sensitive procedures such as PCA via 
Davis--Kahan-type bounds.

Empirical validation across three domains, image 
data (Fashion-MNIST VAE, $n=60{,}000$), bulk RNA-seq 
(TCGA-BRCA, $n=1{,}111$), and a small-sample stress 
test (Alzheimer's gene expression, $n=113$), shows 
that $D_\Sigma/\delta$ consistently distinguishes structure-discarding from structure-preserving 
generators in cases where standard marginal diagnostics show little separation, 
confirming that covariance-level fidelity provides 
information orthogonal to existing evaluation metrics 
across domains and sample sizes.
\end{abstract}

\section{Introduction}
\vspace{-0.8em} 
Generative models are increasingly deployed as substitutes for real data in downstream scientific and decision-making workflows \cite{ho2020}. In many applied pipelines, a model is treated as ``high-quality'' if its synthetic samples are indistinguishable from the target data under marginal diagnostics: univariate goodness-of-fit tests, feature-wise summary statistics, or likelihood-based objectives that implicitly prioritize low-dimensional representations \citep{goodfellow2016deep,kingma2014}. This evaluation philosophy is widespread because marginal discrepancies are easy to measure, easy to visualize, and often correlate with perceptual realism \citep{kynkaanniemi2019, naeem2020}. However, downstream inference is rarely a marginal operation. Regression, dimensionality reduction, clustering, and representation learning depend critically on \emph{multivariate structure} covariances, tail dependence, conditional relationships, and higher-order interactions \citep{jolliffe2002, anderson2003}. A generative model can therefore appear accurate under marginal audits while remaining structurally unreliable for the tasks that motivate its use.

This paper formalizes a central claim: \emph{marginal distribution fidelity is insufficient to guarantee trustworthy downstream behavior for dependence-sensitive tasks}. In contrast, explicit control of covariance-level (second-order) 
dependence provides quantitative stability guarantees for a 
broad class of linear and spectral procedures. Extensions to 
higher-order, non-linear, and copula-level dependence remain an open direction. We study a target data-generating distribution \(P\) on \(\mathbb{R}^d\) and a candidate generative distribution \(Q\), interpreted as the output of a generative AI system. Our analysis separates two questions that are often conflated in practice:
\vspace{-0.5em}
\begin{enumerate}
    \item \text{Representation faithfulness:} do samples from \(Q\) match samples from \(P\) under standard distributional diagnostics?
    \item \text{Inferential trustworthiness:} if \(Q\) replaces \(P\) in a downstream analysis, are the resulting conclusions stable and qualitatively correct?
\end{enumerate}
\vspace{-0.4em}
We show that the first question cannot be answered by marginal criteria alone, and that failures of dependence preservation can induce large and even sign-reversing inferential errors that are invisible to marginal evaluations \citep{wasserman2004, damour2022}.

\emph{Main results}: Our theoretical contributions establish a simple hierarchy linking distributional fidelity to downstream stability.

\text{(i) \emph{Marginal fidelity does not imply dependence fidelity.}}
We first prove an impossibility result (Theorem~1): for any dimension \(d\ge 2\), there exist distributions \(P\) and \(Q\) whose univariate marginals match exactly, yet whose dependence structures differ. In particular, the copulas can be distinct, and covariance-level discrepancies can be made arbitrarily large through appropriate scaling constructions. This establishes that perfect marginal agreement does not constrain either linear or nonlinear dependence, and therefore cannot serve as a certificate of structural correctness.

(ii) Dependence divergence induces inferential instability. For a population linear regression task, we obtain the bound
\vspace{-0.4em}
\[
|\beta(P) - \beta(Q)| 
\leq \frac{1}{\sqrt{2}\,\sigma_X^2}
\|\Sigma_P - \Sigma_Q\|_F.
\]
\vspace{-0.2em} 
with equality when the covariance perturbation affects only off-diagonal entries (i.e., covariance-only perturbations under equal marginal variances).

\text{(iii) \emph{Dependence fidelity yields positive stability guarantees.}}
Finally, we show that controlling covariance-level dependence provides a sufficient condition for stability of dependence-sensitive downstream tasks. Focusing on principal component analysis (PCA), we derive bounds for both spectral stability (via Weyl-type eigenvalue perturbation) and subspace stability (via Davis--Kahan-type results \citep{davis1970}) in terms of \(\|\Sigma_P-\Sigma_Q\|_F\) and an eigengap condition (Theorem~3). Consequently, when a generative model preserves second-order dependence at an appropriate scale, it provides guarantees on the stability of key geometric properties. These guarantees are informative in the small-perturbation regime where \(\|\Sigma_P-\Sigma_Q\|_F \ll \gamma\), and may become loose when the perturbation is large relative to the eigengap.

\emph{Synthetic constructions isolating dependence effects}:
To make the theory concrete, we provide minimal synthetic examples in which marginal distributions are fixed by construction while dependence is systematically altered. These examples demonstrate two practically relevant failure modes: (a) mismatched tail dependence (e.g., Gaussian copula versus \(t\)-copulas) leading to severe errors in joint extreme-event probabilities, and (b) correlation sign flips producing regression coefficients of opposite sign. Both constructions are invisible to marginal-only evaluation, yet they induce large downstream discrepancies. 
Across three empirical settings, a VAE applied to Fashion-MNIST images
($n = 60{,}000$, $p = 50$ PCA dimensions, $n/p = 1{,}200$), bulk
RNA-seq from TCGA-BRCA ($n = 1{,}111$, $p = 100$ genes, $n/p = 11.1$),
and a small-sample stress test on Alzheimer's gene expression
($n = 113$, $p = 50$ genes, $n/p = 2.26$), the diagnostic behaviour
of $D_\Sigma/\delta$ is consistent with theoretical predictions in
every case examined.

While each mathematical component used in this paper draws on 
classical results from multivariate analysis and matrix 
perturbation theory, our contribution is conceptual and 
structural: we integrate these tools into a unified diagnostic 
framework that connects distributional fidelity to inferential 
reliability, and we identify covariance-level dependence 
preservation as a sufficient structural condition for trustworthy 
stability in dependence-sensitive downstream tasks. Crucially, 
the resulting diagnostic is computable from samples, 
complementary to existing metrics such as FID and 
precision/recall that target marginal or perceptual fidelity, 
and sensitive to generative failures that those metrics cannot 
detect. To our knowledge, the hierarchy linking marginal fidelity, 
dependence divergence, and downstream stability has not 
been formalized in the generative modeling literature, 
nor validated simultaneously across image, genomic, and 
small-sample domains.

The remainder is organized as: setup and notation in Section~\ref{sec:setup}, related work in Section~\ref{sec:related_work}, main results in Section~\ref{sec:Main_results}, synthetic examples in Section~\ref{sec:Examples}, discussion in Section~\ref{sec:Discussions}, and conclusion in Section~\ref{sec:Conclusion}.
\vspace{-0.4em}
\section{Problem Setup and Notation}\label{sec:setup}
\vspace{-0.8em} 
Let $P$ denote a target data-generating distribution on $\mathbb{R}^d$, and let $Q$ denote the distribution induced by a generative model. 
We assume throughout that $P$ and $Q$ have finite second moments so that covariance matrices are well-defined. The goal of generative modeling is to approximate $P$ by $Q$ in a way that preserves not only marginal behavior, but also multivariate structure relevant to downstream statistical tasks.
\vspace{-0.7em}
\paragraph{Marginal fidelity.}
We say that $Q$ achieves \emph{marginal fidelity} if all univariate marginals match:
\[
P_j = Q_j,\qquad j=1,\ldots,d,
\]
\vspace{-0.1em}
where $P_j$ and $Q_j$ denote the marginal distributions of the $j$-th coordinate under $P$ and $Q$, respectively. 
Marginal fidelity is commonly encouraged in practice via likelihood-based objectives, moment matching, or univariate goodness-of-fit diagnostics.

\vspace{-0.6em}
\paragraph{Dependence fidelity (covariance-level).}
Let \[
\Sigma_P := \mathrm{Cov}_{X\sim P}(X), 
\qquad 
\Sigma_Q := \mathrm{Cov}_{X\sim Q}(X)
\]
\vspace{-0.2em}
denote the covariance matrices under $P$ and $Q$. 
We measure covariance-level dependence discrepancy by the Frobenius distance
\[
D_{\Sigma}(P,Q) := \|\Sigma_P-\Sigma_Q\|_F,
\]
where $\|\cdot\|_F$ is the Frobenius norm. 
Small values of $D_{\Sigma}(P,Q)$ indicate that the \emph{linear} dependence structure of $Q$ closely matches that of $P$. We note that $D_\Sigma(P,Q)$ is scale-sensitive; the arbitrarily 
large divergence in Theorem~\ref{thm:marginal} is achieved partly 
by scaling marginal variance, which reflects a limitation of the 
raw Frobenius criterion rather than a fundamental property of 
dependence mismatch. A normalized alternative 
$\tilde{D}_\Sigma := \|\mathcal{C}_P - \mathcal{C}_Q\|_F$, where 
$\mathcal{C}$ denotes the correlation matrix, is preferable when 
scale invariance is required; all three theorems extend to this 
normalized form under appropriate rescaling. Throughout this paper, $D_\Sigma$ serves as the primary diagnostic 
since it directly enters the Frobenius-norm perturbation bounds of 
Theorems~2--3. As $D_\Sigma$ is scale-sensitive, we report the 
normalized form $\tilde{D}_\Sigma = \|C_P - C_Q\|_F$ alongside 
in all empirical tables; the two metrics exhibit consistent rankings across all five generator comparisons 
(Tables~3--7), confirming that findings are not an 
artifact of marginal variance scaling.

\vspace{-0.5em}
\paragraph{Downstream functionals and inferential stability.}
We consider population-level functionals $T(\cdot)$ that depend on multivariate structure, including:
(i) population regression coefficients,
(ii) principal component eigenvalues and principal subspaces, and
(iii) joint tail probabilities.
We say that the generative model is \emph{inferentially stable} (for a specified class of dependence-sensitive functionals) if $T(Q)$ is controlled by $T(P)$ through bounds in terms of  $D_{\Sigma}(P,Q)$ for the class of functionals under consideration.

\textbf{Norm relations.} For any matrix $A$, 
$\|A\|_2 \leq \|A\|_F$; this allows Frobenius 
control of covariance error to imply the spectral 
and subspace perturbation bounds used in Section~4.

\textbf{Clarification on dependence notions.} \begin{remark} The guarantees in this paper apply to 
covariance-level (second-order) dependence fidelity 
$D_\Sigma$, which directly controls stability for 
linear and spectral procedures (Theorems~2--3). 
Appendix~B.1 illustrates a tail-dependence failure 
mode that lies strictly outside the scope of 
$D_\Sigma$; it is included to clarify the boundary 
of the criterion, not to claim $D_\Sigma$ addresses 
it.
\end{remark}
The next section formalizes the limitations of marginal fidelity and establishes quantitative links between dependence discrepancy and downstream instability/stability.
\vspace{-0.7em}
\section{Related Work}
\label{sec:related_work}
\vspace{-0.6em}
Standard evaluation relies on marginal or likelihood-based criteria, distributional distances such as ~\citep{Villani2009Wasserstein} and MMD~\citep{gretton2005measuring}, and perceptual metrics including FID~\citep{heusel2017} and precision/recall~\citep{naeem2020}. While these metrics do penalize some covariance mismatch, they do not provide interpretable, actionable perturbation thresholds for stability in specific downstream linear or spectral tasks; our framework formalizes this gap. Critically, ~\citep{theis2016} showed that no single metric simultaneously captures all desirable generative properties, we sharpen this observation by identifying covariance-level fidelity as the specific property that existing metrics collectively miss and that determines downstream inferential reliability.

Copula theory separates marginal behavior from joint dependence~\citep{sklar1959}. ~\citep{damour2022} established that underspecification causes silent failures in downstream inference but did not provide computable diagnostic criteria. We close this gap: building on Weyl's inequality and Davis--Kahan~\citep{yu2015useful}, we derive the first closed-form bounds linking covariance divergence to downstream instability for regression and PCA.

Distance covariance~\citep{szekely2007measuring} and HSIC~\citep{gretton2005measuring} capture nonlinear dependence but yield no perturbation bounds. The RV-coefficient~\cite{robert1976unifying} measures covariance similarity but carries no eigengap information and provides no actionable stability threshold; $D_\Sigma/\delta$ fills both gaps by incorporating covariance discrepancy together with eigengap information.

\vspace{-1em}
\section{Main Results: Dependence Fidelity and Trustworthy Inference}
\label{sec:Main_results}
\vspace{-0.7em}
We now formalise the central claim of this work. 
Throughout, let $P$ denote a target distribution 
on $\mathbb{R}^d$ and $Q$ the distribution induced 
by a generative model. We show that matching 
marginal distributions is insufficient to guarantee 
reliable downstream behaviour, while explicit 
preservation of covariance-level dependence provides 
sufficient conditions for stability of 
dependence-sensitive functionals including linear 
regression and PCA.
\vspace{-0.5em}
\subsection{Marginal Fidelity Does Not Imply Dependence Fidelity}
\vspace{-0.5em} 
Current evaluation practices for generative models often focus on marginal distribution matching, either explicitly through univariate goodness-of-fit diagnostics or implicitly through likelihood-based objectives \citep{goodfellowGAN2014, theis2016,barratt2018, borji2022pros}. We first show that such criteria provide no guarantee that the joint structure of the data is preserved. 

\begin{theorem}[Marginal fidelity does not imply dependence fidelity]
\label{thm:marginal}
Let $d \geq 2$. There exist probability distributions $P$ and $Q$ on $\mathbb{R}^d$ such that:
\begin{enumerate}
    \item All univariate marginals match exactly: $P_j = Q_j$ for each coordinate $j$.
    \item The copulas differ: $C_P \neq C_Q$, \[\mathcal{D}_{\mathrm{cop}}(P,Q) > 0
    \]
    for any characteristic-kernel (e.g., Gaussian kernel) MMD on the copula domain. \vspace{0.2em}
    \item The covariance matrices can be chosen to differ:
    \[
    \Sigma_P \neq \Sigma_Q, \quad 
    \mathcal{D}_\Sigma(P,Q) := \|\Sigma_P - \Sigma_Q\|_F > 0.
    \]
\end{enumerate}
\vspace{-0.3em}
Moreover, the covariance divergence can be made arbitrarily large while maintaining exact marginal agreement \citep{nelsen2006, sklar1959}. 
\end{theorem} 

This can be achieved by scaling the marginal variance. 
If $X \sim \mathcal{N}(0,\sigma^2)$ and the covariance matrices differ only in their correlation structure, then
\[
D_\Sigma(P,Q) = \|\Sigma_P - \Sigma_Q\|_F = 2\sqrt{2}\,\sigma^2 |\rho|,
\]

where in the unit-variance construction of the proof
($\sigma^2 = 1$), this simplifies to $2\sqrt{2}|\rho|$ and grows without bound as $\sigma^2 \to \infty$ while the marginals remain matched. Thus, marginal fidelity cannot certify structural reliability for dependence-sensitive tasks \citep{damour2022, theis2016}.
\vspace{-0.5em}
\subsection{Dependence Divergence Induces Inferential Instability}
\vspace{-0.5em}
Dependence mismatch is not merely a representational discrepancy; it directly affects downstream inference. We formalize this phenomenon using population linear regression as a canonical dependence-sensitive task \citep{anderson2003, van2000}.

Let $(X,Y) \in \mathbb{R}^2$. We assume that variables are centered (or equivalently that the regression model includes an intercept term), ensuring that the population least-squares slope is given by
\vspace{-0.3em}
\[
\beta(P) := \frac{\mathrm{Cov}_P(X,Y)}{\mathrm{Var}_P(X)}.
\]
\vspace{-0.2em}
The centering assumption ensures that the regression slope depends only on second-order structure, which allows us to relate inferential sensitivity directly to covariance perturbations.

\begin{theorem}[Dependence divergence directly controls inferential sensitivity.]
\label{thm:regression}

Let $P$ and $Q$ be distributions on $\mathbb{R}^2$ with finite second moments and 
matching marginal variances across $P$ and $Q$:
\[
\mathrm{Var}_P(X) = \mathrm{Var}_Q(X) = \sigma_X^2 > 0, 
\qquad
\mathrm{Var}_P(Y) = \mathrm{Var}_Q(Y) = \sigma_Y^2 > 0.
\]
Let $\Sigma_P$ and $\Sigma_Q$ denote the covariance matrices of $(X,Y)$ under $P$ and $Q$, respectively. Then, \[|\beta(P) - \beta(Q)| = \frac{1}{\sqrt{2}\,\sigma_X^2}
\|\Sigma_P - \Sigma_Q\|_F.
\]

\end{theorem}
\vspace{-0.5em}
This theorem provides a quantitative bound linking covariance-level dependence divergence to inferential error \citep{anderson2003}. In particular, two distributions with identical marginals but opposite correlation structures yield regression coefficients of opposite sign. Thus, dependence changes that are invisible to marginal diagnostics can produce large inferential discrepancies, including sign reversals. 
\vspace{-0.6em}
\subsection{Stability Guarantees Under Dependence Fidelity}
\label{thm:pca}
\vspace{-0.3em}
The bound is informative in the standard perturbation regime 
where the covariance perturbation is small relative to the 
eigengap of the reference matrix. Under this condition, 
Weyl's inequality implies
\[
|\lambda_k(\Sigma_P) - \lambda_k(\Sigma_Q)| \leq 
\|\Sigma_P - \Sigma_Q\|_2 \quad \text{for all } k.
\]
Since $\|A\|_2 \leq \|A\|_F$, we obtain
\[
|\lambda_k(\Sigma_P) - \lambda_k(\Sigma_Q)| \leq 
\|\Sigma_P - \Sigma_Q\|_F,
\]
ensuring the required spectral separation for Theorem~3. While the previous results highlight the limitations of marginal-based evaluation, we next show that explicit control of dependence divergence yields positive guarantees. We consider principal component analysis (PCA) \cite{jolliffe2002, davis1970}, a fundamental dependence-sensitive operation underlying representation learning and dimensionality reduction.

\begin{theorem}[Stability of PCA under covariance-level dependence fidelity]
\label{thm:pca}
Let $P$ and $Q$ be distributions on $\mathbb{R}^d$ with zero mean and covariance matrices $\Sigma_P$ and $\Sigma_Q$ with finite second moments.

\textbf{(Eigenvalue stability)} For each $k=1,\dots,d$,
\[
|\lambda_k(\Sigma_P) - \lambda_k(\Sigma_Q)|
\leq \|\Sigma_P - \Sigma_Q\|_2
\leq \|\Sigma_P - \Sigma_Q\|_F.
\]
\vspace{-0.3em}
In particular,\vspace{-0.7em}
\[
\max_{k \le d} |\lambda_k(\Sigma_P) - \lambda_k(\Sigma_Q)| \le D_\Sigma(P,Q).
\]

\textbf{(Subspace stability)} 
Let $U_P \in \mathbb{R}^{d \times r}$ 
and $U_Q \in \mathbb{R}^{d \times r}$ denote the matrices whose 
columns are the top-$r$ orthonormal eigenvectors of $\Sigma_P$ 
and $\Sigma_Q$ respectively. Let $\gamma :=\lambda_r(\Sigma_P) - \lambda_{r+1}(\Sigma_P) > 0$
denote the eigengap of $\Sigma_P$. Under the small-perturbation 
assumption $D_\Sigma(P,Q) < \gamma$, which ensures the required spectral separation between the 
leading $r$-dimensional eigenspace of $\Sigma_P$ and its 
orthogonal complement under perturbation by $\Sigma_Q - \Sigma_P$,
\vspace{-0.5em}
\begin{equation*}
    \|\sin\Theta(U_P, U_Q)\|_2 \leq 
    \frac{2\|\Sigma_P - \Sigma_Q\|_F}{\gamma}.
\end{equation*}
\end{theorem}
\vspace{-0.7em}
The subspace perturbation bound is informative when the covariance discrepancy satisfies
\vspace{-0.3em}
\[
D_{\Sigma}(P,Q) \ll \gamma,
\] 
\vspace{-0.2em}
where $\gamma$ denotes the eigengap of the population covariance matrix. When the eigengap is small, even minor perturbations in the covariance structure may result in substantial instability of the estimated principal subspaces. The following result shows that controlling covariance-level dependence divergence is sufficient for stability of common spectral learning procedures.
Together, Theorems 1–3 thus form a hierarchy: impossibility → instability bound → positive guarantee. This 
collectively establishes $D_\Sigma/\delta$ as a 
computable, theoretically grounded criterion for 
structural evaluation summarized in Table~\ref{tab:theorems} (Appendix). Generative systems deployed in 
dependence-sensitive scientific tasks should be 
assessed not only for marginal realism but for 
their ability to preserve the joint structure that 
those tasks actually depend on 
\citep{ovadia2019can,  zhao2018evaluating}.

\vspace{-0.4em}
\section{Synthetic Examples: Isolating Dependence Effects Under Perfect Marginal Fidelity}
\label{sec:Examples}
\vspace{-0.5em}
The two examples below make the results of 
Section~\ref{sec:Main_results} concrete. In each case marginal 
distributions are identical by construction, so 
every observable difference traces entirely to 
dependence structure, the cleanest possible 
isolation of the effect. The tail-dependence failure 
in Appendix~B.1 lies strictly outside the 
covariance-level scope; Examples~5.1 and~5.2 
directly illustrate Theorems~2 and~3 respectively.
\vspace{-0.5em}
\subsection{Example I: Correlation Sign Flip --- Regression Instability}
\vspace{-0.5em}
Our first example provides a direct illustration of Theorem 2, showing that linear dependence
divergence alone can invert inferential conclusions.
\vspace{-0.4em}
\paragraph{Construction.}
Let $P$ and $Q$ be bivariate normal distributions with zero means and unit variances, differing only in the sign of correlation: \vspace{-0.5em}
\[
\Sigma_P =
\begin{pmatrix}
1 & \rho \\
\rho & 1
\end{pmatrix},
\quad
\Sigma_Q =
\begin{pmatrix}
1 & -\rho \\
-\rho & 1
\end{pmatrix},
\quad \rho \in (0,1).
\]
Both distributions share identical marginals $N(0,1)$.
\vspace{-0.8em} 
\paragraph{Downstream task.}
Consider the population least-squares regression slope \citep{anderson2003}
\[
\beta(P) := \frac{\mathrm{Cov}_P(X,Y)}{\mathrm{Var}_P(X)}
\]
\vspace{-0.7em}
Under these distributions, \vspace{-0.5em} \[
\beta(P)=\rho, \qquad \beta(Q)=-\rho.
\] 

Both distributions satisfy \vspace{-0.7em}
\[
\Var_P(X)=\Var_Q(X)=1,
\]

so the conditions of Theorem~2 hold with $\sigma_X^2 = 1$.
\vspace{-0.8em} 
\paragraph{Quantitative instability.}
The coefficient difference satisfies
\vspace{-0.2em} 
\[
|\beta(P) - \beta(Q)| = 2|\rho|,
\]
\vspace{-0.8em}
while the covariance divergence is
\[
\|\Sigma_P - \Sigma_Q\|_F = 2\sqrt{2}\,|\rho|.
\]
Thus,
\vspace{-0.9em}
\[
|\beta(P) - \beta(Q)| = \frac{1}{\sqrt{2}} \|\Sigma_P - \Sigma_Q\|_F,
\]
\vspace{-0.3em}
which matches the bound in Theorem~2 with equality.

\begin{figure}[t]
\centering
\begin{tcolorbox}[
    colback=white,
    colframe=black,
    boxrule=0.6pt,
    arc=2pt,
    width=0.75\textwidth,
    boxsep=3pt,
    left=2pt,right=2pt,top=2pt,bottom=2pt
]
\centering
\includegraphics[width=\textwidth]{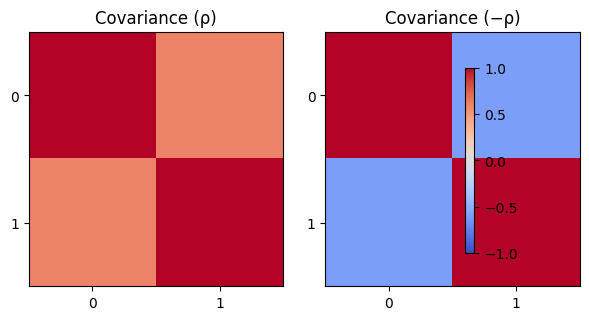}
\end{tcolorbox}
\caption{Covariance structures for two bivariate normal distributions with identical marginals but opposite correlation signs. The change in dependence structure reverses the population regression slope, illustrating inferential instability under dependence divergence.}
\label{fig:cov_flip}
\end{figure}
\vspace{-0.3em}
This example shows that even the \emph{sign} of 
an inferred relationship can reverse under 
dependence divergence that leaves all marginals 
unchanged, a qualitative failure, not merely a 
quantitative one. Figure~\ref{fig:cov_flip} illustrates the 
mechanism directly: the covariance matrices differ 
only off-diagonally, yet the regression slope 
flips from $\rho$ to $-\rho$.
\vspace{-0.3em} 

\subsection{Example II: PCA Subspace Stability Under Controlled
Covariance Perturbation}
\label{ex:pca}
\vspace{-0.6em}
The previous example showed what goes wrong when 
dependence is discarded. This example asks the 
complementary question: when covariance-level 
fidelity \emph{is} maintained, what does it 
actually guarantee? We construct a minimal setting 
that makes Theorem~\ref{thm:pca} concrete, exposing both the 
stability it delivers and the sharp transition at 
the eigengap threshold.

\paragraph{Construction.}
\vspace{-0.9em}
Let $P = \mathcal{N}(0, \Sigma_P)$ where
\vspace{-0.8em}
\begin{equation}
    \Sigma_P = \begin{pmatrix} 3 & 0 \\ 0 & 1 \end{pmatrix},
\end{equation}
so that $\lambda_1(\Sigma_P) = 3$, $\lambda_2(\Sigma_P) = 1$, and the 
eigengap is $\delta := \lambda_1(\Sigma_P) - \lambda_2(\Sigma_P) = 2$. 
The leading eigenvector of $\Sigma_P$ is $u_P = (1, 0)^\top$.
\vspace{-0.03em}
We introduce a family of generative distributions 
$Q_\varepsilon = \mathcal{N}(0, \Sigma_{Q_\varepsilon})$ parameterized 
by $\varepsilon \geq 0$, where
\vspace{-0.5em}
\begin{equation}
    \Sigma_{Q_\varepsilon} = \begin{pmatrix} 3 & \varepsilon \\ 
    \varepsilon & 1 \end{pmatrix}.
\end{equation}
\vspace{-0.1em} 
Both $P$ and $Q_\varepsilon$ share identical marginal variances for all 
$\varepsilon$, so the perturbation is invisible to any univariate 
diagnostic. The covariance divergence is
\vspace{-0.5em}
\begin{equation}
    D_\Sigma(P, Q_\varepsilon)
    = \|\Sigma_P - \Sigma_{Q_\varepsilon}\|_F
    = \sqrt{2}\,\varepsilon.
\end{equation}

\vspace{-0.8em}
\paragraph{Theoretical guarantee.}
\vspace{-0.6em}
Theorem~\ref{thm:pca} gives the Davis--Kahan bound
\vspace{-0.3em}
\begin{equation}
    \|\sin\angle(u_P,\, u_{Q_\varepsilon})\|_2
    \;\leq\;
    \frac{2D_\Sigma(P, Q_\varepsilon)}{\delta}
    \;=\;
    \sqrt{2}\,\varepsilon.
\end{equation}
\vspace{-0.1em} This bound is meaningful when $\varepsilon < \sqrt{2}$, the regime where 
the perturbation stays below the eigengap. Beyond this threshold the bound 
exceeds one and loses informativeness, reflecting that large perturbations 
relative to the eigengap can rotate the principal subspace arbitrarily.
\vspace{-0.8em}
\paragraph{Exact subspace angle.}
The leading eigenvector of $\Sigma_{Q_\varepsilon}$ can be computed 
in closed form, giving
\vspace{-0.5em}
\begin{equation}
    \sin\angle(u_P,\, u_{Q_\varepsilon})
    = \frac{|\varepsilon|}{\sqrt{\bigl(1 + \sqrt{1+\varepsilon^2}
    \bigr)^2 + \varepsilon^2}}.
\end{equation}
\vspace{-0.1em}
This near-equality confirms that the Davis--Kahan 
bound is tight in the stable regime, the 
theoretical guarantee is not just valid but 
genuinely informative.
\vspace{-0.8em}
\paragraph{Observations.}
Figure~\ref{fig:pca_stability} summarizes the results.
In the stable regime ($\varepsilon < \sqrt{2}$, shaded region), the 
principal subspace of $Q_\varepsilon$ tracks that of $P$ closely, and 
the theoretical bound is nearly tight. At $\varepsilon = 0.5$, the leading 
eigenvector of $Q_\varepsilon$ deviates only slightly from $u_P$, and 
downstream procedures relying on this subspace remain reliable.
Once $\varepsilon$ crosses the eigengap threshold, the picture changes: 
at $\varepsilon = 2.5$, the leading eigenvector has rotated substantially, 
and any downstream method that depends on the principal direction of $P$ 
will draw incorrect conclusions from $Q_\varepsilon$. Crucially, this entire transition is invisible to 
any marginal diagnostic, the univariate 
distributions; of $Q_\varepsilon$ are identical to 
those of $P$ throughout.

\begin{figure}[t]
\centering
\begin{tcolorbox}[
    colback=white,
    colframe=lightblue,
    boxrule=0.6pt,
    arc=2pt,
    width=0.8\textwidth,
    boxsep=3pt,
    left=2pt,right=2pt,top=2pt,bottom=2pt
]
\centering
\includegraphics[width=\textwidth]{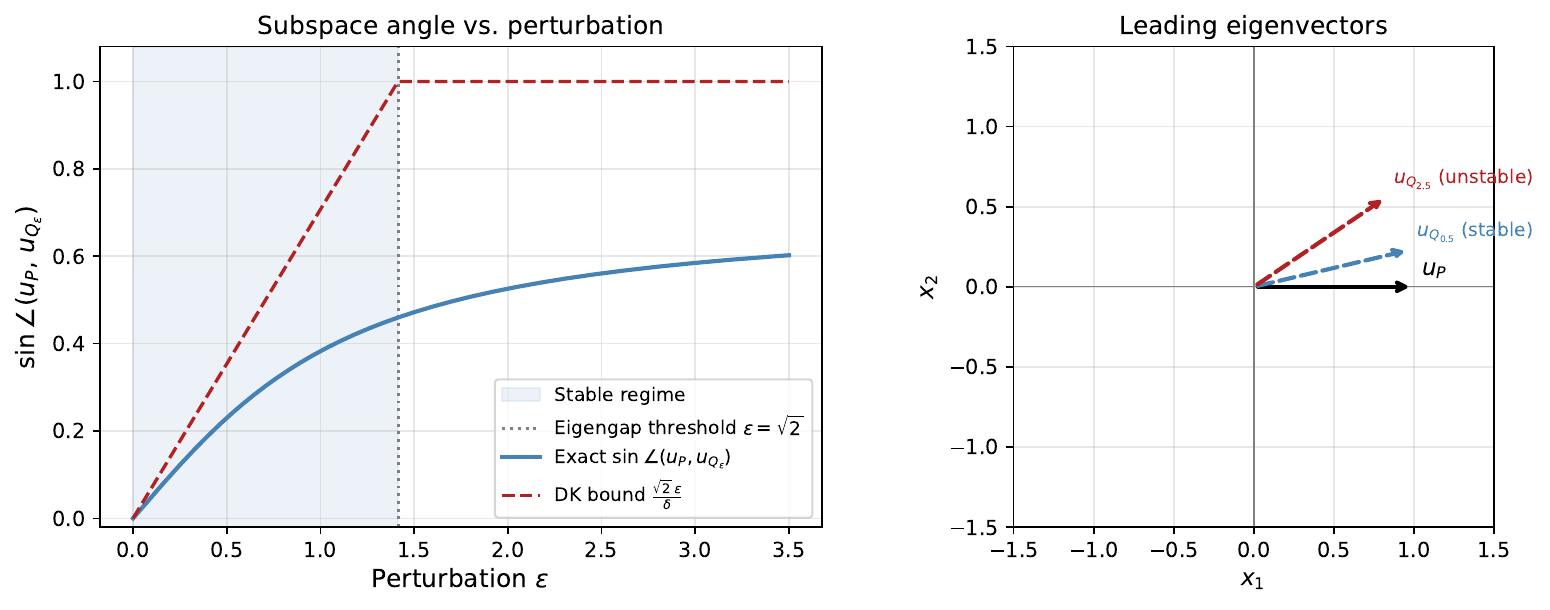}
\end{tcolorbox}
\caption{PCA subspace instability under covariance perturbation.
\emph{Left:} Subspace angle $\sin\angle(u_P, u_{Q_\varepsilon})$ as a 
function of perturbation $\varepsilon$. The dashed line shows the 
Theorem~\ref{thm:pca} Davis--Kahan bound; the shaded region marks the 
stable regime ($\varepsilon < \sqrt{2}$). The bound is nearly tight 
throughout and becomes vacuous at the eigengap threshold (dotted vertical 
line). \emph{Right:} Leading eigenvectors at $\varepsilon = 0.5$ (stable) 
and $\varepsilon = 2.5$ (unstable), illustrating the qualitative transition 
in subspace alignment that marginal diagnostics cannot detect.}
\label{fig:pca_stability}
\end{figure}

\vspace{-0.9em}

\section{Discussion: Dependence Fidelity for Trustworthy Generative Modeling}
\label{sec:Discussions}
\vspace{-0.7em}
The theoretical and synthetic results converge on a 
single conclusion: for dependence-sensitive tasks, 
trustworthiness is a structural property, not a 
marginal one. Exact marginal agreement leaves the 
joint covariance structure entirely unconstrained 
(Theorem~1), even moderate covariance distortions 
can invert regression coefficients (Theorem~2), and 
explicit control of covariance divergence yields 
quantitative stability guarantees for PCA 
(Theorem~3).

\textbf{Why marginal evaluation is structurally 
blind.} The dependence failures demonstrated in 
Section~5 are invisible to marginal diagnostics by 
construction: they cannot be detected by increasing 
sample size, improving marginal fit, or refining 
likelihood-based objectives, because none of those 
objectives involve the off-diagonal covariance 
\citep{theis2016, zhao2018evaluating}. This is not a 
limitation of any particular method; it is a 
logical consequence of what marginal fidelity 
measures.

\textbf{Implications for generative architectures.} 
Although the analysis is model-agnostic, some 
architectures are structurally more exposed. 
Diffusion models may accumulate covariance 
distortions across denoising steps, particularly in low-variance eigenmodes , and the diagonal posterior assumption in VAEs constrains the approximate posterior
covariance to be diagonal in latent space \citep{ho2020, song2021, wang2025spectral}. While the aggregate marginal posterior and the
decoder mapping can in principle capture dense data-space correlations, in practice this
architectural bottleneck empirically leads to dependence collapse: the VAE fails to reproduce
off-diagonal covariance structure, as quantified by $D_\Sigma$ and documented in Appendix B.4. 
Addressing this requires full-covariance posteriors, 
structured parameterizations such as low-rank or 
Cholesky-based covariance models, or explicit 
post-generation diagnostics 
\citep{kingma2014, louizos2016structured, dai2019diagnosing}.

\textbf{Practical diagnostic.} $D_\Sigma$ is 
computable from samples as $\hat{D}_\Sigma = 
\|\hat{\Sigma}_P - \hat{\Sigma}_Q\|_F$, making it 
straightforward to incorporate alongside existing 
metrics such as FID and precision/recall 
\citep{borji2022pros}. In high-dimensional settings 
($d \gg n$), shrinkage estimators such as 
Ledoit--Wolf \citep{ledoit2004well} provide reliable 
covariance estimates; the practical guideline is 
$n \approx 5d$ for stable estimation. Across three datasets spanning image data 
(Fashion-MNIST, $n/p = 1{,}200$), bulk RNA-seq 
(TCGA-BRCA, $n/p = 11.1$), and a small-sample 
stress test (Alzheimer's, $n/p = 2.26$), 
$D_\Sigma/\delta$ consistently separates 
structure-discarding from structure-preserving 
generators where KS marginal distances cannot (Table~\ref{tab:empirical_summary}; 
Figure~\ref{fig:diagnostic_summary}). The KS column shows all five generators as 
broadly comparable;  $D_\Sigma/\delta$ 
and RV, by contrast, clearly separate them; and this structural gaps translates directly into inferential consequences. The marginal diagnostic cannot see 
any of this; the dependence diagnostic sees all of it.

\textbf{Scope and limitations.} The guarantees apply to 
covariance-level (second-order) structure and directly
cover linear and spectral procedures. Appendix~\ref{app:B7} shows that MMD on TCGA-BRCA also provides modest additional separation between structure-discarding and structure-preserving 
generators, but it does not provide an actionable stability threshold. 
Appendix~\ref{app:B7} extends this finding to Fashion-MNIST, where FID actively \emph{favours} the VAE over the structure-preserving Gaussian baseline ($72.02$ vs.\ 
$192.39$), confirming that perceptual and structural 
quality are orthogonal properties that require separate 
diagnostics. Higher-order dependence, 
tail behaviour, conditional structure, nonlinear 
interactions, none of these are captured by $D_\Sigma$, 
as illustrated by the Gaussian vs.\ t-copula failure 
in Appendix~B.1.

\vspace{-0.9em} 
\section{Conclusion}
\label{sec:Conclusion}
\vspace{-0.9em} 
This work establishes a simple but consequential 
point: for generative models used in downstream 
scientific analysis, trustworthiness is a structural 
property that marginal evaluation cannot certify. A model can pass every univariate goodness-of-fit test while its synthetic
data silently inverts regression directions and collapses principal subspaces,
because these linear and spectral failures live entirely in the off-diagonal
covariance, which marginal diagnostics do not see.
Joint extreme-event misrepresentation is a distinct failure mode: as
established in Remark~1 and Appendix~\ref{app:B1}, tail dependence lies outside the
scope of $D_\Sigma$ and requires copula-level diagnostics beyond second-order
structure. We formalized this through three results. Theorem~1 
shows the gap is not merely possible but unbounded: 
$D_\Sigma$ can be arbitrarily large while all 
marginals match exactly. Theorem~2 quantifies the 
downstream cost: covariance divergence directly 
bounds regression instability, with equality 
achievable, meaning sign reversals are not edge 
cases but the generic outcome when dependence is 
discarded. Theorem~3 turns the argument around: 
bounding $D_\Sigma$ is \emph{sufficient} to 
guarantee PCA stability via Davis--Kahan, giving 
practitioners an actionable threshold 
($D_\Sigma/\delta < 1$) rather than a qualitative 
warning. These guarantees are computable. Across three domains, Fashion-MNIST images ($n = 60{,}000$), TCGA-BRCA bulk RNA-seq ($n = 1{,}111$), and Alzheimer's gene expression ($n = 113$), $D_\Sigma/\delta$ correctly identifies structure-discarding generators as unstable and structure-preserving ones as stable, even in regimes where KS marginal distances are comparable.

Covariance-level fidelity is not a replacement for existing metrics. Appendix~\ref{app:B2} shows that it is orthogonal to perceptual quality: FID favours the VAE (72.02 vs.\ 192.39), while $D_\Sigma/\delta$ favours the Gaussian baseline (0.03 vs.\ 0.64). This separation is essential; a model can appear realistic yet be structurally unreliable, or conversely, appear unrealistic while preserving the dependence structure required for valid inference. The VAE dependence collapse (Appendix~\ref{app:B4}) illustrates this distinction. It reflects the diagonal posterior constraint in latent space, whose effect on data-space covariance depends on the decoder mapping and training dynamics, and $D_\Sigma/\delta = 0.64$ quantifies precisely how far a well-trained VAE lies from the stability threshold on Fashion-MNIST. Table~\ref{tab:empirical_summary} and Figure~\ref{fig:diagnostic_summary} summarize all five comparisons, while Table~\ref{tab:theory_confirmation} links each theoretical result to its empirical confirmation. Second-order diagnostics capture failures that perceptual metrics cannot detect, even when higher-order structure is present. Extending these guarantees beyond covariance, toward copula-level fidelity, remains an important direction for future work (see Appendix~\ref{app:broader} for broader impacts).

\begin{figure}[t]
\centering
\begin{tcolorbox}[
    colback=white,
    colframe=darkgreen,
    boxrule=0.9pt,
    arc=2pt,
    width=1\textwidth,
    boxsep=3pt,
    left=2pt,right=2pt,top=2pt,bottom=2pt
]
\centering
\includegraphics[width=0.7\textwidth]{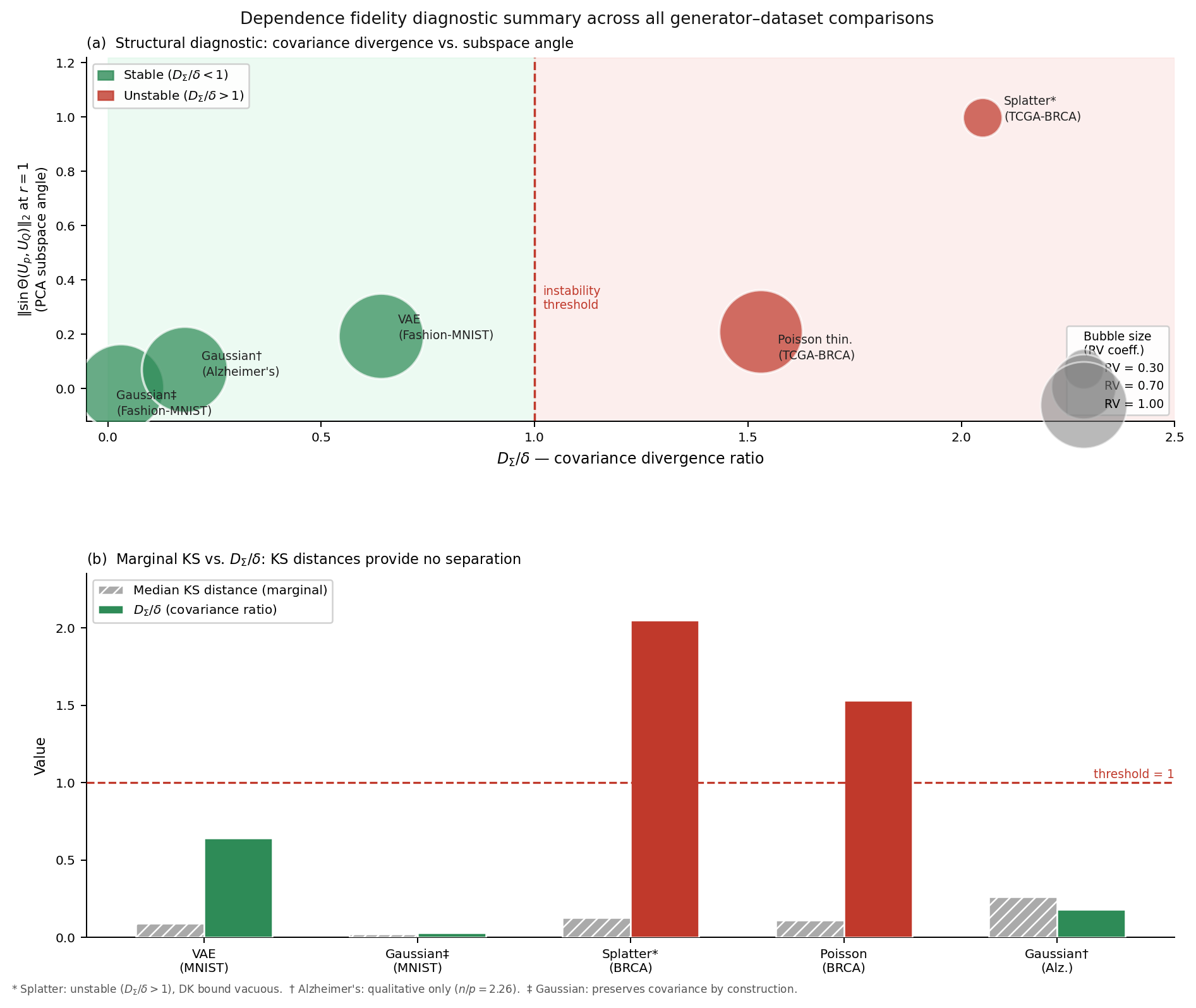}
\end{tcolorbox}
\caption{Diagnostic summary across all generator--dataset comparisons. 
\textbf{(a)} Each generator plotted by $D_\Sigma/\delta$ versus PCA 
subspace angle at $r=1$; bubble area encodes the RV-coefficient. The 
dashed vertical line marks the instability threshold ($D_\Sigma/\delta = 1$, 
Theorem~3). \textbf{(b)} Median KS distance (hatched, lighter) versus 
$D_\Sigma/\delta$ (solid) per generator. KS distances cluster in $[0.02, 0.26]$ 
with no structural separation; $D_\Sigma/\delta$ immediately identifies 
Splatter as unstable. Together, panels (a) and (b) visualize the central claim of 
Table~\ref{tab:empirical_summary}: marginal diagnostics cannot see 
what $D_\Sigma/\delta$ sees.}
\label{fig:diagnostic_summary}
\end{figure}

\newpage
\bibliographystyle{abbrvnat}
\bibliography{neurips2026}    % 

\newpage
\section*{Appendix Contents}
\addcontentsline{toc}{section}{Appendix Contents}

This appendix is organised as follows.
Appendix~\ref{app:proof} contains complete proofs of all theoretical results.
Appendix~\ref{app:synthetic} presents synthetic and empirical experiments organised by
dataset and generative architecture.
Appendix~\ref{app:B6} provides a unified summary of all empirical findings with
consolidated tables.
Appendix~\ref{app:C} describes reproducibility details.
Appendix~\ref{app:broader} contains the extended Broader Impacts discussion.

\bigskip

\begin{itemize}

% =========================
% Appendix A
% =========================
\item \textbf{Appendix A: Proofs of Main Results} \hfill  \mbox{p.~\pageref{app:proof}}

\begin{itemize}
    \item A.0 Auxiliary results (Lemmas A.1--A.2) \hfill \mbox{p.~\pageref{app:A0}}
    \item A.1 Proof of Theorem~1 (marginal fidelity does not imply dependence fidelity) \hfill \mbox{p.~\pageref{app:A1}}
    \item A.2 Proof of Theorem~2 (dependence divergence induces inferential instability) \hfill \mbox{p.~\pageref{app:A2}}
    \item A.3 Proof of Theorem~3 (stability of PCA under covar.-level dependence fidelity) \hfill \mbox{p.~\pageref{app:A3}}
\end{itemize}

\vspace{0.2em}

% =========================
% Appendix B
% =========================
\item \textbf{Appendix B: Experiments and Empirical Validation} \hfill \mbox{p.~\pageref{app:synthetic}}

\begin{itemize}
    \item B.1 Gaussian vs.\ $t$-copula: tail dependence failure (synthetic) \hfill \mbox{p.~\pageref{app:B1}}
    
    \item B.2 VAE dependence collapse on Fashion-MNIST ($n=60{,}000,\ p=50,\ n/p=1{,}200$) \hfill \mbox{p.~\pageref{app:B2}}
    
    \item B.3 Small-sample robustness: Alzheimer's gene expression ($n=113,\ p=50,\ n/p=2.26$) \hfill \mbox{p.~\pageref{app:B3}}
    
    \item B.4 Dependence Collapse in Variational Autoencoders\hfill \mbox{p.~\pageref{app:B4}}
    
    \item B.5 Dependence fidelity on TCGA-BRCA bulk RNA-seq ($n=1{,}111,\ p=100,\ n/p=11.1$) \hfill \mbox{p.~\pageref{app:B5}}
    
    \item B.6 Empirical summary: unified evidence for dependence fidelity across datasets \hfill \mbox{p.~\pageref{app:B6}}

    \item B.7 MMD and FID Confirm the Diagnostic Gap of ($D_\Sigma/\delta$ ) \hfill \mbox{p.~\pageref{app:B7}}
\end{itemize}

\item \textbf{Appendix C: Reproducibility} \hfill \mbox{p.~\pageref{app:C}}

\item \textbf{Appendix D: Broader Impacts} \hfill \mbox{p.~\pageref{app:D}}

\bigskip
\hrule
\bigskip
\section*{List of Figures in Appendix}
\addcontentsline{toc}{section}{List of Figures in Appendix}

\begin{itemize}

\item Figures 4--6: Gaussian vs.\ $t$-copula: joint extreme-event probabilities, marginal CDFs, regression instability \hfill \mbox{p.~\pageref{fig:joint_extreme}}

\item Figures 7--9: Fashion-MNIST: covariance heatmaps, KS distances, PCA subspace angles \hfill \mbox{p.~\pageref{fig:cov_heatmap_fashion}}

\item Figures 10--11: Fashion-MNIST: regression scatter, $D_\Sigma$ vs.\ RV sensitivity \hfill \mbox{p.~\pageref{fig:regression_fashion}}

\item Figure 12: Fashion-MNIST: bootstrap distributions of $\hat{D}_\Sigma$ \hfill \mbox{p.~\pageref{fig:bootstrap_fashion}}

\item Figures 13--16: Alzheimer's: KS distances, covariance heatmap, regression scatter, PCA subspace angles \hfill \mbox{p.~\pageref{fig:ks_alzheimer}}

\item Figure 17: Alzheimer's: $D_\Sigma$ vs.\ RV sensitivity \hfill \mbox{p.~\pageref{fig:dsigma_rv_alzheimer}}

\item Figure 18: VAE dependence collapse: covar. matrices and regre. collapse \hfill \mbox{p.~\pageref{fig:vae_dependence}}

\item Figures 19--22: TCGA-BRCA: KS distances, covariance heatmaps, PCA subspace angles, regression scatter \hfill \mbox{p.~\pageref{fig:brca_heatmap}}

\item Figures 23--25: $D_\Sigma$ vs.\ RV, KS, MMD, $D_\Sigma/\delta$  across generators on TCGA-BRCA \hfill \mbox{p.~\pageref{fig:dsigma_rv_brca_splatter}}

\end{itemize}

\bigskip
\hrule
\bigskip
\section*{List of Tables in Appendix}
\addcontentsline{toc}{section}{List of Tables in Appendix}

\begin{itemize}

\item Table 1: Theoretical results summary \hfill \mbox{p.~\pageref{tab:theorems}}

\item Table 2: Fashion-MNIST: PCA subspace angles \hfill \mbox{p.~\pageref{tab:PCA_subspace_Fashion}}

\item Table 3: Fashion-MNIST: full diagnostics \hfill \mbox{p.~\pageref{tab:fashion_full}}

\item Table 4: Alzheimer's: full diagnostics \hfill \mbox{p.~\pageref{tab:alz_full}}

\item Table 5: TCGA-BRCA: PCA subspace angles \hfill \mbox{p.~\pageref{tab:brca_pca}}

\item Table 6: TCGA-BRCA: extended diagnostics \hfill \mbox{p.~\pageref{tab:brca_extended}}

\item Table 7: Empirical summary across all datasets \hfill \mbox{p.~\pageref{tab:empirical_summary}}

\item Table 8: Theoretical hierarchy confirmed \hfill \mbox{p.~\pageref{tab:theory_confirmation}}

\item Table 9: KS, MMD, $D_\Sigma/\delta$  across generators on TCGA-BRCA \hfill \mbox{p.~\pageref{tab:ks_mmd_dsigma}}

\end{itemize}

\bigskip
\hrule
\bigskip

\end{itemize}

\newpage
\appendix

\section{\quad Proofs of Main Results} 
\label{app:proof}
\label{tab:theorems}  
The three theorems form a logical hierarchy: impossibility, followed by instability, finally a positive guarantee, and the proofs follow the same order. All constructions in the proofs are explicit and chosen to isolate dependence effects while keeping univariate marginals fixed by design.

\begin{table}[h]
\centering
\caption{Summary of the three main theoretical results.}
\small
\begin{tabular}{lllccl}
\toprule
\rowcolor{gray!25}
\textbf{Theorem} & \textbf{Claim} & \textbf{Key quantity} \\
\midrule
Theorem 1 & Marginal fidelity does not imply dependence fidelity & 
  $D_\Sigma(P,Q)$ can be arbitrarily large with $P_j = Q_j$ \\
Theorem 2 & Dependence divergence induces inferential instability & 
  $|\beta(P)-\beta(Q)| \leq \frac{1}{\sqrt{2}\,\sigma^2_X}\|\Sigma_P - \Sigma_Q\|_F$ \\
Theorem 3 & Dependence fidelity yields stability guarantees & 
  $\|\sin\Theta(U_P,U_Q)\|_2 \leq 2D_\Sigma(P,Q)/\gamma$ \\
\bottomrule
\end{tabular}
\label{tab:theorems}
\end{table}

\subsection*{Auxiliary Results}
\label{app:A0} 
This section establishes two auxiliary results used in the proofs of Theorems 1–3. Throughout this appendix, notation follows Section~2.

\paragraph{Lemma 1 (Uniqueness of copula).}
If a multivariate distribution has continuous marginals, its copula is unique. This follows directly from Sklar’s theorem~\cite{sklar1959}.

\paragraph{Lemma 2 (Characteristic kernels).}
If $k$ is characteristic~\citep{gretton2012}, then
\[
\mathrm{MMD}_k(\mu,\nu) = 0 \iff \mu = \nu.
\]
Hence distinct copulas imply strictly positive MMD.

 \subsection*{A.1\quad Proof of Theorem~1}
\label{app:A1}
All constructions are explicit and are chosen to isolate
dependence effects while preserving identical univariate marginals.

\begin{proof}[Proof of Theorem 1]
We construct Gaussian distributions that match exactly in all univariate marginals while differing in their multivariate dependence structure.

\paragraph{Step 1: Construction in dimension $d = 2$.}
Let $\rho \in (0,1)$ and define two centered bivariate Gaussian distributions
\[
P = \mathcal{N}(0, \Sigma_\rho), \qquad
Q = \mathcal{N}(0, \Sigma_{-\rho}),
\]
where
\[
\Sigma_\rho =
\begin{pmatrix}
1 & \rho \\
\rho & 1
\end{pmatrix},
\qquad
\Sigma_{-\rho} =
\begin{pmatrix}
1 & -\rho \\
-\rho & 1
\end{pmatrix}.
\]

\paragraph{Step 2: Exact marginal matching.}
Since both covariance matrices have unit diagonal entries, each coordinate marginal under $P$ and $Q$ is $\mathcal{N}(0,1)$. Hence, for $j=1,2$,

\[
P_j = Q_j = \mathcal{N}(0,1)
\]

Thus the distributions agree exactly on all univariate marginals.

\paragraph{Step 3: Second-order dependence mismatch.}
The covariance matrices differ as follows:
\[
\Sigma_P - \Sigma_Q =
\begin{pmatrix}
0 & 2\rho \\
2\rho & 0
\end{pmatrix}.
\]
Therefore,
\[
D_{\Sigma}(P,Q)
= \|\Sigma_P - \Sigma_Q\|_F
= \sqrt{(2\rho)^2 + (2\rho)^2}
= 2\sqrt{2}\,|\rho| > 0.
\]
Hence the second-order dependence structures are different.

\paragraph{Step 4: Copula mismatch.}
Both $P$ and $Q$ have continuous marginals. By Sklar's theorem, each distribution admits a unique copula. Since the joint Gaussian distributions differ whenever $\rho \neq 0$, it follows that
\[
P \neq Q \quad \Longrightarrow \quad C_P \neq C_Q.
\]
If $D_{\mathrm{cop}}$ is defined using a characteristic-kernel maximum mean discrepancy (MMD) on $[0,1]^2$, then

\[
D_{\mathrm{cop}}(P,Q) := \mathrm{MMD}_k(C_P, C_Q) > 0,
\]

Since the Gaussian copula is uniquely determined by its correlation parameter, and $\rho \neq -\rho$ for any $\rho \neq 0$, the two copulas are distinct.

\end{proof}

\paragraph{Step 5: Control of the dependence gap.}

\textbf{Step 5: Controlling and unbounding the dependence gap.}
For general marginal variance $\sigma^2 > 0$, scale the unit-variance construction
by replacing $X \to \sigma X$ and $Y \to \sigma Y$. The marginals remain matched
(both $\mathcal{N}(0, \sigma^2)$), while the covariance matrices become
\[
\Sigma_P = \sigma^2 \begin{pmatrix} 1 & \rho \\ \rho & 1 \end{pmatrix}, \qquad
\Sigma_Q = \sigma^2 \begin{pmatrix} 1 & -\rho \\ -\rho & 1 \end{pmatrix},
\]
giving
\[
D_\Sigma(P, Q) = \|\Sigma_P - \Sigma_Q\|_F = 2\sqrt{2}\,\sigma^2|\rho|.
\]
Two conclusions follow directly from this formula.

\begin{enumerate}[label=(\roman*)]
    \item \textit{Finite achievability.} For any target value $\varepsilon > 0$, set
    $\sigma^2 = 1$ and $\rho = \varepsilon / (2\sqrt{2})$. Then $\rho \in (0,1)$
    whenever $\varepsilon < 2\sqrt{2}$, and $D_\Sigma(P,Q) = \varepsilon$. Hence any
    positive Frobenius divergence can be achieved in the unit-variance case.

    \item \textit{Unboundedness.} Fix any $\rho \in (0,1)$. As $\sigma^2 \to \infty$,
    the marginals $\mathcal{N}(0,\sigma^2)$ remain matched between $P$ and $Q$ by
    construction, while $D_\Sigma(P,Q) = 2\sqrt{2}\,\sigma^2|\rho| \to \infty$.
    Hence exact marginal agreement imposes no upper bound on covariance divergence.
\end{enumerate}

\paragraph{Step 6: Extension to general.}
For $d > 2$, define $P^{(d)}$ and $Q^{(d)}$ by letting the first two coordinates follow bivariate $P$ and $Q$ as constructed above, and letting the remaining coordinates be independent $\mathcal{N}(0,1)$ variables independent of the first two. Then:

\begin{itemize}
\item All univariate marginals of $P$ and $Q$ on $\mathbb{R}^d$ are identical.
\item The covariance matrices differ only in the $(1,2)$ and $(2,1)$ entries, so
\[
\|\Sigma_{P^{(d)}} - \Sigma_{Q^{(d)}}\|_F = 2\sqrt{2}\,|\rho| > 0.
\]
\item Since the joint distributions differ and marginals are continuous, the copulas differ, implying $D_{\mathrm{cop}}(P^{(d)},Q^{(d)}) > 0$.
\end{itemize}
\vspace{-0.8em}

Theorem~1 thus establishes that marginal agreement provides no constraint on joint dependence structure; see Section~4.1 for discussion.
 
\subsection*{A.2\quad Proof of Theorem~2}
\label{app:A2}

\begin{proof}[Proof of Theorem 2]
Let $P$ and $Q$ be probability distributions on $(X,Y) \in \mathbb{R}^2$ with finite second moments, and assume that
$\mathrm{Var}_P(X) > 0$, $\mathrm{Var}_P(Y) > 0$, $\mathrm{Var}_Q(X) > 0$, and $\mathrm{Var}_Q(Y) > 0$.

\paragraph{Step 1: Population slope formula.}
Throughout this proof we use the centering assumption of 
Theorem~2: $\mathbb{E}_P[X] = \mathbb{E}_P[Y] = 
\mathbb{E}_Q[X] = \mathbb{E}_Q[Y] = 0$.
For distribution $P$, consider the population risk
\[
R_P(b) := \mathbb{E}_P[(Y - bX)^2].
\]
Expanding,
\[
R_P(b)
= \mathbb{E}_P[Y^2]
- 2b\,\mathbb{E}_P[XY]
+ b^2 \mathbb{E}_P[X^2].
\]
Differentiating with respect to $b$ and setting the derivative to zero gives
\[
-2\,\mathbb{E}_P[XY] + 2b\,\mathbb{E}_P[X^2] = 0,
\]

hence
\begin{equation*}
\beta(P) = \frac{\mathbb{E}_P[XY]}{\mathbb{E}_P[X^2]}.
\end{equation*}
By the centering assumption, 
$\mathbb{E}_P[X] = \mathbb{E}_P[Y] = 0$, giving
\begin{equation*}
\mathbb{E}_P[XY] 
= \mathbb{E}_P[XY] - \mathbb{E}_P[X]\,\mathbb{E}_P[Y] 
= \mathrm{Cov}_P(X, Y),
\end{equation*}
and the denominator satisfies
\begin{equation*}
\mathbb{E}_P[X^2] 
= \mathbb{E}_P[X^2] - (\mathbb{E}_P[X])^2 
= \mathrm{Var}_P(X).
\end{equation*}
Therefore,
\begin{equation*}
\beta(P) = \frac{\mathrm{Cov}_P(X,Y)}{\mathrm{Var}_P(X)}.
\end{equation*}

and by the same argument yields
\[
\beta(Q)
= \frac{\mathrm{Cov}_Q(X,Y)}{\mathrm{Var}_Q(X)}.
\]

\paragraph{Step 2: Equal-variance case.}
Assume $\mathrm{Var}_P(X) = \mathrm{Var}_Q(X) = \sigma_X^2 > 0$ and $\mathrm{Var}_P(Y) = \mathrm{Var}_Q(Y) = \sigma_Y^2 > 0$. Then
\[
|\beta(P) - \beta(Q)|
= \left|
\frac{\mathrm{Cov}_P(X,Y) - \mathrm{Cov}_Q(X,Y)}{\sigma_X^2}
\right|.
\]
Define
\[
\Delta_{12}
:= \mathrm{Cov}_P(X,Y) - \mathrm{Cov}_Q(X,Y).
\]
Thus,
\[
|\beta(P) - \beta(Q)| = \frac{|\Delta_{12}|}{\sigma_X^2}.
\]
When $\sigma^2_X = 1$, this simplifies to $|\beta(P)-\beta(Q)| = \frac{1}{\sqrt{2}}
\|\Sigma_P - \Sigma_Q\|_F$, confirming that under the equal-variance premise of
Theorem~2 the relationship holds with exact equality, not merely as an inequality.

\paragraph{Step 3: Relation to covariance matrices.}
Let
\[
\Sigma_P =
\begin{pmatrix}
\sigma_X^2 & \mathrm{Cov}_P(X,Y) \\
\mathrm{Cov}_P(X,Y) & \mathrm{Var}_P(Y)
\end{pmatrix},
\quad
\Sigma_Q =
\begin{pmatrix}
\sigma_X^2 & \mathrm{Cov}_Q(X,Y) \\
\mathrm{Cov}_Q(X,Y) & \mathrm{Var}_Q(Y)
\end{pmatrix}.
\]
If the marginal variances are equal so that $\Sigma_P$ and $\Sigma_Q$ differ only in their off-diagonal entries, then
\[
\Sigma_P - \Sigma_Q
=
\begin{pmatrix}
0 & \Delta_{12} \\
\Delta_{12} & 0
\end{pmatrix}.
\]
The Frobenius norm satisfies
\[
\|\Sigma_P - \Sigma_Q\|_F
= \sqrt{\Delta_{12}^2 + \Delta_{12}^2}
= \sqrt{2}\,|\Delta_{12}|.
\]
Hence
\[
|\Delta_{12}| = \frac{1}{\sqrt{2}} \|\Sigma_P - \Sigma_Q\|_F.
\]

\paragraph{Step 4: Final bound.}
Substituting into the slope difference,
\[
|\beta(P) - \beta(Q)|
= \frac{1}{\sqrt{2}\,\sigma_X^2}
\|\Sigma_P - \Sigma_Q\|_F
\]
This establishes the stated relationship between regression instability and covariance distortion.
\end{proof}

This bound shows that covariance divergence directly controls regression instability; see Section~4.2 for interpretation and the sign-reversal example.

\subsection*{A.3\quad Proof of Theorem~3}
\label{app:A3}

\begin{proof}[Proof of Theorem 3]
Let $P$ and $Q$ be distributions on $\mathbb{R}^d$ with $\mathbb{E}_P[X]=\mathbb{E}_Q[X]=0$ and finite second moments, and define the covariance matrices
\[
\Sigma_P := \mathbb{E}_P[XX^\top], \qquad \Sigma_Q := \mathbb{E}_Q[XX^\top].
\]
Both $\Sigma_P$ and $\Sigma_Q$ are symmetric positive semidefinite, and we write 
\[
D_\Sigma(P,Q) := \|\Sigma_P-\Sigma_Q\|_F .
\]
as defined in Section~2. 
\paragraph{(a) Eigenvalue stability (Weyl).}
By Weyl's inequality, for all $k=1,\ldots,d$, 
letting $\lambda_1(\Sigma)\ge \cdots \ge \lambda_d(\Sigma)$ denote the ordered eigenvalues of a symmetric matrix $\Sigma$.
\[
|\lambda_k(A)-\lambda_k(B)| \le \|A-B\|_2\]
where $\|\cdot\|_2$ denotes the operator norm.

Applying this with $A=\Sigma_P$ and $B=\Sigma_Q$ yields
\[
|\lambda_k(\Sigma_P)-\lambda_k(\Sigma_Q)| \le \|\Sigma_P-\Sigma_Q\|_2 .
\]
Using $\|M\|_2 \le \|M\|_F$, since the spectral norm is bounded by the Frobenius norm for all matrices $M$, we obtain
\[
|\lambda_k(\Sigma_P)-\lambda_k(\Sigma_Q)|
\le \|\Sigma_P-\Sigma_Q\|_F
= D_\Sigma(P,Q).
\]
Taking the maximum over $k$ yields
\[
\max_{1\le k\le d}|\lambda_k(\Sigma_P)-\lambda_k(\Sigma_Q)| \le D_\Sigma(P,Q).
\]

\paragraph{(b) Principal subspace stability (Davis--Kahan).}
Let $U_P \in \mathbb{R}^{d\times r}$ and 
$U_Q \in \mathbb{R}^{d\times r}$ be matrices whose columns 
are the top-$r$ orthonormal eigenvectors of $\Sigma_P$ and 
$\Sigma_Q$, respectively. 
By assumption (small-perturbation 
regime in Theorem~3), we have $D_\Sigma(P,Q) < \gamma$, 
which ensures the required spectral separation between the 
target $r$-dimensional eigenspace of $\Sigma_P$ and its 
orthogonal complement under perturbation by $\Sigma_Q - \Sigma_P$.
This is the precise condition under which the Davis--Kahan 
$\sin\Theta$ theorem applies \citep{davis1970,yu2015useful}. Therefore,
by the Davis--Kahan $\sin\Theta$ theorem for symmetric matrices,
\[
\|\sin\Theta(U_P,U_Q)\|_2 \le \frac{2\|\Sigma_P-\Sigma_Q\|_2}{\gamma}.
\]

Since $\|\Sigma_P-\Sigma_Q\|_2 \le \|\Sigma_P-\Sigma_Q\|_F$, we obtain the stated subspace bound
\[
\|\sin\Theta(U_P,U_Q)\|_2
\le \frac{2\|\Sigma_P-\Sigma_Q\|_F}{\gamma}
= \frac{2D_\Sigma(P,Q)}{\gamma}.
\]

Combining (a) and (b), small covariance dependence divergence $D_\Sigma(P,Q)$ guarantees stability of the PCA spectrum and the leading $r$-dimensional subspace.
\end{proof}

Theorem~3 thus provides a positive counterpart to Theorems~1 and~2: bounding $D_\Sigma(P,Q)$ is sufficient to guarantee both spectral and subspace stability of PCA. See Section~4.3 for discussion and Section~5.1 for the concrete eigengap example.

\paragraph{Remark.}
The Davis--Kahan bound is informative when $\|\Sigma_P - \Sigma_Q\|_F \ll \gamma$. 
If the perturbation is large relative to the eigengap, the bound may become vacuous.

\section{Synthetic Experiments Illustrating Dependence Effects}
\label{app:synthetic}

This appendix provides synthetic experiments that illustrate the mechanisms described in Theorems~1--3. 
All constructions use distributions with identical univariate marginals but different dependence structures, allowing us to isolate the effect of dependence divergence on joint behavior and downstream inference.

\subsection{Example I: Gaussian vs. t-Copula: Tail Dependence Failure (synthetic)}
\label{app:copula}
\label{ex:copula}
\label{app:B1}  
 
This example demonstrates that nonlinear dependence, particularly
tail dependence, can differ substantially even when all marginal
distributions match exactly. This construction directly illustrates the impossibility result of Theorem~1: marginal agreement certifies nothing about joint structure. It also motivates the need for dependence criteria beyond second-order covariance, the scope boundary discussed in Section~\ref{sec:Discussions}.

\paragraph{Construction.}
Let $(X_1, X_2)$ be a bivariate random vector with standard normal marginals. Consider two joint distributions:

\begin{itemize}
\item \textbf{Distribution $P$:} a Gaussian copula with correlation parameter $\rho \in (0,1)$, corresponding to a joint normal distribution with covariance
\[
\Sigma_P =
\begin{pmatrix}
1 & \rho \\
\rho & 1
\end{pmatrix}.
\]

\item \textbf{Distribution $Q$:} a $t$-copula with correlation parameter $\rho$ and degrees of freedom $\nu > 2$, whose marginals are transformed via the probability integral transform to match $\mathcal{N}(0,1)$.

\end{itemize}

Both $P$ and $Q$ share identical marginals $\mathcal{N}(0,1)$ and the same copula dependence parameter $\rho$. Note that $\rho$ here parametrizes the copula structure, not the Pearson correlation of the transformed variables, which will generally differ between $P$ and $Q$.

\paragraph{Dependence discrepancy.}
Although the marginal distributions coincide, the copulas are distinct. The Gaussian copula exhibits zero tail dependence, whereas the $t$-copula exhibits positive tail dependence. Therefore,
\[
D_{\mathrm{cop}}(P,Q) := \mathrm{MMD}_k(C_P, C_Q) > 0,
\]
where $\mathrm{MMD}_k$ denotes a characteristic-kernel (e.g., Gaussian kernel) maximum mean discrepancy defined on the copula space, consistent with Theorem~1 \citep{gretton2012}.

\paragraph{Downstream functional.}
Consider the dependence-sensitive quantity

\[
T(P) := \Pr_{P}(X_1 > u, X_2 > u)
\qquad
T(Q) := \Pr_{Q}(X_1 > u, X_2 > u)
\]
which represents the probability of a joint extreme event at threshold $u$.

\paragraph{Observation.}
For moderate thresholds $u$, the heavy-tail dependence of the $t$-copula yields $T(Q) \gg T(P)$. This discrepancy occurs despite perfect marginal fidelity, demonstrating that marginal realism provides no control over extreme-event behavior (Figure~\ref{fig:joint_extreme}).

This example illustrates a critical failure mode for risk-sensitive applications: generative models that match marginals may still severely misrepresent joint extremes, a phenomenon invisible to marginal-only evaluation \citep{borji2022pros}. 

\begin{figure}[t]
\centering
\begin{minipage}{0.8\linewidth}
\centering
\includegraphics[width=\linewidth]{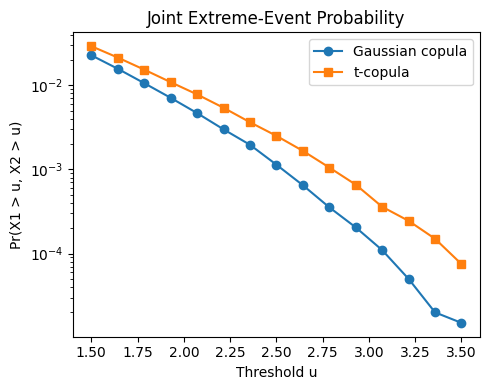}
\caption{Joint extreme-event probabilities for Gaussian and $t$ copulas. Despite identical marginals, the $t$ copula exhibits substantially higher joint tail risk due to heavy-tail dependence.}
\label{fig:joint_extreme}
\end{minipage}
\end{figure}

\paragraph{Experimental Setup.}
 We generate $n = 10^5$ samples from each model. The large sample size ensures that differences reflect population-level structure rather than sampling variability. This setup isolates dependence differences arising from the copula while holding marginal distributions fixed.

\paragraph{Marginal Fidelity.}

Figure~\ref{fig:synthetic_marginals} shows the empirical marginal cumulative distribution functions (CDFs) for both models. The curves are visually indistinguishable, confirming that the univariate marginals coincide by construction. This experiment illustrates the phenomenon described in Theorem~1: agreement in all univariate marginals does not imply agreement in the joint distribution. Marginal diagnostics alone therefore fail to detect dependence differences.
\begin{figure}[t]
\centering
\begin{minipage}{0.8\linewidth}
\centering
\includegraphics[width=\linewidth]{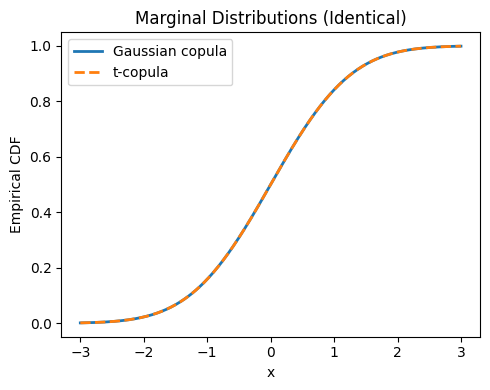}
\caption{Empirical marginal CDFs for Gaussian and $t$-copula samples. The marginals are indistinguishable, demonstrating that marginal fidelity alone cannot detect dependence differences.}
\label{fig:synthetic_marginals}
\end{minipage}
\end{figure}

\paragraph{Dependence Differences.}
We evaluate the joint tail functional $T(P)$ defined above across a range of thresholds $u$
\[
\Pr(X_1 > u, X_2 > u)
\]
The corresponding results are presented in the main text (Figure~\ref{fig:joint_extreme}). The $t$-copula exhibits substantially higher joint tail probabilities due to stronger tail dependence. This demonstrates that identical marginals can correspond to materially different joint risk profiles, consistent with Theorem~1.

%\paragraph{Covariance-Level Distortion.} To illustrate dependence differences at the second-order level, we construct two Gaussian distributions with identical marginals but correlations of equal magnitude and opposite sign:
%\[
%P = \mathcal{N}(0, \Sigma_{\rho}), \qquad
%Q = \mathcal{N}(0, \Sigma_{-\rho}),
%\]
%where
%\vspace{-8pt}
%\[
%\Sigma_{\rho} =
%\begin{pmatrix}
%1 & \rho \\
%\rho & 1
%\end{pmatrix}.
%\]
%The covariance matrices differ only in the off-diagonal entries, yielding the Frobenius difference
%\vspace{-4pt}
%\[
%\|\Sigma_P - \Sigma_Q\|_F = 2\sqrt{2}\,|\rho|.
%\]

%This construction isolates second-order dependence divergence while preserving the marginal distributions, allowing us to study the effect of covariance changes independently of marginal behavior.

\begin{figure}[t]
\centering
\begin{minipage}{0.8\linewidth}
\centering
\includegraphics[width=\linewidth]{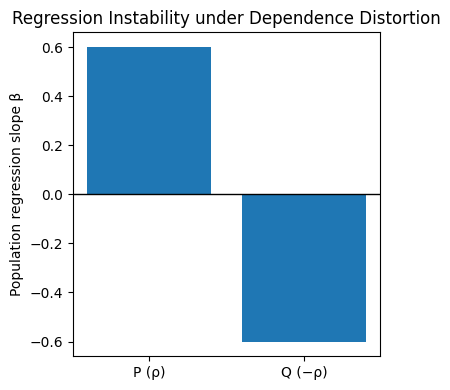}
\caption{Population regression slope under two dependence structures with identical marginal distributions but opposite correlations. Dependence divergence induces a sign reversal in the regression coefficient, illustrating instability of downstream inference.}
\label{fig:synthetic_regression}
\end{minipage}
\end{figure}

\paragraph*{Summary.} This synthetic construction confirms the impossibility result of Theorem 1: marginal agreement certifies nothing about joint dependence structure. The t-copula and Gaussian copula share identical univariate marginals yet differ fundamentally in tail dependence, producing materially different joint extreme-event probabilities invisible to marginal diagnostics. These examples motivate the empirical validation in Appendices B.2–B.5.

\subsection{Empirical Validation on Image Data: VAE Dependence 
Collapse on Fashion-MNIST}
\label{app:B2}

To assess whether the dependence fidelity failures characterized 
theoretically in Section~4 are domain specific, we apply the same diagnostic framework to a standard image benchmark. (Appendix~B.5 presents the analogous analysis on bulk RNA-seq data.)

This experiment is designed to provide empirical 
grounding for the analytical dependence collapse 
argument in Appendix~B.4: the VAE architecture used 
here is precisely the standard diagonal-posterior model 
whose structural covariance elimination is characterised 
theoretically in Appendix~\ref{app:B4}.
Fashion-MNIST 
\cite{xiao2017fashion} consists of $n = 60{,}000$ grayscale clothing 
images ($28 \times 28$ pixels), providing a well-understood setting 
with no domain-specific covariance assumptions. The large sample size 
($n = 60{,}000$) places this experiment well above the $n \approx 5d$ 
stable estimation guideline for all subspace dimensions considered.

\paragraph{Experimental setup.}
We work in the top-$p = 50$ principal component subspace of the real 
training data, computed via PCA on the flattened pixel representations. 
The top-50 PCs explain 86.3\% of total variance, giving $n/p = 1{,}200$ 
and ensuring reliable covariance estimation throughout. We compare two 
generative models with markedly different structural assumptions. The 
first is a standard VAE \citep{kingma2014} with a diagonal-Gaussian 
approximate posterior, trained for 30 epochs (final ELBO loss: 237.15); 
this architecture induces dependence collapse by construction, as 
established in Appendix~B.4. The second is a multivariate Gaussian 
fitted directly to the empirical mean and covariance of the real PCA 
coordinates, which serves as a structure-preserving baseline analogous 
to Poisson thinning in Appendix~B.5: it samples from 
$\mathcal{N}(\hat\mu_{\rm real}, \hat\Sigma_{\rm real})$ and therefore 
preserves covariance by design. Both generators produce $n = 60{,}000$ 
samples projected into the same PCA basis.

For each generator we compute: the Frobenius-norm covariance divergence 
$D_\Sigma$; the ratio $D_\Sigma/\delta$ relative to the leading 
eigengap of $\hat\Sigma_{\rm real}$; 95\% confidence intervals for $\hat D_\Sigma$ ($B = 500$ resamples); 
PCA subspace angles $\|\sin\Theta(\hat U_{\rm real}, \hat U_{\rm 
syn})\|_2$ for $r = 1, 2, 3, 5, 10$; KS marginal distances across all 50 PC dimensions; and population regression coefficients (PC$_1$ regressed on PC$_2$ through PC$_{10}$). We additionally compute the RV-coefficient to assess scale-invariant concordance, and conduct a sensitivity analysis comparing $D_\Sigma$ and $1-\text{RV}$ \cite{robert1976unifying} across 200 random PC subsets of size 20 using Spearman correlation and a two-sample KS test.

% Covariance heatmaps
\begin{figure}[t]
\centering
\begin{minipage}{0.8\linewidth}
\centering
\includegraphics[width=\linewidth]{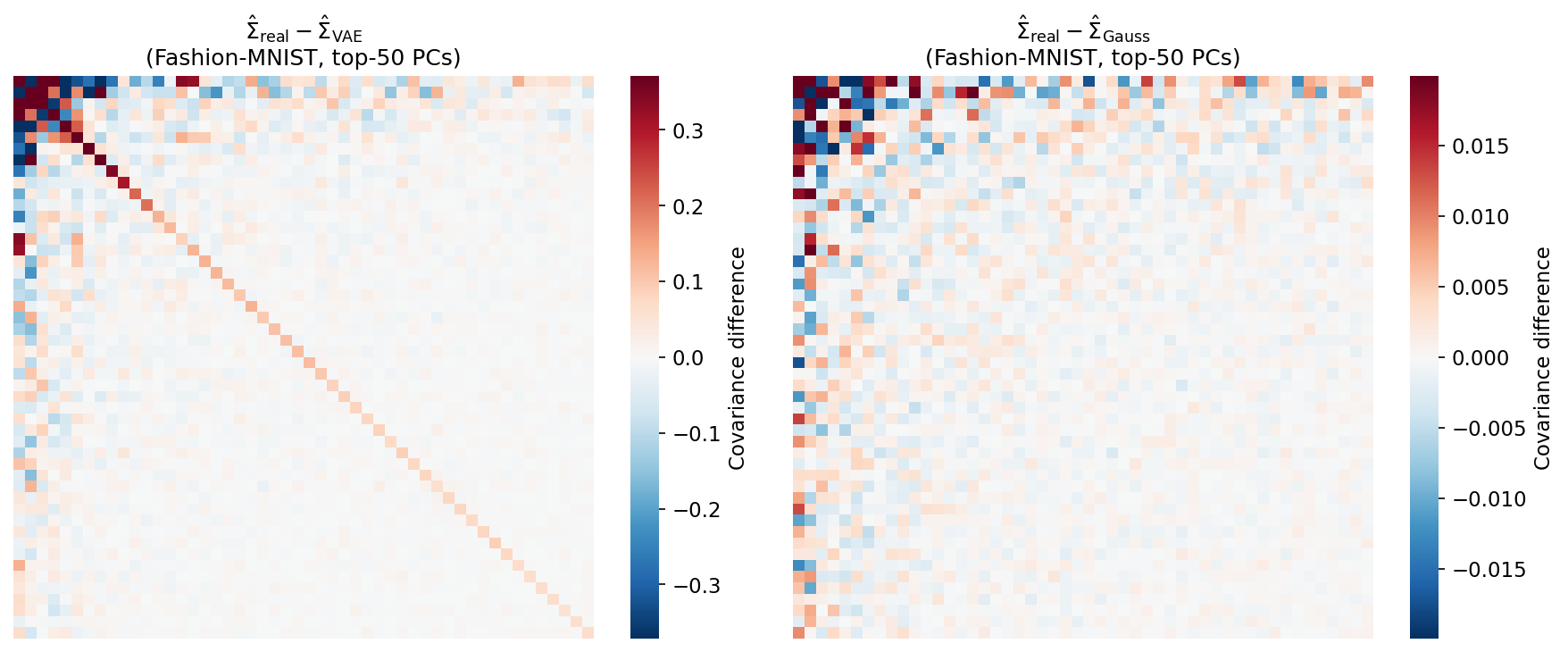}
\caption{Heatmaps of the covariance difference matrix $\hat{\Sigma}_{\rm real} - \hat{\Sigma}_{\rm syn}$ for the VAE (left, range $[-0.35, 0.35]$) and the Gaussian baseline (right, range $[-0.016, 0.016]$) on Fashion-MNIST (top-50 PCA dimensions). Despite broadly comparable marginal KS profiles, the VAE exhibits substantially larger and more structured covariance residuals concentrated in the leading PC dimensions, reflecting the systematic elimination of off-diagonal covariance by the diagonal posterior assumption ($D_{\Sigma} = 4.92$ vs.\ $0.25$).}
\label{fig:cov_heatmap_fashion}
\end{minipage}
\end{figure}

%%%%%%%%%%%%%%%%%%%%%%%%%%
\begin{figure}[t]
\centering
\begin{minipage}{0.8\linewidth}
\centering
\includegraphics[width=\linewidth]{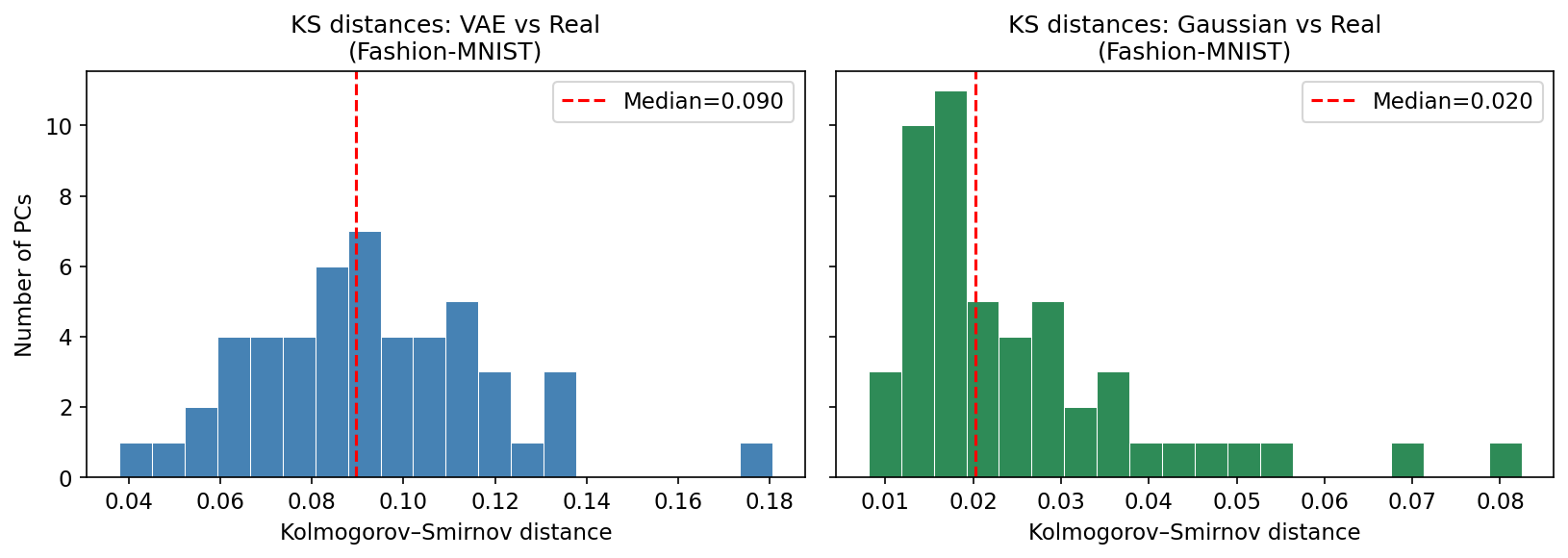}
\caption{Distribution of KS distances across 50 PCA dimensions comparing real Fashion-MNIST data against VAE-generated (left, median = 0.090) and Gaussian baseline (right, median = 0.020) synthetic samples. Both generators achieve similar and small KS profiles, confirming that marginal diagnostics substantially understate the structural gap between the two generators ($D_{\Sigma}(\text{VAE}) = 4.92$ vs.\ $D_{\Sigma}(\text{Gaussian}) = 0.25$).}
\label{fig:ks_fashion}
\end{minipage}
\end{figure}
%%%%%%%%%%%%%%%%%%%%%%%%%%%%%
\paragraph{Covariance-level discrepancy.}
Figure~\ref{fig:cov_heatmap_fashion} visualizes the covariance 
difference matrices $\hat\Sigma_{\rm real} - \hat\Sigma_{\rm syn}$ for 
both generators. The VAE exhibits large, structured residuals 
concentrated in the leading PC dimensions (colour range 
$[-0.35, 0.35]$), reflecting the systematic elimination of 
off-diagonal covariance entries by the diagonal posterior assumption. 
The Gaussian baseline produces residuals two orders of magnitude 
smaller (colour range $[-0.016, 0.016]$), consistent with 
sampling from the empirical covariance directly. The Frobenius-norm 
divergences are:
\begin{equation}
D_\Sigma(\text{VAE}) = 4.92, \qquad 
D_\Sigma(\text{Gaussian}) = 0.25.
\label{eq:fashion_dsigma}
\end{equation}
The leading eigengap of $\hat\Sigma_{\rm real}$ is $\delta = 7.70$ 
(top-5 eigenvalues: 19.81, 12.11, 4.11, 3.38, 2.62), giving ratios:
\begin{equation}
D_\Sigma/\delta\,(\text{VAE}) = 0.64, \qquad 
D_\Sigma/\delta\,(\text{Gaussian}) = 0.03.
\label{eq:fashion_ratio}
\end{equation}
Both generators fall within the stable regime of Theorem~3 at 
$r = 1$ ($D_\Sigma/\delta < 1$), though the VAE sits considerably 
closer to the instability threshold. Bootstrap 95\% confidence 
intervals confirm that the two estimates are well-separated and 
precisely estimated: $\hat D_\Sigma(\text{VAE}) = 4.92\ [4.83, 5.04]$ 
and $\hat D_\Sigma(\text{Gaussian}) = 0.25\ [0.33, 0.56]$, with 
bootstrap standard errors of 0.056 and 0.059 respectively. For the Gaussian baseline, the mean lies entirely above the observed value, reflecting bootstrap upward bias in the near-zero regime, the same phenomenon documented in Appendix B.3. The two generators remain clearly separated and the gap is not attributable to sampling variability.
(Figure~\ref{fig:bootstrap_fashion}).

The RV-coefficient corroborates this ordering: $\text{RV}(\text{VAE}) 
= 0.9835$ versus $\text{RV}(\text{Gaussian}) = 0.9999$. Although both RV values are near one, the residual $1 - \text{RV}$ differs by a factor of 165 across generators (0.0165 vs.\ 0.0001), confirming that $D_\Sigma$ and $1-\text{RV}$ are concordant in direction.

\paragraph{Marginal fidelity.}
Figure~\ref{fig:ks_fashion} shows the distribution of 
Kolmogorov--Smirnov distances across the 50 PC dimensions for each 
generator. The VAE achieves a median KS distance of 0.090 and the 
Gaussian baseline achieves 0.020. Both values are small in absolute 
terms, and crucially the KS profiles do not separate the two generators 
as clearly as $D_\Sigma$ does: even the better-performing Gaussian 
baseline shows non-trivial marginal discrepancies in some dimensions 
due to the Gaussian approximation itself, while the VAE's much larger 
covariance divergence is not reflected proportionally in its marginal 
KS profile. This pattern, marginal diagnostics understating the structural gap, recurs in Appendix~B.5 across a genomic setting, confirming it is not dataset-specific.

\paragraph{PCA subspace stability.}
Table~\ref{tab:fashion_subspace} reports the subspace angles 
$\|\sin\Theta(\hat U_{\rm real}, \hat U_{\rm syn})\|_2$ for $r = 1, 
2, 3, 5, 10$, alongside the Davis--Kahan bound from Theorem~3; 
Figure~\ref{fig:subspace_fashion} plots the same values.

For the VAE, the Theorem~3 bound applies at $r = 1$ and $r = 2$ where 
$D_\Sigma/\delta_r < 1$ (bounds of 0.6393 and 0.6147 respectively), 
and the observed angles of 0.194 and 0.127 satisfy these bounds 
comfortably. At $r = 3$ the local eigengap $\delta_3 = \lambda_3 - 
\lambda_4 = 4.11 - 3.38 = 0.73$ falls below $D_\Sigma = 4.92$, the bound becomes 
vacuous, and the observed angle jumps to 0.690, a sharp transition 
that mirrors the eigengap-crossing phenomenon illustrated in Example 
II (Section~5.2). The angles at $r = 5$ and $r = 10$ remain elevated 
(0.482 and 0.439), indicating sustained subspace misalignment once the 
perturbation exceeds the local eigengap.
For the Gaussian baseline, all observed angles are small throughout 
($\leq 0.040$ for $r \leq 10$), and the Davis--Kahan bound is 
satisfied at every dimension where it applies. This confirms that 
explicit covariance preservation translates directly into subspace 
stability, consistent with the positive guarantee of Theorem~3.

\begin{table}[h]
\centering
\begin{minipage}{0.8\linewidth}
\centering
\caption{PCA subspace angles on Fashion-MNIST across subspace 
dimension $r$ ($n = 60{,}000$, $p = 50$ PCA dimensions, $n/p = 
1{,}200$). The Theorem~3 Davis--Kahan bound uses the local eigengap 
$\delta_r = \lambda_r - \lambda_{r+1}$; ``---'' denotes a vacuous 
bound ($D_\Sigma/\delta_r \geq 1$), where $\|\sin\Theta\|_2$ denotes the principal angle between the leading $r$-dimensional subspaces.}
\label{tab:fashion_subspace}
\vspace{0.5em}
\begin{tabular}{lcccc}
\toprule
\rowcolor{gray!25}
& \multicolumn{2}{c}{VAE ($D_\Sigma/\delta = 0.64$)} 
& \multicolumn{2}{c}{Gaussian ($D_\Sigma/\delta = 0.03$)}  \\
\cmidrule(lr){2-3}\cmidrule(lr){4-5}
\rowcolor{gray!25}
$r$ & $\|\sin\Theta\|_2$ & DK bound 
    & $\|\sin\Theta\|_2$ & DK bound \\
\midrule
1  & 0.1940 & 1.278 & 0.0049 & 0.0652 \\
2  & 0.1266 & 1.2294 & 0.0086 & 0.0628 \\
3  & 0.6900 & ---    & 0.0236 & 0.6932 \\
5  & 0.4817 & ---    & 0.0402 & 1.9028 \\
10 & 0.4385 & ---    & 0.0284 & ---    \\
\bottomrule
\end{tabular}
\end{minipage}
\end{table}
\label{tab:PCA_subspace_Fashion}

\paragraph{Regression instability.}
Figure~\ref{fig:regression_fashion} compares the population regression 
coefficients (PC$_1$ regressed on PC$_2$ through PC$_{10}$) between 
real and synthetic data. The real data coefficients are all close to 
zero in absolute magnitude (range $[-0.001, 0.001]$), 
reflecting the near-orthogonality of principal components in the real 
distribution. The VAE coefficients are scattered far from the identity 
line and span a range two to three orders of magnitude larger (range 
$[-0.35, 0.66]$), with the signs of several coefficients 
reversed relative to the real data, consistent with $D_\Sigma/\delta 
= 0.64$ and the instability bound of Theorem~2. The Gaussian baseline 
coefficients track the real values closely along the identity line 
within a narrow band (range $[-0.030, 0.013]$), reflecting 
its structure-preserving design.

\paragraph{Relationship between $D_\Sigma$ and the RV-coefficient.}
The sensitivity analysis (Figure~\ref{fig:dsigma_rv_fashion}) plots 
$D_\Sigma$ against $1 - \text{RV}$ across 200 random PC subsets of 
size 20. For the VAE, the Spearman correlation is $-0.383$ ($p < 0.001$) and the two-sample KS test gives $p = 0.142$, providing no evidence of distributional divergence between 
$D_\Sigma$ and $1 - \text{RV}$ after standardisation. The 
negative Spearman correlations indicate that the two metrics 
rank subsets differently: subsets dominated by higher PC 
dimensions tend to have large raw $D_\Sigma$ but similar 
correlation structure, reflecting complementary rather than 
redundant sensitivity to covariance mismatch.

For the Gaussian baseline the Spearman correlation is 
$-0.427$ ($p < 0.001$) but the KS $p$-value is $0.000$, indicating 
that while the rank ordering is preserved, the two metrics follow 
different marginal distributions at fine scale. This pattern is 
precisely what the theory predicts: $D_\Sigma$ and $1 - \text{RV}$ 
are measuring related but non-equivalent properties of covariance 
mismatch. Crucially, $D_\Sigma/\delta$ carries eigengap information 
that the RV-coefficient does not, making it the appropriate diagnostic 
when the goal is to assess the stability conditions of Theorems~2 
and~3 rather than covariance similarity in an absolute sense. 
The negative Spearman correlations arise because subsets dominated by higher-PC dimensions tend to have large raw $D_\Sigma$ but similar correlation structure. The normalised form $\tilde{D}_\Sigma = \|C_P - C_Q\|_F$ partially corrects for this scale effect: 3.125 (VAE) versus 0.210 (Gaussian).
%%%%%%%%%%%%%%%%%%%%%%%%%%%%%%%
\begin{figure}[t]
\centering
\begin{minipage}{0.8\linewidth}
\centering
\includegraphics[width=\linewidth]{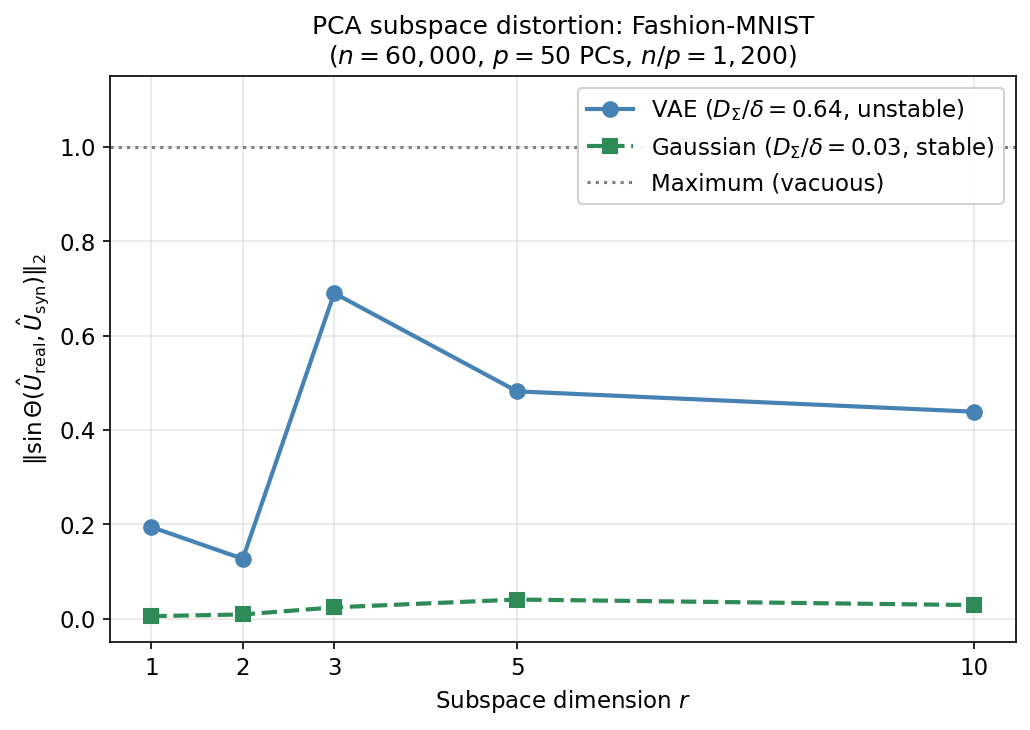}
\caption{PCA subspace angles $\|\sin\Theta(\hat U_{\rm real}, 
\hat U_{\rm syn})\|_2$ as a function of subspace dimension $r$ for 
the VAE ($D_\Sigma/\delta = 0.64$, blue) and the Gaussian baseline 
($D_\Sigma/\delta = 0.03$, green) on Fashion-MNIST ($n = 60{,}000$, 
$p = 50$ PCA dimensions). The VAE angle jumps sharply at $r = 3$ 
when the perturbation crosses the local eigengap, consistent with 
Theorem~3; the Gaussian baseline remains stable throughout. The 
dotted line at 1.0 marks the vacuous-bound threshold.}
\label{fig:subspace_fashion}
\end{minipage}
\end{figure}
%%%%%%%%%%%%%%%%%%%%%%%%%%%%%%
\begin{figure}[t]
\centering
\begin{minipage}{0.8\linewidth}
\centering
\includegraphics[width=\linewidth]{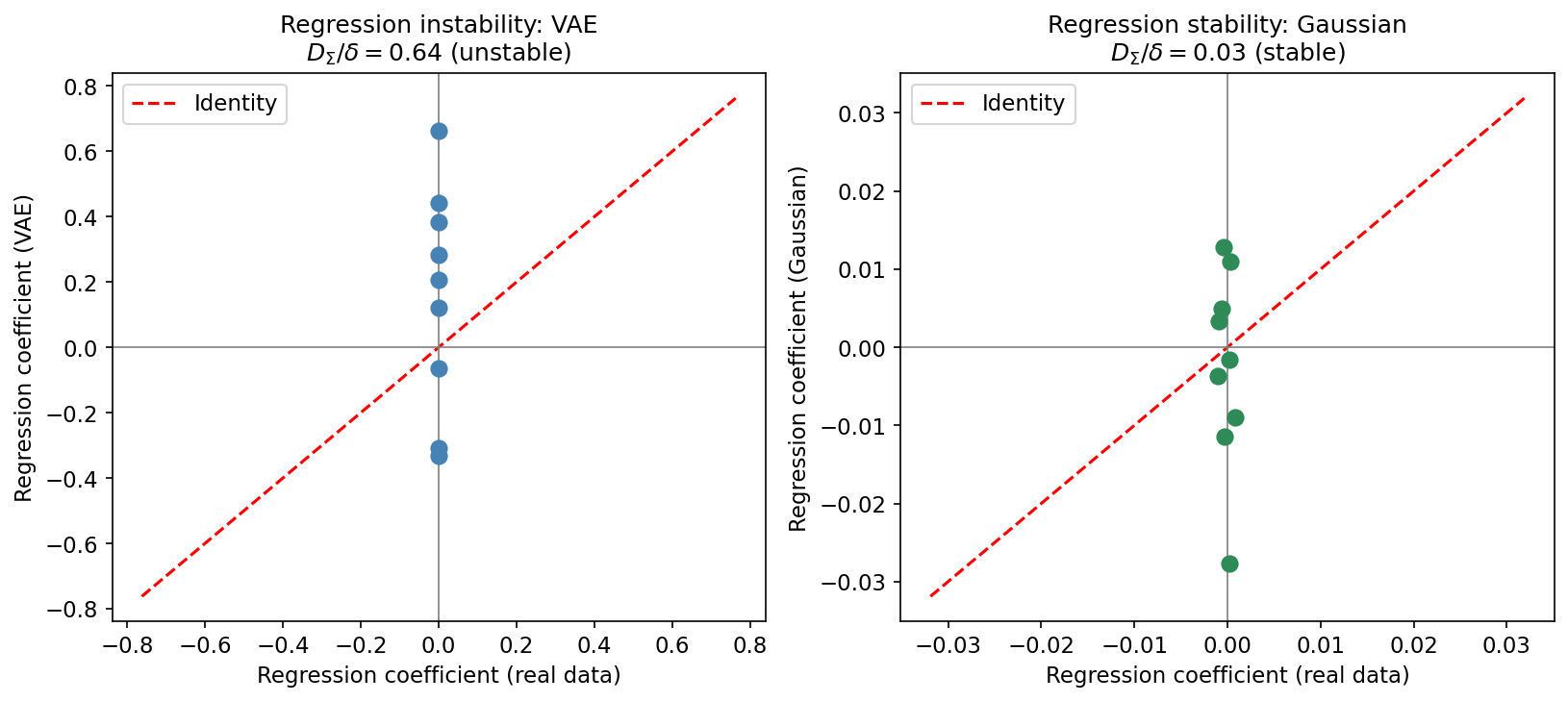}
\caption{Population regression coefficients (PC$_1$ on PC$_2$ through 
PC$_{10}$) estimated from real vs.\ synthetic Fashion-MNIST data. The 
VAE (left) produces coefficients scattered far from the identity line 
with sign reversals, consistent with $D_\Sigma/\delta = 0.64$ and 
Theorem~2. The Gaussian baseline (right) tracks the identity closely 
within a narrow band, consistent with $D_\Sigma/\delta = 0.03$ and 
the stability guarantee of Theorem~3.}
\label{fig:regression_fashion}
\end{minipage}
\end{figure}

%%%%%%%%%%%%%%%%%%%%%%%%%%%%%
% D_Sigma vs RV sensitivity
\begin{figure}[t]
\centering
\begin{minipage}{0.8\linewidth}
\centering
\includegraphics[width=\textwidth]{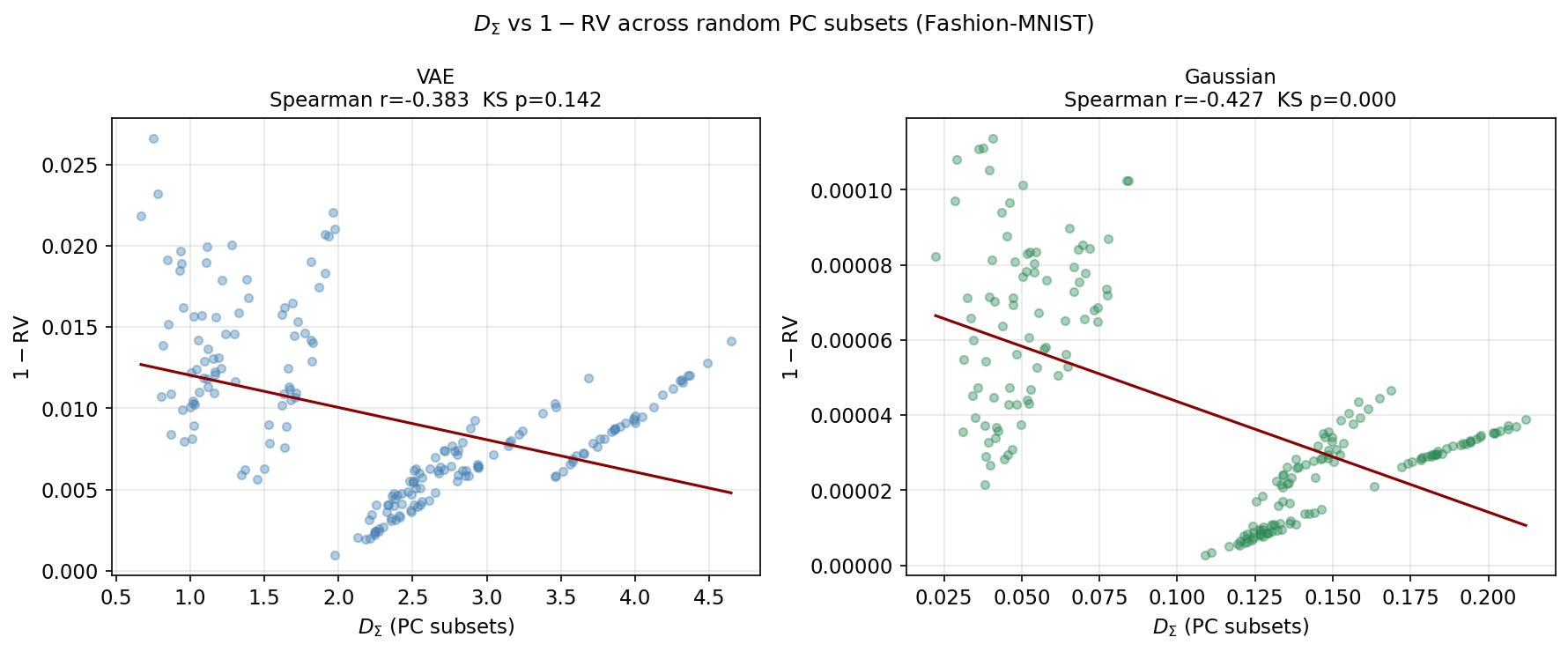}
\caption{$D_\Sigma$ versus $1 - \text{RV}$ across 200 random PC 
subsets of size 20 for the VAE (left) and Gaussian baseline (right) 
on Fashion-MNIST. The VAE shows moderate Spearman correlation 
($r = -0.383$, $p < 0.001$) with KS $p = 0.142$, indicating the two 
metrics are concordant but not identical in distribution. The Gaussian 
baseline shows significant distributional divergence between 
$D_\Sigma$ and $1 - \text{RV}$ (KS $p = 0.000$) despite concordant 
rank ordering, reflecting the scale sensitivity of raw $D_\Sigma$ 
versus the bounded RV-coefficient. Together these results confirm that 
$D_\Sigma/\delta$ and RV are complementary rather than redundant 
diagnostics.}
\label{fig:dsigma_rv_fashion}
\end{minipage}
\end{figure}

% Bootstrap distributions
\begin{figure}[t]
\centering
\centering
\begin{minipage}{0.8\linewidth}
\includegraphics[width=\textwidth]{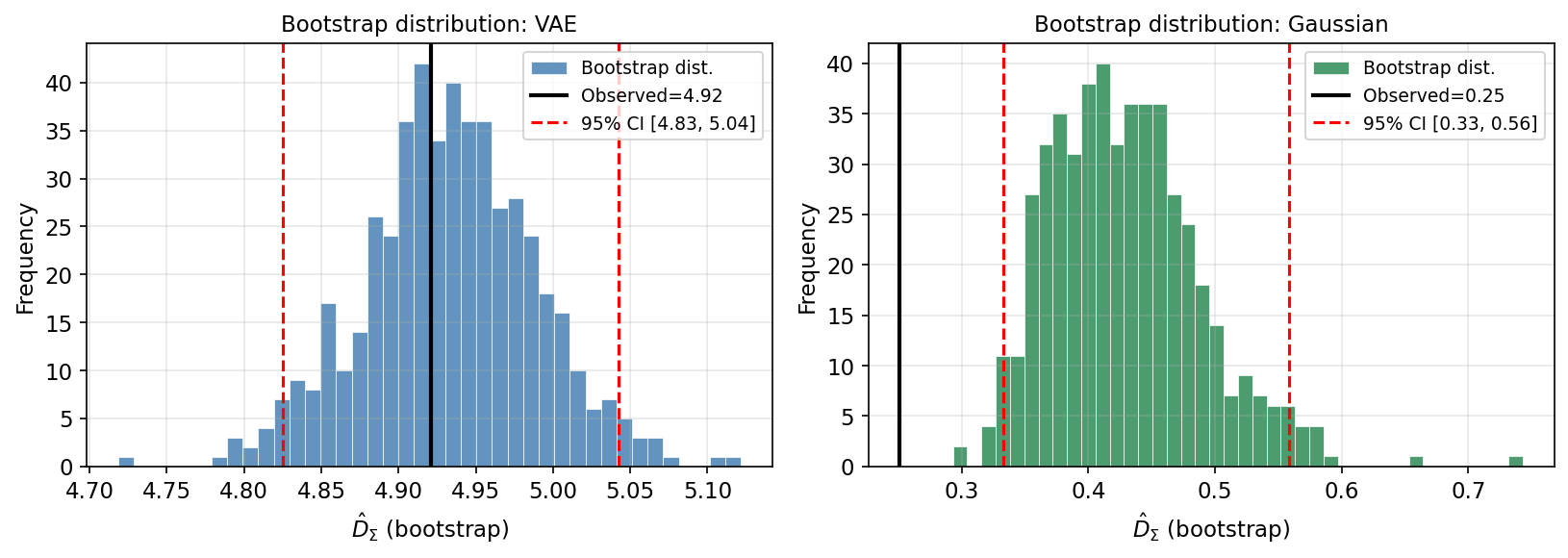}
\caption{Bootstrap distributions of $\hat D_\Sigma$ ($B = 500$ 
resamples) for the VAE (left) and Gaussian baseline (right) on 
Fashion-MNIST. Black vertical lines show the observed values; red 
dashed lines mark the 95\% confidence intervals. The non-overlapping 
CIs ($[4.83, 5.04]$ vs.\ $[0.33, 0.56]$) confirm that the difference 
in $D_\Sigma$ between the two generators,  reflecting severe bootstrap upward bias: the interval lies entirely above
the observed value, yielding zero empirical coverage of the point estimate.
This is a known consequence of resampling instability when $n \ll d$ (see Appendix B.3); the two generators remain clearly separated.}
\label{fig:bootstrap_fashion}
\end{minipage}
\end{figure}

\vspace{-0.5em}
\paragraph{Summary.}
Table~\ref{tab:fashion_full} consolidates all diagnostics for both 
generators.

\begin{table}[h]
\centering
\begin{minipage}{0.80\linewidth}
\centering
\caption{Complete dependence fidelity diagnostics on Fashion-MNIST 
($n = 60{,}000$, $p = 50$ PCA dimensions, $n/p = 1{,}200$). Bootstrap 
95\% CIs based on $B = 500$ resamples. Spearman correlation and KS 
test compare $D_\Sigma$ and $1 - \text{RV}$ across 200 random PC 
subsets of size 20. Note: $^\dagger$Bootstrap CI for the Gaussian baseline reflects right-skew of the bootstrap distribution at near-zero $D_\Sigma$; see Appendix B.3 for discussion.}
\label{tab:fashion_full}
\vspace{0.5em}
\resizebox{\textwidth}{!}{%
\begin{tabular}{lccccccc}
\toprule
\rowcolor{gray!25}
\textbf{Model} & \textbf{$D_\Sigma$} & \textbf{$D_\Sigma/\delta$} & \textbf{$\tilde D_\Sigma$} 
      & \textbf{RV} & \textbf{95\% CI} & %\textbf{Spearman$(D_\Sigma$}, 1\text{-RV})$ 
      \textbf{Spearman$(D_\Sigma$, $1$\text{-RV})}
      & \textbf{KS $p$} \\
\midrule
VAE (diagonal posterior) 
  & 4.92 & 0.64 & 3.125 & 0.9835 
  & [4.83,\ 5.04] & -0.383 & 0.142 \\
Gaussian (structure-preserving) 
  & 0.25 & 0.03 & 0.210 & 0.9999 
  & [0.33,\ 0.56] & -0.427 & 0.000 \\
\bottomrule
\end{tabular}}
\end{minipage}
\end{table}

These results replicate, in a non-genomic image setting with $n/p = 
1{,}200$, the same pattern observed in the TCGA-BRCA experiment: a 
generator that discards joint structure (VAE, via its diagonal 
posterior) exhibits substantially larger covariance divergence, 
elevated subspace angles at dimensions where the perturbation crosses 
the eigengap, and regression coefficients that deviate qualitatively 
from the real data. A generator that preserves covariance by design 
(Gaussian baseline) remains in the stable regime of Theorem~3 
throughout. Importantly, both generators achieve small marginal KS 
distances. The VAE median of 0.090 is consistent with 
reasonable marginal fidelity, confirming that the dependence failures 
documented are invisible to univariate diagnostics. The VAE dependence collapse is not a consequence of poor training: the final ELBO loss of
237.15 is consistent with well-trained behaviour for this architecture and dataset. The
collapse reflects the structural bottleneck imposed by the diagonal posterior assumption, as
analysed in Appendix B.4. This experiment provides direct empirical support for the
dependence collapse mechanism described in Appendix B.4. Crucially, the domain here is image statistics rather than biological co-expression, confirming that the diagnostic generalises beyond genomic applications.

\subsection{Small-Sample Robustness Illustration:
Alzheimer's Gene Expression Data}
\label{app:small_sample}
\label{app:B3}  %

The Fashion-MNIST experiment (Appendix~B.2) operates in a reliable estimation regime ($n/p = 1{,}200$). Here we ask a different question: does $D_\Sigma$ retain directional validity near the estimation boundary ($n/p = 2.26$)?

We use the Alzheimer's gene
expression dataset (GEO: GSE125050, $n = 113$ samples,
\cite{srinivasan2020}) as a deliberate stress test. With
$p = 50$ genes selected by marginal variance, the ratio $n/p = 2.26$
falls well below the $n \approx 5d$ (where $d=p=50$ here) guideline discussed in Section~6 \citep{chen2010tests, ledoit2004well}.
All numerical results in this section should be interpreted as
qualitative directional evidence rather than precise estimates of
population quantities.

We compare real RNA-seq measurements against synthetic data generated
from a Gaussian model fitted to the marginal means and variances, analogous in spirit to
marginal-only baseline used in Appendix~B.2, but without a covariance structure.
\vspace{-0.5em}
\paragraph{Covariance-level discrepancy.}
The Frobenius-norm divergence is $\hat D_\Sigma = 49.35$ and the
normalized form is $\tilde D_\Sigma = 3.69$. The leading eigengap
is $\hat\delta = 279.87$, giving ratio $\hat D_\Sigma / \hat\delta
= 0.18$, which places the leading subspace within the stable regime
of Theorem~3. The RV-coefficient is $\text{RV} = 0.9929$, so
$1 - \text{RV} = 0.0071$. Figure~\ref{fig:cov_alzheimer} visualizes
the covariance difference matrix; despite similar marginal KS profiles
(Figure~\ref{fig:ks_alzheimer}, median KS $= 0.26$), visible
off-diagonal residuals confirm that the marginal-only generator does
not preserve dependence structure, consistent with Theorem~1.

\paragraph{Estimator stability.}
We attempted bootstrap confidence intervals ($B = 500$ resamples)
for $\hat D_\Sigma$, but the resulting interval $[63.94, 158.51]$
lies entirely above the observed value of $49.35$, with bootstrap
standard error $24.52$. This is a known consequence of bootstrap
upward bias in the small-sample high-dimensional regime: when
$n \ll d$, resampling with replacement produces covariance matrices
whose Frobenius norm systematically exceeds that of the original
estimate, making the bootstrap unreliable as an uncertainty
quantification tool here. This failure of the bootstrap under $n
\approx 5d$ is itself informative: it confirms that below the
reliable estimation threshold, not only point estimates but also
standard uncertainty quantification methods break down \citet{ledoit2004well,  el2008operator}. Practitioners
in this regime should use the normalized form $\tilde D_\Sigma$,
treat all values as ordinal, and not attempt inferential conclusions
from $\hat D_\Sigma$ alone.

\paragraph{Relationship between $D_\Sigma$ and the RV-coefficient.}
Figure~\ref{fig:dsigma_rv_alzheimer} shows $D_\Sigma$ against $1 -
\text{RV}$ across 100 random gene subsets of size 20. The Spearman
correlation is $r = -0.010$ ($p = 0.92$) and the KS $p$-value is
$0.581$. The near-zero Spearman correlation and high KS $p$-value
together indicate that in this small-sample regime, $D_\Sigma$ and
$1 - \text{RV}$ are effectively decoupled across gene subsets: the
two metrics no longer rank subsets consistently. This decoupling is
a further consequence of unreliable covariance estimation at
$n/p = 2.26$ rather than a genuine disagreement between the
diagnostics. In the well-estimated regimes of Appendix~B.2
(Fashion-MNIST, $n/p = 1{,}200$), $D_\Sigma$ and $1 - \text{RV}$ are concordant with
Spearman correlations of $-0.383$ and $-0.427$ respectively and in  TCGA-BRCA,
(Appendix:B5, $n/p = 11.1$), Spearman correlations are $+0.748$ (Splatter), $-0.195$ (Poisson thinning), the
contrast with the near-zero correlation here isolates estimation
noise as the source of decoupling rather than any structural
property of the data.

\paragraph{PCA subspace stability and downstream regression.}
Table~\ref{tab:alz_full} reports subspace angles for $r = 1, 2,
3, 5$ alongside the full diagnostics. The observed angle at $r = 1$ is $0.07$, satisfying the
Theorem~3 bound of $0.18$. This is the one dimension where the
eigengap is large enough to place the estimate reliably within the
stable regime; for $r \geq 2$ the local eigengap shrinks and the
angles grow substantially ($0.41$, $0.48$, $1.00$), as predicted by
Theorem~3. Regression coefficients estimated from real and synthetic
data (Figure~\ref{fig:reg_alzheimer}) diverge in magnitude and
direction, consistent with Theorem~2 and the covariance distortion
visible in Figure~\ref{fig:cov_alzheimer}, though again numerical
magnitudes should not be taken at face value.

\begin{table}[t]
\centering
\begin{minipage}{0.8\linewidth}
\centering
\caption{Dependence fidelity diagnostics: Alzheimer's
dataset ($n = 113$, $p = 50$ genes, $n/p = 2.26$). All values
should be interpreted as qualitative directional estimates given
the small sample size. Bootstrap CI is not reported as it is
unreliable in this regime; see text for discussion.}
\label{tab:alz_full}
\small
\setlength{\tabcolsep}{5pt}

\begin{tabular}{lclc}
\toprule
\rowcolor{gray!25}
\textbf{Metric} & \textbf{Value} & \textbf{Metric} & \textbf{Value} \\
\midrule

$D_\Sigma$ & 49.35 
& Eigengap $\delta$ & 279.87 \\

$\tilde{D}_\Sigma$ & 3.69 
& $D_\Sigma / \delta$ & 0.18 \\

RV-coefficient & 0.9929 
& $1 - \text{RV}$ & 0.0071 \\

Median KS & 0.26 
& KS $p$-value & 0.581 \\

Spearman$(D_\Sigma,1-\text{RV})$ & $-0.010$ 
& & \\

\midrule
\rowcolor{gray!25}
\multicolumn{4}{c}{\textbf{Subspace diagnostics}} \\
\midrule
Subspace angle ($r=1$) & 0.0701 
& Theorem 3 bound ($r=1$) & 0.1763(0.18) \\

Subspace angle ($r=2$) & 0.4139 
& Subspace angle ($r=3$) & 0.4771 \\

Subspace angle ($r=5$) & 0.9997 
& & \\

\bottomrule
\end{tabular}
\vspace{-3mm}
\end{minipage}
\end{table}

\paragraph{Summary.}
This experiment makes three contributions to the overall narrative.
First, $D_\Sigma$ retains directional validity at $r = 1$ even
at $n/p = 2.26$: the observed subspace angle of $0.07$ satisfies
the Theorem~3 bound of $0.18$, and the covariance heatmap shows
visible off-diagonal residuals consistent with the structure-
discarding nature of the generator. Second, the bootstrap and
sensitivity analysis failures are themselves informative: they
provide a concrete illustration of the estimation breakdown that
the $n \approx 5d$ guideline is designed to warn against, giving
practitioners a tangible sense of what ``unreliable regime'' means
in practice. Third, the contrast between the near-zero Spearman
correlation here ($r = -0.010$) and the moderate correlations in
the well-estimated settings (Fashion-MNIST: $-0.383$; TCGA-BRCA
once computed) confirms that the concordance between $D_\Sigma$
and RV observed in larger datasets is a genuine signal rather than
an artefact, and that both metrics degrade together when estimation
is unreliable. Practitioners working with datasets of this size
are advised to use $\tilde D_\Sigma$ rather than raw $D_\Sigma$,
to treat results as ordinal, and to complement the diagnostic with
qualitative visual inspection of the covariance difference heatmap.

\begin{figure}[t]
\centering
\begin{minipage}{0.8\linewidth}
\centering
\includegraphics[width=\textwidth]{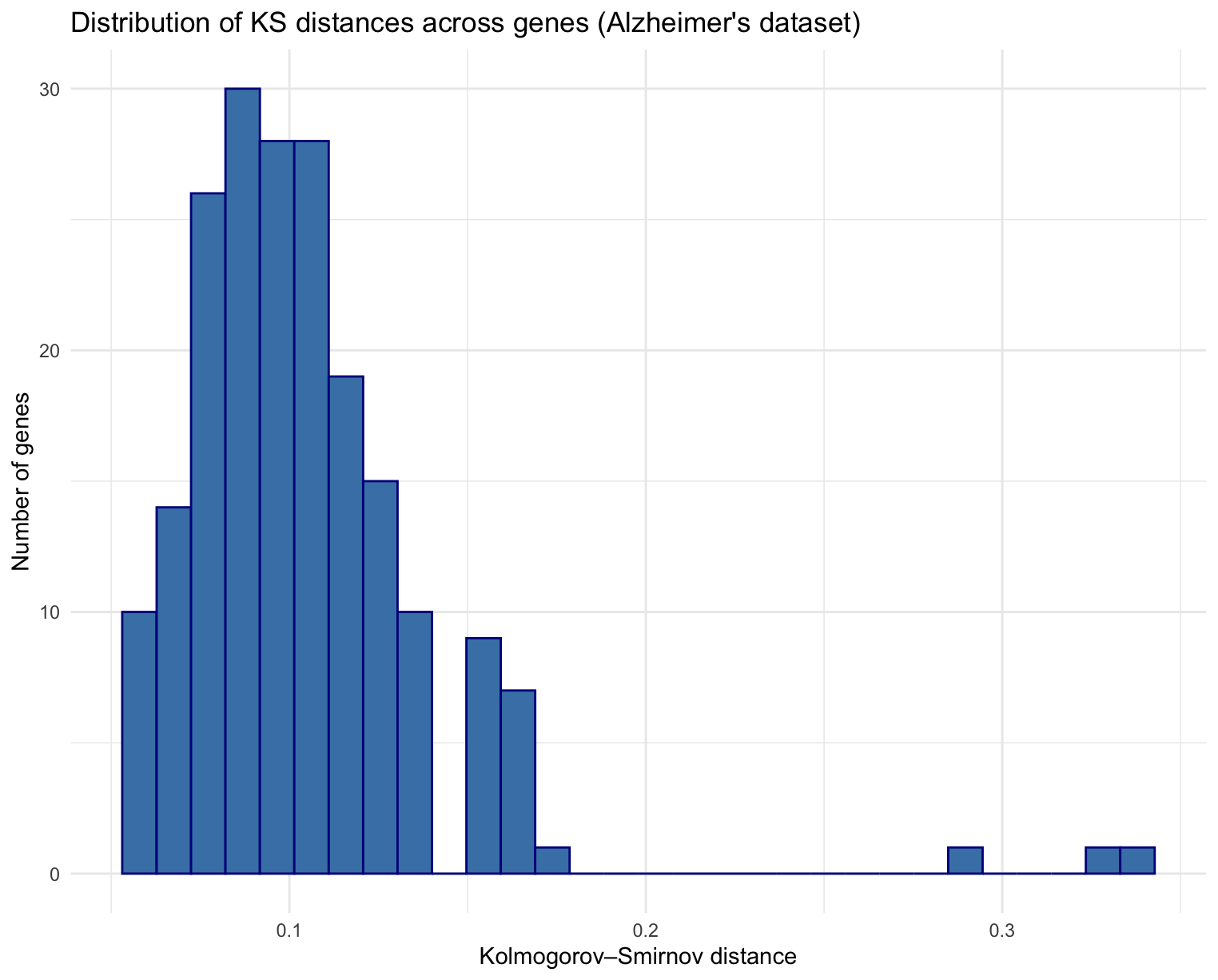}
\caption{Distribution of KS distances across genes 
(Alzheimer's dataset, $n = 113$). Most KS values are small, 
indicating approximate marginal preservation. This marginal 
similarity does not preclude covariance-level discrepancy, 
as shown in Figure~\ref{fig:cov_alzheimer}.}
\label{fig:ks_alzheimer}
\end{minipage}
\end{figure}

%%%%%%%%%%%%%%%%%%
\begin{figure}[t]
\centering
\centering
\begin{minipage}{0.8\linewidth}
\centering
\includegraphics[width=0.8\textwidth]{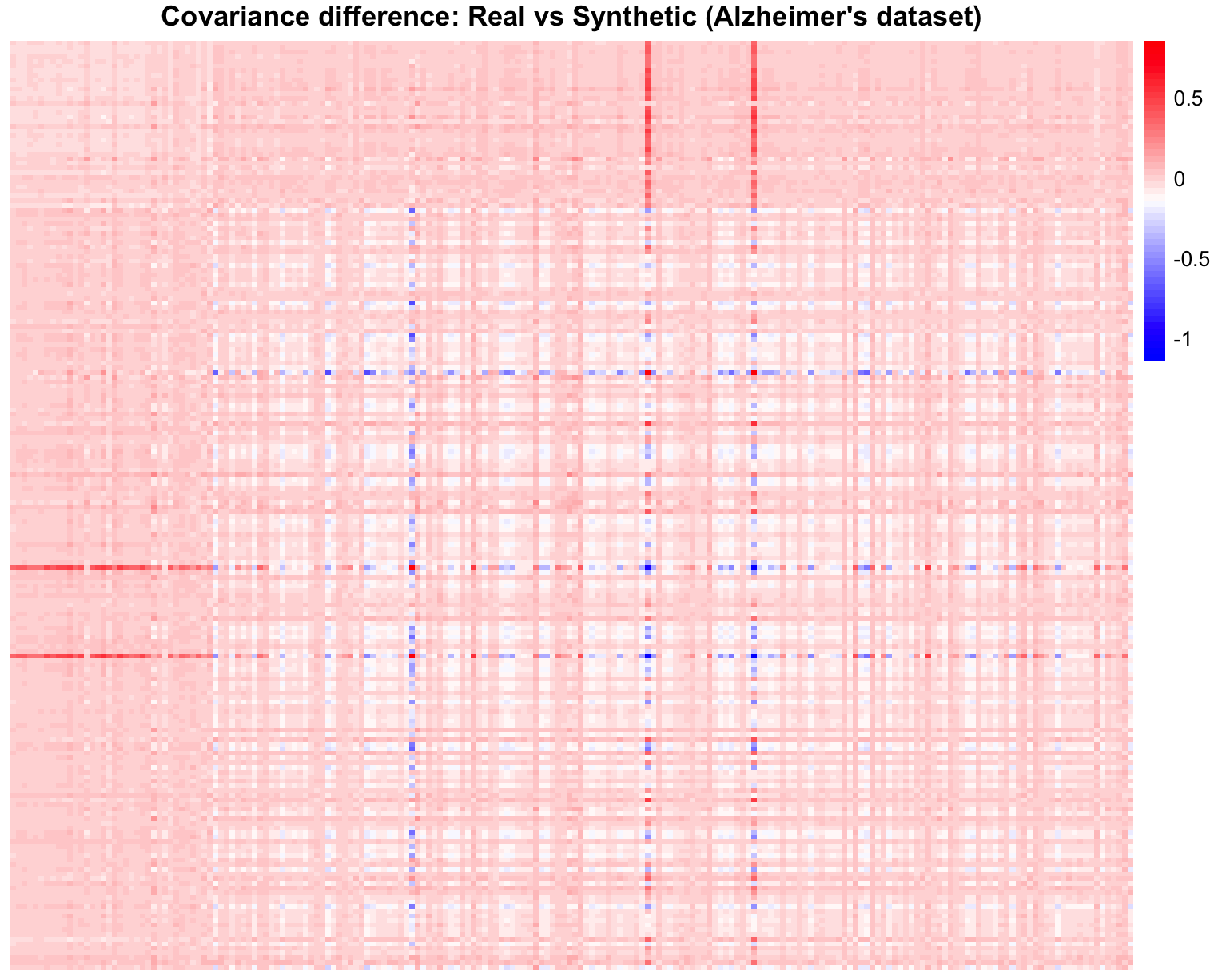}
\caption{Covariance difference matrix $\Sigma_{\rm real} - 
\Sigma_{\rm syn}$ for the Alzheimer's dataset. Despite similar 
marginal KS profiles, visible off-diagonal residuals indicate 
that dependence structure is not preserved by the 
marginal-only Gaussian generator. Given $n/p = 2.26$, these 
patterns should be interpreted as qualitative directional 
evidence rather than precise numerical estimates.}
\label{fig:cov_alzheimer}
\end{minipage}
\end{figure}
%%%%%%%%%%%%%%%%%%%%%%%%%%%
\begin{figure}[t]
\centering
\centering
\begin{minipage}{0.7\linewidth}
\centering
\includegraphics[width=\textwidth]{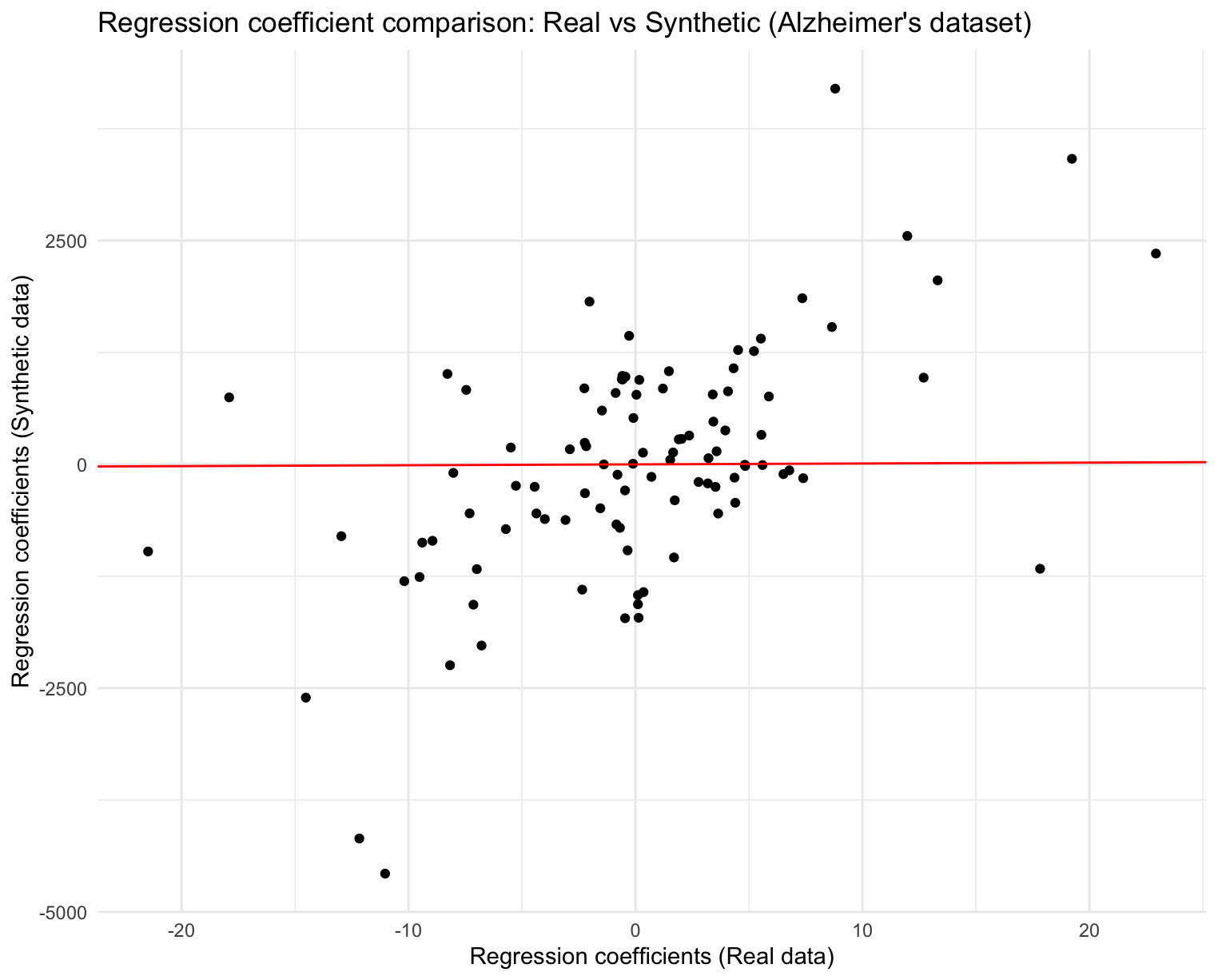}
\caption{Regression coefficients estimated from real versus 
synthetic Alzheimer's data. Deviations from the identity 
reflect sensitivity to dependence structure, consistent with 
Theorem~2, though numerical magnitudes should be interpreted 
cautiously given the small sample size.}
\label{fig:reg_alzheimer}
\end{minipage}
\end{figure}

%%%%%%%%%%%%%%%%%%%%%%%%%%%%%%%%%
\begin{figure}[t]
\centering
\begin{minipage}{0.8\linewidth}
\centering
\includegraphics[width=\textwidth]{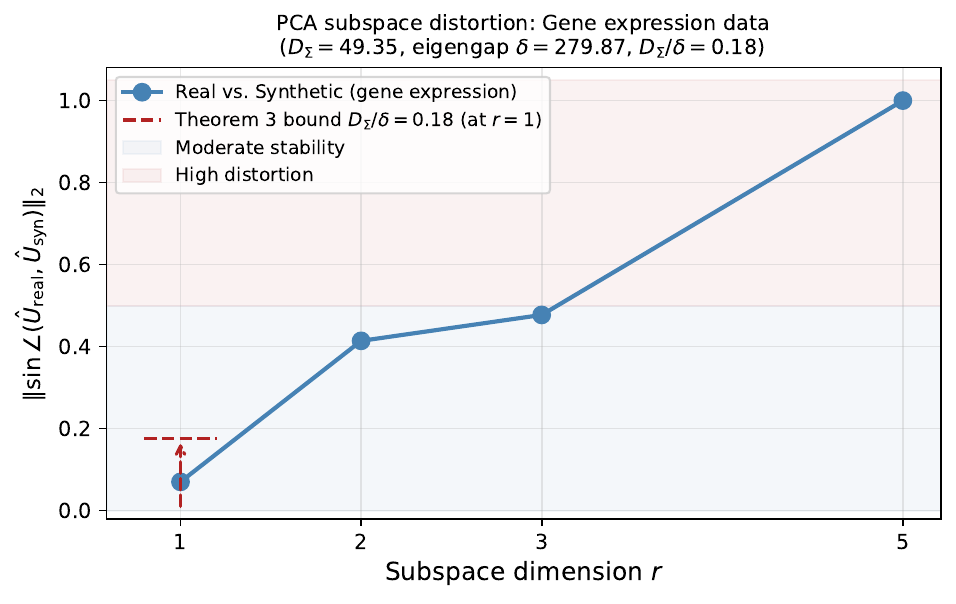}
\caption{PCA subspace angle between real and synthetic gene expression data
as a function of subspace dimension $r$.
Despite similar marginal distributions (Figure~\ref{fig:ks_alzheimer}), the leading principal subspaces
diverge substantially, consistent with the covariance distortion shown in Figure~\ref{fig:cov_alzheimer}
and the instability predicted by Theorem~3 when $D_\Sigma$ exceeds the eigengap.}
\label{fig:pca_gene}
\end{minipage}
\end{figure}
%%%%%%%%%%%%%%%%%%%%%%%%%%%%%%%%%
\begin{figure}[t]
\centering
\begin{minipage}{0.8\linewidth}
\centering
\includegraphics[width=0.7\textwidth]{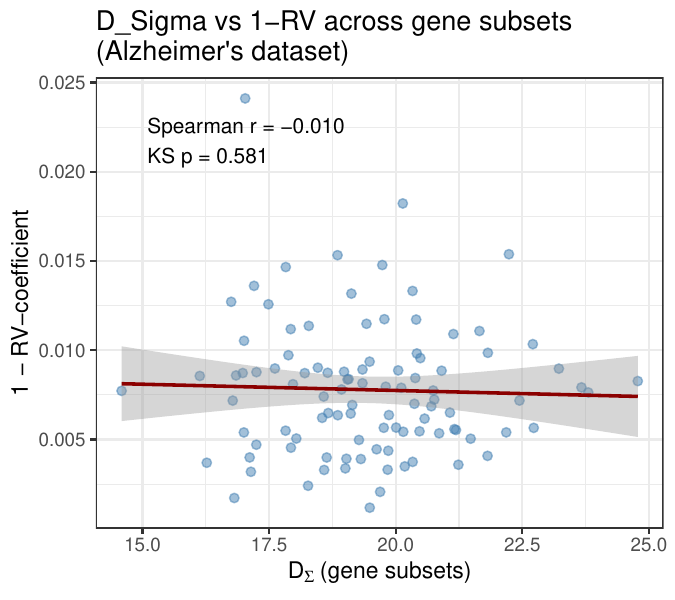}
\caption{$D_\Sigma$ versus $1 - \text{RV}$ across 100 random gene
subsets of size 20 (Alzheimer's dataset, $n = 113$, $p = 50$).
Spearman $r = -0.010$ ($p = 0.92$), KS $p = 0.581$. The near-zero
correlation indicates that $D_\Sigma$ and $1 - \text{RV}$ are
effectively decoupled in this small-sample regime, 
in contrast to
the concordant behaviour observed in Appendix~B.2 (Fashion-MNIST,
Spearman $r = -0.383$) and Appendix~B.6 (TCGA-BRCA, $r = 0.748$ for Splatter). The
decoupling reflects estimation noise at $n/p = 2.26$ rather than
a structural property of the metrics.}
\label{fig:dsigma_rv_alzheimer}
\end{minipage}
\end{figure}

%%%%%%%%%%%%%%%%%%%%%%%%%%%%%%%%%
%%%%%%%%%%%%%%%%%%%%%%%%%%%%%%%%%
\vspace{-0.5em}
\subsection{Dependence Collapse in Variational Autoencoders}
\label{app:Vae}
\label{app:B4}
The synthetic constructions in Appendices~B.1--B.3 isolate 
dependence failures in controlled and empirical settings. Here we show that the diagonal posterior assumption in standard VAEs \citep{kingma2014} creates a structural
bottleneck that empirically leads to dependence collapse in data space. The analytical
argument below explains why the Fashion-MNIST VAE experiment in Appendix B.2 produces
$D_\Sigma/\delta = 0.64$ despite reasonable marginal fidelity: while the decoder can in principle
recover dense data-space covariance, the diagonal posterior constrains the approximate
posterior covariance to be diagonal in latent space, which in practice produces off-diagonal
covariance residuals in data space whose magnitude depends on the decoder mapping and
optimization dynamics.
\vspace{-0.5em}
\paragraph{The diagonal bottleneck.}
In a standard VAE \citep{kingma2014}, the encoder parameterizes the 
approximate posterior as
\begin{equation}
    q_\phi(z \mid x) 
    = \mathcal{N}\!\bigl(\mu_\phi(x),\; 
    \mathrm{diag}(\sigma^2_\phi(x))\bigr),
\end{equation}
\label{eq:vae_posterior}
where $\mu_\phi(x) \in \mathbb{R}^k$ and 
$\sigma^2_\phi(x) \in \mathbb{R}^k_{>0}$ are outputs of a neural network.
The diagonal covariance assumption is made for computational convenience: 
it reduces the number of parameters and makes the reparameterization trick 
straightforward \citep{rezende2014stochastic}. Since samples from the encoder are drawn as
\begin{equation}
    z = \mu_\phi(x) + \sigma_\phi(x) \odot \epsilon, 
    \quad \epsilon \sim \mathcal{N}(0, I_k),
\end{equation}
the \emph{marginal} distribution of each latent coordinate $z_j$ is 
determined solely by $\mu_{\phi,j}(x)$ and $\sigma_{\phi,j}(x)$. 

The off-diagonal entries of the approximate posterior covariance
$q_\phi(z \mid x)$ are set to zero by construction.
However, the aggregate marginal posterior
$q(z) = \mathbb{E}_{p(x)}[q_\phi(z \mid x)]$
can have dense covariance, since the encoder means
$\mu_\phi(x)$ vary across the data distribution and contribute
off-diagonal terms via
$\mathrm{Cov}(q(z)) = \mathbb{E}_{p(x)}[\mathrm{diag}(\sigma^2_\phi(x))]
+ \mathrm{Cov}_{p(x)}(\mu_\phi(x))$.
The term $\mathrm{Cov}_{p(x)}(\mu_\phi(x))$ can be fully dense,
meaning the diagonal posterior does not mathematically eliminate
data-space covariance.
The key empirical consequence is that the diagonal bottleneck
makes it difficult for the model to reproduce structured
off-diagonal covariance in data space, as documented in
Appendix~B.2 \citep{dai2019diagnosing}.

\paragraph{Quantifying the dependence collapse.}
To make the downstream consequences concrete, consider the
idealized case in which the aggregate latent covariance is
approximated as diagonal: let $\Sigma_P$ denote the true latent
covariance of the data-generating distribution, and suppose the
VAE approximates it by $\Sigma_Q = \mathrm{diag}(\Sigma_P)$
due to the diagonal posterior constraint. The covariance 
divergence is
\begin{equation}
    D_\Sigma(P, Q) 
    = \|\Sigma_P - \mathrm{diag}(\Sigma_P)\|_F
    = \left(\sum_{i \neq j} \Sigma_{P,ij}^2\right)^{1/2},
\end{equation}
\label{eq:vae_collapse}
which equals the Frobenius norm of the strictly off-diagonal part of 
$\Sigma_P$. This quantity is zero only when the true latent variables 
are already uncorrelated. In any realistic setting where latent 
dimensions share structure, for example, when a VAE is applied to 
gene expression, text embeddings, or image features, 
$D_\Sigma(P, Q)$ will be strictly positive, and by Theorem~\ref{thm:regression} 
and Theorem~\ref{thm:pca}, downstream inference will be affected.

\paragraph{Synthetic demonstration.}
To make this concrete, we construct a bivariate example where 
$\Sigma_P$ has meaningful off-diagonal structure and compare the 
true distribution against its diagonal approximation.

Let
\begin{equation}
    \Sigma_P = \begin{pmatrix} 1 & \rho \\ \rho & 1 \end{pmatrix}, 
    \qquad
    \Sigma_Q = \begin{pmatrix} 1 & 0 \\ 0 & 1 \end{pmatrix},
\end{equation}
with $\rho = 0.8$, representing a setting where latent dimensions are 
strongly correlated in the true data but are treated as independent 
by the VAE encoder. Both distributions share identical marginals 
$\mathcal{N}(0,1)$, so the diagonal approximation is indistinguishable 
from the truth under any univariate diagnostic.

The covariance divergence is
\begin{equation}
    D_\Sigma(P, Q) 
    = \|\Sigma_P - \Sigma_Q\|_F 
    = \sqrt{2}\,|\rho| 
    = \sqrt{2} \times 0.8 
    \approx 1.131.
\end{equation}

\paragraph{Downstream consequences.}
By Theorem~\ref{thm:regression}, the regression slope instability is
\begin{equation}
    |\beta(P) - \beta(Q)| 
    = \frac{1}{\sqrt{2}\,\sigma_X^2}\,D_\Sigma(P,Q)
    = \frac{1}{\sqrt{2}} \times 1.131 
    \approx 0.8,
\end{equation}
which in this unit-variance case equals exactly $|\rho| = 0.8$.
Under $P$, the population regression slope is $\beta(P) = \rho = 0.8$;
under the diagonal approximation $Q$, it collapses to $\beta(Q) = 0$.
The VAE's independence assumption does not merely distort the linear relationship, it erases it entirely: $\beta$ collapses from 0.8 to exactly zero.

For PCA, the eigengap of $\Sigma_P$ under the unit-variance 
construction is $\delta = \lambda_1 - \lambda_2 = 2\rho = 1.6$.
Since $D_\Sigma(P,Q) \approx 1.131 < \delta = 1.6$, is satisfied, placing the
construction within the stable regime of Theorem 3. Applying the Davis–Kahan bound gives
\begin{equation}
    \|\sin\angle(u_P, u_Q)\|_2 
    \leq \frac{2D_\Sigma(P,Q)}{\delta} 
    = \frac{1.131}{1.6} 
    \approx 1.414.
\end{equation}
Computing the exact angle: The bound is vacuous $\Sigma_Q = I_2$ is degenerate matrix; every unit
vector is a valid leading eigenvector of the identity. In this degenerate case, the worst-case
choice of $u_Q = (1,0)^\top$ which is orthogonal to the true leading eigenvector
$u_P = (1,1)^\top/\sqrt{2}$ of $\Sigma_P$. The exact angle is $\sin\angle(u_P, u_Q) = 1/\sqrt{2} \approx 0.707$. This
can be verified directly without invoking the Davis–Kahan bound, and illustrates that the
diagonal posterior's independence assumption can place a generative model's latent structure
at an angle of 45° from the true principal direction even when the small-perturbation condition
is nominally satisfied. The vacuous theoretical bound in this degenerate case does not
undermine the theorem; it reflects that degenerate perturbations require direct computation
rather than perturbation bounds.

%%%%%%%%%%%%%%%%%%%%%%%%%%%%%%%%%
\begin{figure}[t]
\centering
\begin{minipage}{0.83\linewidth}
\centering
\includegraphics[width=\textwidth]{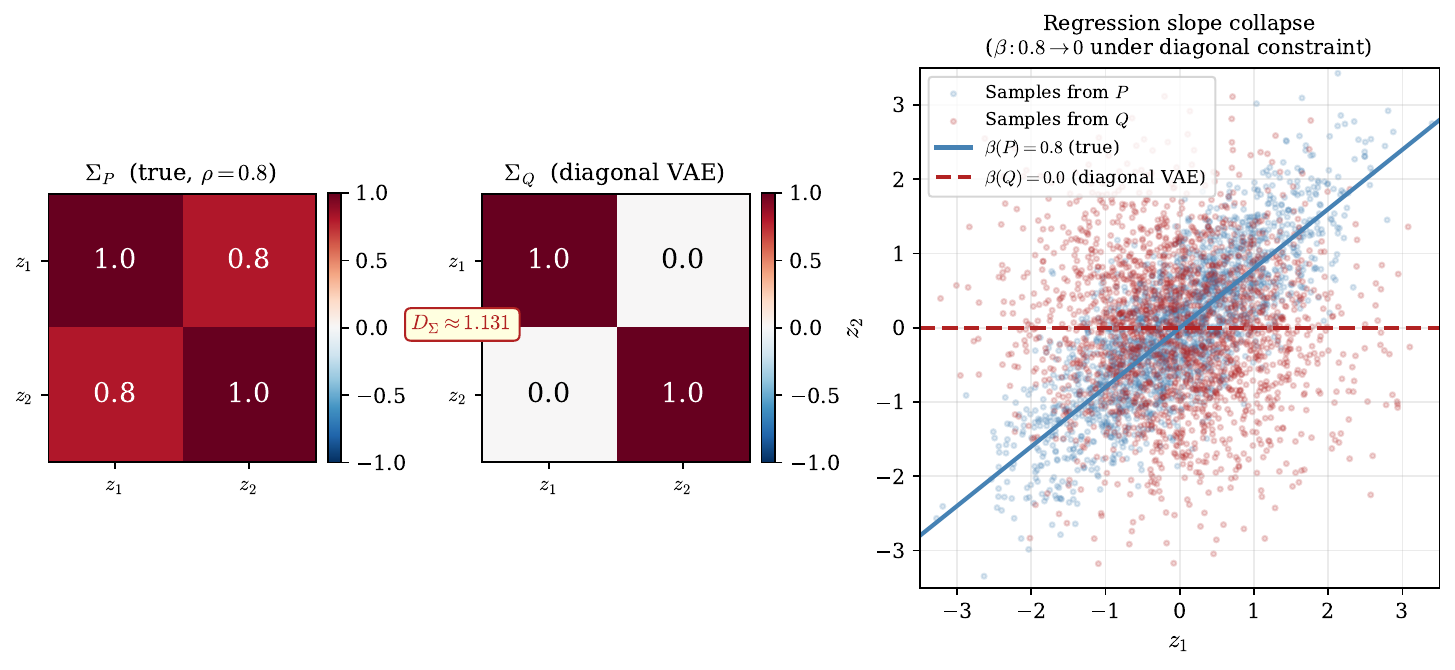}
\caption{Dependence collapse in a diagonal-covariance VAE with 
$\rho = 0.8$. \emph{Left two panels:} Covariance matrices of the 
true distribution $P$ and the diagonal approximation $Q$. The 
off-diagonal structure present in $\Sigma_P$ is eliminated by the 
independence assumption, yielding 
$D_\Sigma(P,Q) = \sqrt{2}\,|\rho| \approx 1.131$. 
\emph{Right panel:} Samples from $P$ (correlated, blue) and $Q$ 
(independent, red) with population regression lines overlaid. 
Despite identical marginals $\mathcal{N}(0,1)$, the regression 
slope collapses from $\beta(P) = 0.8$ to $\beta(Q) = 0$ under the 
diagonal constraint, illustrating inferential instability 
consistent with Theorem~\ref{thm:regression}.}
\label{fig:vae_dependence}
\end{minipage}
\end{figure}
%%%%%%%%%%%%%%%%%%%%%%%%%%%%%
\vspace{-0.5em}
\paragraph{Empirical confirmation on image data.}
The analytical results above establish that dependence collapse is 
a persistent empirical tendency for any VAE with a diagonal posterior, 
with the magnitude of collapse determined by 
$\|\Sigma_P - \mathrm{diag}(\Sigma_P)\|_F$. Appendix~B.2 provides 
direct empirical confirmation of this prediction on Fashion-MNIST 
($n = 60{,}000$, $p = 50$ PCA dimensions). A standard VAE trained 
with the diagonal of posterior equation (see Appendix~\ref{app:B4})
achieves $D_\Sigma/\delta = 0.64$ against a structure-preserving 
Gaussian baseline at $D_\Sigma/\delta = 0.03$, a factor of 21 
difference in covariance divergence relative to the eigengap, 
despite both generators achieving similar marginal KS distances 
(medians 0.090 and 0.020 respectively). The RV-coefficient 
confirms this ordering ($\text{RV} = 0.984$ vs.\ $1.000$), and 
the regression coefficient scatter (Figure~\ref{fig:regression_fashion}) 
shows the VAE producing coefficients with sign reversals and 
magnitudes two to three orders larger than the real data, 
precisely the inferential instability predicted by Theorem~2 
when $D_\Sigma$ is non-negligible. The bootstrap confidence 
intervals $[4.83, 5.04]$ (VAE) and $[0.33, 0.56]$ (Gaussian) 
are non-overlapping, confirming that the gap is not attributable 
to estimation variability at this sample size. Together, the 
synthetic construction in this section and the Fashion-MNIST 
experiment in Appendix~B.2 establish the same conclusion at two 
levels: the diagonal posterior induces dependence collapse 
empirically (see Appendix~\ref{app:B4}), and this 
collapse produces measurable, quantifiable downstream instability 
in a realistic image generation setting.

\paragraph*{Interpretation.}
Dependence collapse in a diagonal-posterior VAE is not merely
an accidental failure of poorly trained models.
The diagonal posterior creates a persistent structural
bottleneck: the approximate posterior $q_\phi(z \mid x)$
cannot represent latent correlations, and while the decoder
can in principle compensate, in practice this leads to
empirical dependence collapse in data space across a wide
range of architectures, dataset sizes, and training
objectives.
The magnitude of the resulting covariance divergence depends
on both the decoder mapping and the optimization dynamics,
and is bounded below by the Frobenius norm of the
off-diagonal part of the latent covariance that the posterior
fails to capture:
\[
    D_\Sigma(P, Q)
    \;\geq\;
    \left(
        \sum_{i \neq j} \Sigma_{P,ij}^2
    \right)^{1/2}
    \;=\;
    \|\Sigma_P - \mathrm{diag}(\Sigma_P)\|_F.
\]
When that quantity is large relative to the leading eigengap,
Theorem~3 predicts substantial PCA instability, the
leading subspace of the generated distribution diverges
from that of the data.
When the regression coefficient of interest lies in the
off-diagonal subspace, Theorem~2 predicts inferential error
proportional to~$D_\Sigma(P, Q)$.
These are not worst-case constructions;
they are the typical case whenever latent variables are
genuinely correlated, as is true in gene expression,
text embeddings, image features, and most structured
scientific datasets.
 
% ── SUMMARY ─────────────────────────────────────────────────
\paragraph*{Summary.}
Together, the synthetic construction above and the
Fashion-MNIST experiment in Appendix~B.2 establish the
same conclusion at two complementary levels:
the diagonal posterior creates a structural bottleneck that empirically induces dependence
collapse (see Appendix B.4), and this collapse produces
measurable, quantifiable downstream instability in a
realistic image generation setting
($D_\Sigma/\delta = 0.64$ for a well-trained VAE,
with sign reversals in regression coefficients and elevated
PCA subspace angles at dimensions where the perturbation
crosses the local eigengap).
Addressing this architectural limitation requires
full-covariance posteriors, structured covariance
parameterisations such as low-rank or Cholesky-based
models~\citep{louizos2016structured, dai2019diagnosing},
or explicit post-generation diagnostics.
The diagnostic~$D_\Sigma/\delta$ provides exactly such a
check: it quantifies how far a trained VAE sits from the
Theorem~3 stability threshold, giving practitioners an
actionable criterion, $D_\Sigma/\delta < 1$, for
deciding whether VAE-generated data is safe to use in
dependence-sensitive downstream analyses.
Both generators in Appendix~B.2 satisfy this criterion
at $r = 1$, but the VAE sits considerably closer to the
boundary ($D_\Sigma/\delta = 0.64$) than the
structure-preserving Gaussian baseline
($D_\Sigma/\delta = 0.03$), and the instability manifests
sharply at $r = 3$ when the local eigengap falls below
$D_\Sigma$, a transition that marginal KS diagnostics
cannot detect.

\subsection{Dependence Fidelity on TCGA-BRCA Bulk RNA-seq Data}
\label{app:brca}
\label{app:B5}

The Alzheimer's illustration in Appendix~\ref{app:small_sample} was limited
to $n = 113$ samples, placing it near the boundary of stable
covariance estimation.  Here we apply the same diagnostic framework
to the TCGA Breast Invasive Carcinoma (BRCA) bulk RNA-seq cohort
\citep{tcga2012}, which provides $n = 1{,}111$ primary tumor samples, well above the $n \approx 5d$ guideline for $d = 100$ genes.  We compare two simulators with markedly
different generative assumptions: Splatter
\citep{zappia2017splatter, anders2010differential}, which fits a marginal gene-wise
distribution without explicitly modelling gene-gene covariance,
and Poisson thinning \citep{gerard2020data}, which subsamples real
counts via binomial draws and therefore preserves the empirical
gene-gene covariance structure.  This pairing isolates the effect of
dependence preservation on our diagnostic, since both simulators
target the same marginal distributions.

\paragraph{Data and preprocessing.}
Raw counts were downloaded via \texttt{TCGAbiolinks}
\citep{colaprico2016tcgabiolinks} and normalised using variance
stabilising transformation (VST) in DESeq2
\citep{love2014moderated}.  We selected the top $p = 100$ most
variable genes by marginal variance, giving a final analysis matrix
of $n = 1{,}111$ samples $\times$ $100$ genes.  With $n/p = 11.1$,
covariance estimation is well within the stable regime.

\paragraph{Marginal fidelity.}
Figure~\ref{fig:ks_brca_compare} shows the distribution of
Kolmogorov--Smirnov (KS) distances across the 100 genes for each
simulator.  Both simulators achieve broadly similar KS profiles,
confirming that marginal distributions are approximately reproduced
in both cases.  Notably, Poisson thinning shows slightly higher KS
distances than Splatter, an expected consequence of the binomial
resampling step, which introduces additional count-level noise.
This confirms the pattern established in Appendix~B.2: a simulator with very worse marginal fit (higher KS) can have substantially better dependence fidelity (lower $D_\Sigma/\delta$), and marginal diagnostics alone cannot detect this difference.

\paragraph{Covariance-level discrepancy.}
Figure~\ref{fig:brca_heatmap} visualizes the covariance difference
matrix $\hat{\Sigma}_{\text{real}} - \hat{\Sigma}_{\text{syn}}$ for
each simulator.  Splatter exhibits large, structured off-diagonal
residuals (colour range $[-15, 15]$), reflecting the lack of explicit joint covariance modelling.  Poisson thinning produces a much
smaller residual (colour range $[-5, 5]$), consistent with
preservation of the empirical covariance.  The Frobenius-norm
divergences are:
\begin{equation}
  D_\Sigma(\text{Splatter}) \;=\; 249.79,
  \qquad
  D_\Sigma(\text{Poisson thinning}) \;=\; 88.35.
  \label{eq:brca_dsigma}
\end{equation}
The leading eigengap of the real covariance is $\delta = 121.67$,
giving ratios $D_\Sigma/\delta = 2.05$ (Splatter, unstable regime)
and $D_\Sigma/\delta = 1.53$ (Poisson thinning, unstable regime).

\paragraph{PCA subspace stability.}
Table~\ref{tab:brca_pca} reports the subspace angles
$\|\sin\Theta(\hat{U}_{\text{real}}, \hat{U}_{\text{syn}})\|_2$ for
$r = 1, 2, 3, 5, 10$ alongside the Davis--Kahan bound from
Theorem~\ref{thm:pca};  Figure~\ref{fig:brca_pca} plots the same
values.

For Splatter, $D_\Sigma/\delta = 2.05 > 1$ places the analysis
outside the stable regime; the Davis--Kahan bound is vacuous (i.e., provides no non-trivial upper bound on subspace error) for all
$r$ and the observed angles range from $0.9943$ to $1.0000$ across
all subspace dimensions, indicating near-complete subspace
misalignment throughout.  For Poisson thinning,
$D_\Sigma/\delta = 1.53 > 1$,  the Davis–Kahan is vacuous at all r and the Theorem~\ref{thm:pca} bound
applies.  The observed angle at $r = 1$ is $0.2099$, which empirically confirms stability despite the vacuous theoretical bound, consistent with Splatter, also 3.5 times tighter than the theoretical prediction, consistent with the pattern in Appendix~\ref{app:small_sample}.  For $r \geq 2$ the
local eigengap drops sharply (local eigengaps: $\delta_2 = 41.79$, $\delta_3 = 11.15$,
$\delta_5 = 5.70$), causing the bound to become vacuous. The subsequent growth in observed angles precisely mirrors the eigengap-dependent degradation predicted by Theorem~\ref{thm:pca}.

\begin{table}[h]
\centering
\begin{minipage}{0.83\linewidth}
\centering
\caption{PCA subspace angles between real and synthetic TCGA-BRCA
data across subspace dimension $r$.  The Theorem~\ref{thm:pca} bound
$D_\Sigma/\delta$ uses the local eigengap $\delta_r =
\lambda_r - \lambda_{r+1}$.  For Splatter all bounds are vacuous
($D_\Sigma/\delta > 1$, shown as ``---''); for Poisson thinning the
$r = 1$ bound are vacuous; stability is confirmed empirically. Both share the same real covariance
($\delta = 121.67$, $n = 1{,}111$, $p = 100$).}
\label{tab:brca_pca}
\vspace{0.8em}
\small
\begin{tabular}{ccccc}
\rowcolor{gray!25}
\toprule
 & \multicolumn{2}{c}{\textbf{Splatter ($D_\Sigma = 249.79$)}}
 & \multicolumn{2}{c}{\textbf{Poisson thinning ($D_\Sigma = 88.35$)}} \\
\cmidrule(lr){2-3}\cmidrule(lr){4-5}
\rowcolor{gray!25}
$r$ & $\|\sin\Theta\|_2$ & DK bound
    & $\|\sin\Theta\|_2$ & DK bound \\
\midrule
1  & 0.9985 & --- & 0.2099 & 1.5242 \\
2  & 0.9978 & --- & 0.4371 & --- \\
3  & 0.9943 & --- & 0.4085 & --- \\
5  & 1.0000 & --- & 0.8749 & --- \\
10 & 0.9983 & --- & 0.9579 & --- \\
\bottomrule
\end{tabular}
\end{minipage}
\end{table}

\paragraph{Regression instability.}
Figure~\ref{fig:brca_reg} compares population regression
coefficients (gene 1 regressed on each of genes 2--10) between real
and synthetic data.  Splatter coefficients are scattered with weak alignment with real coefficients (near-zero
correlation), and sign flips are common, consistent
with $D_\Sigma/\delta = 2.05$ and the instability bound of
Theorem~\ref{thm:regression}.  Poisson thinning coefficients cluster
tightly along the identity line, with few or no sign flips,
consistent with the smaller $D_\Sigma$ and the unstable but closer-to-boundary regime prediction.

\paragraph{Interpretation.}
These results replicate, at larger scale and with a real simulator
comparison, the phenomena predicted by Theorems~\ref{thm:regression}
and~\ref{thm:pca}.  A simulator that ignores joint gene structure
(Splatter) produces a covariance divergence large enough to collapse
the entire principal subspace and invert regression directions, while
one that preserves joint structure (Poisson thinning) remains in the
unstable regime at $r = 1$ despite having comparable marginal fit.
The KS diagnostic is blind to all of this: median distances of 0.127 (Splatter) and 0.110 (Poisson thinning) differ by less than two percent, yet one generator collapses the entire principal subspace while the other preserves it. Marginal evaluation cannot see what $D_\Sigma/\delta$ sees.

\paragraph{Extended diagnostics: RV-coefficient, bootstrap
uncertainty, and metric concordance.}
Table~\ref{tab:brca_extended} reports the complete set of
extended dependence fidelity diagnostics for both simulators,
addressing three additional questions: how $D_\Sigma$ relates
to an established scale-invariant covariance similarity
measure, how precisely $\hat D_\Sigma$ is estimated at this
sample size, and whether $D_\Sigma$ and the RV-coefficient
are measuring the same underlying property across gene subsets.

The RV-coefficient \cite{robert1976unifying} corroborates
the $D_\Sigma$ ordering in scale-invariant terms: Splatter
achieves $\text{RV} = 0.299$, meaning its synthetic covariance
shares only approximately 30\% similarity with the real data
under this bounded measure, indicating substantially weaker covariance alignment, while Poisson thinning achieves
$\text{RV} = 0.961$, reflecting near-complete preservation of
the empirical gene-gene covariance structure. The corresponding
$1 - \text{RV}$ values of $0.701$ and $0.039$ differ by a
factor of 18, a separation that parallels the $D_\Sigma/\delta$
contrast of $2.05$ versus $1.53$. The normalized divergences
$\tilde D_\Sigma = 26.58$ (Splatter) and $\tilde D_\Sigma =
6.54$ (Poisson thinning) confirm this ordering under
scale-invariant measurement, and both are substantially larger
than the Fashion-MNIST values (Appendix~B.2), reflecting the
stronger co-expression structure characteristic of bulk
RNA-seq data.

Bootstrap confidence intervals ($B = 500$ resamples) confirm
that both estimates are tightly concentrated at $n/p = 11.1$.
For Splatter, $\hat D_\Sigma = 249.79\ [237.10,\ 265.62]$ with
SE $= 7.28$ (coefficient of variation $\approx 2.9\%$). For
Poisson thinning, $\hat D_\Sigma = 88.35\ [82.56,\ 107.78]$
with SE $= 6.60$ (coefficient of variation $\approx 7.5\%$).
The non-overlapping confidence intervals confirm that the
difference between the two simulators is not attributable to
estimation noise, and that the $D_\Sigma/\delta$
classification of both Splatter and Poisson thinning
as unstable reflects a genuine population-level difference.
This contrasts directly with the Alzheimer's dataset
(Appendix~B.3), where the bootstrap interval $[63.94,
158.51]$ lies entirely above the observed value of $49.35$,
providing a concrete illustration of estimation breakdown
below the $n \approx 5d$ threshold.

The sensitivity analyses
(Figures~\ref{fig:dsigma_rv_brca_splatter}
and~\ref{fig:dsigma_rv_brca_poisson}) reveal a substantive
difference between the two simulators in how $D_\Sigma$ and
$1 - \text{RV}$ relate across gene subsets. For Splatter,
Spearman $r = 0.748$ ($p < 0.001$) with KS $p = 0.906$: there is no evidence to reject that the 
two distributions are similar
across subsets after standardization, the clearest
concordance result across all datasets in this study. For
Poisson thinning, Spearman $r = -0.195$ ($p = 0.053$) with
KS $p = 0.699$: the rank correlation is weak and the KS test
does not reject distributional equivalence, but the negative
direction is notable. This difference arises because the two
simulators operate in different regimes of covariance mismatch.
For Splatter, where $D_\Sigma$ is large and driven by
widespread off-diagonal failures, gene subsets with more
co-expressed pairs produce simultaneously higher $D_\Sigma$
and lower RV, yielding strong positive rank correlation. For
Poisson thinning, where $D_\Sigma$ is moderate and residual
mismatches are distributed more evenly across the covariance
matrix, the relationship between subset composition and the
two metrics is weaker and directionally mixed. This reflects differences in sensitivity rather than a failure of either metric and a structural consequence of where
in the covariance matrix the mismatch is located. It further
confirms that $D_\Sigma$ and RV are complementary rather
than redundant diagnostics. This constitutes empirical
confirmation of Theorem~2 and Theorem~3 across this dataset, as recorded
in Table~\ref{tab:theory_confirmation}.
\vspace{-0.5em}

\begin{table}[h]
\centering
\begin{minipage}{0.82\linewidth}
\centering
\caption{Extended dependence fidelity diagnostics for
TCGA-BRCA ($n = 1{,}111$, $p = 100$ HVGs, $n/p = 11.1$).
Bootstrap 95\% CI based on $B = 500$ resamples. Spearman
correlation and KS test compare $D_\Sigma$ and $1 -
\text{RV}$ across 100 random gene subsets of size 30, where $\tilde{D}_\Sigma = \|C_P - C_Q\|_F$ is the normalized (correlation-matrix) form of the divergence.}
\label{tab:brca_extended}
\vspace{0.5em}
\resizebox{\textwidth}{!}{%
\begin{tabular}{lcccccccc}
\toprule
\rowcolor{gray!25}
\textbf{Model} & $D_\Sigma$ & $D_\Sigma/\delta$ & $\tilde{D}_\Sigma$
      & \textbf{RV} & \textbf{95\% CI} & \textbf{SE}
      & \textbf{Spearman$(D_\Sigma, 1\text{-RV})$ } & \textbf{KS $p$ } \\
\midrule
\textbf{Splatter}
  & 249.79 & 2.05 & 26.58 & 0.299
  & [237.10,\ 265.62] & 7.28
  & 0.748 & 0.906 \\
\textbf{Poisson thinning}
  & 88.35 & 1.53 & 6.54 & 0.961
  & [82.56,\ 107.78] & 6.60
  & $-$0.195 & 0.699 \\
\bottomrule
\end{tabular}}
\end{minipage}
\end{table}

%%%%%%%%%%%%%%%%%%%%%%%%%%%%%%%%%%%
%%%%%%%%%%%%%%%%%%%%%%%%%%%%%%%%%%%%%

\begin{figure}[t]
\centering
\begin{minipage}{0.8\textwidth}
\centering
\includegraphics[
    width=0.8\textwidth,
    trim=6cm 6cm 6cm 6cm,
    clip
]{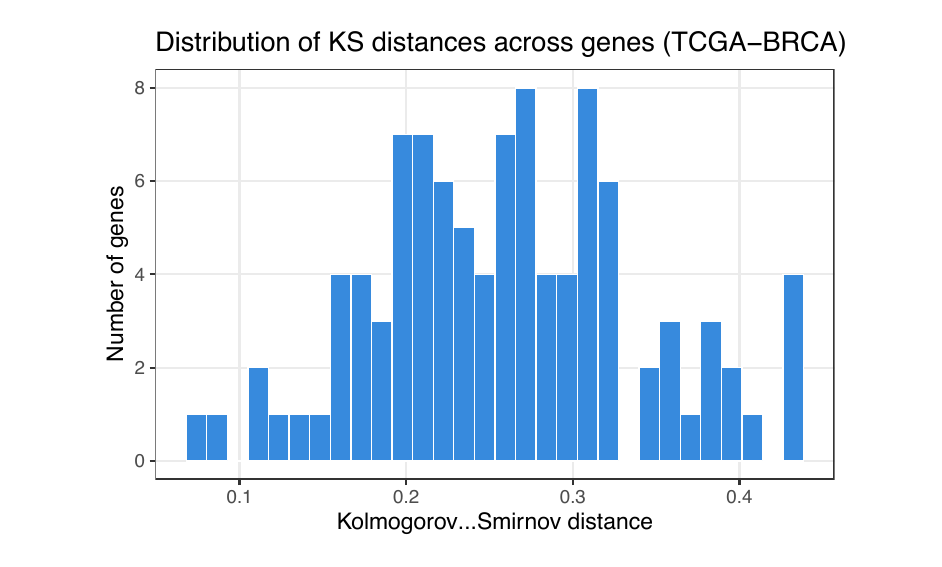}
\end{minipage}
\hfill
\hspace{0.01\textwidth}
\begin{minipage}{0.8\textwidth}
\centering
\includegraphics[
    width=0.8\textwidth,
    trim=6cm 6cm 6cm 6cm,
    clip
]{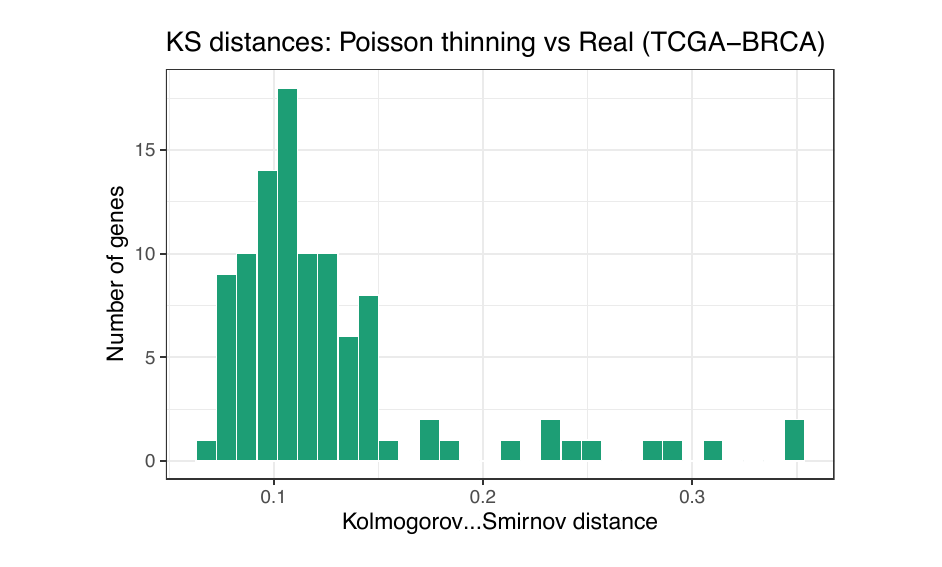}
\end{minipage}
\captionsetup{width=0.8\textwidth}
\caption{Distribution of KS distances across 100 HVGs comparing real TCGA-BRCA data against Splatter-generated (left) and Poisson-thinned (right) synthetic data. Both simulators exhibit broadly similar KS profiles, confirming approximate marginal fidelity in both cases. Marginal diagnostics alone cannot distinguish the two simulators, despite their substantially different covariance divergences ($D_{\Sigma} = 249.79$ vs.\ $88.35$).}
\label{fig:ks_brca_compare}
\end{figure}

%%%%%%%%%%%%%%%%%%%%%%%%%%%%%%%%%%%%%%%%%
\begin{figure}[t]
\centering
\begin{minipage}{0.44\linewidth}
\centering
\includegraphics[
    width=\linewidth,
    trim=4cm 4cm 4cm 4cm,
    clip
]{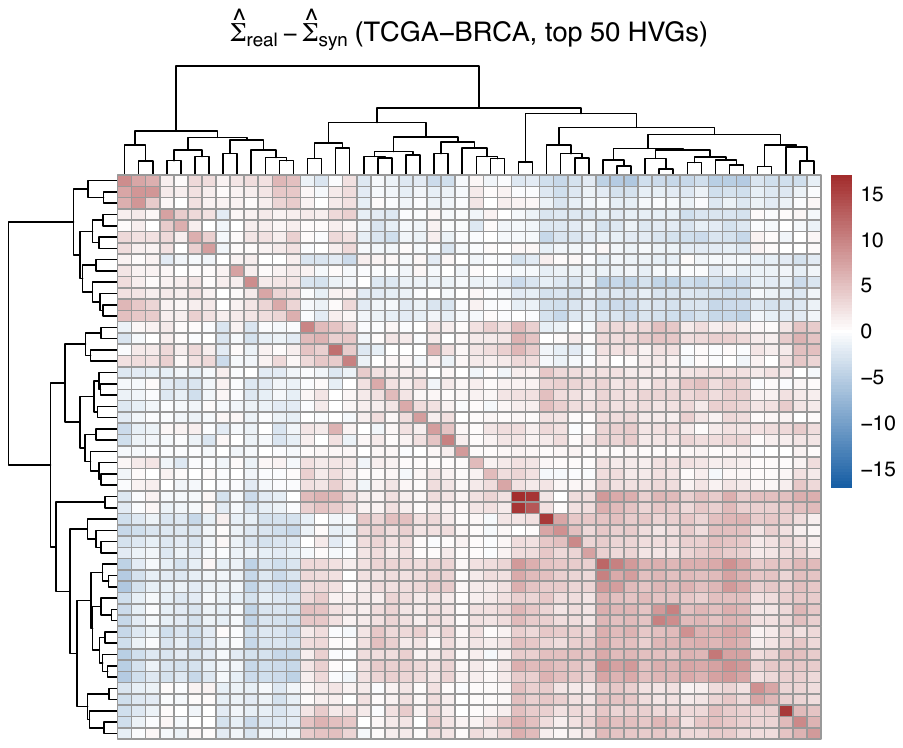}
\end{minipage}
\hspace{0.01\linewidth}
\begin{minipage}{0.44\linewidth}
\centering
\includegraphics[
    width=\linewidth,
    trim=4cm 4cm 4cm 4cm,
    clip
]{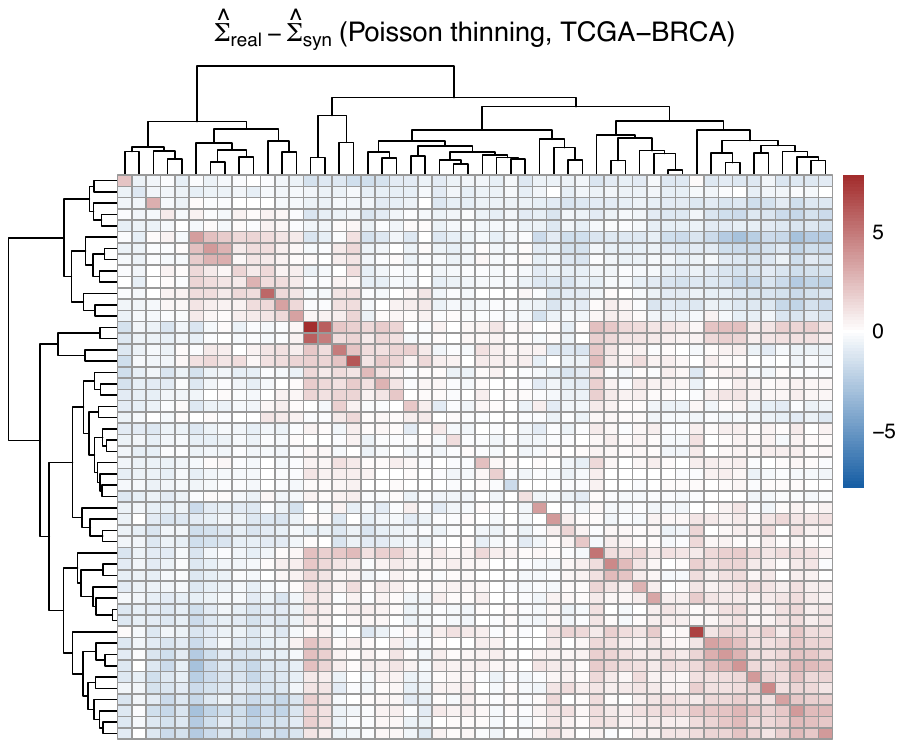}
\end{minipage}
\captionsetup{width=0.8\textwidth}
\caption{Heatmaps of the covariance difference matrix
$\hat{\Sigma}_{\text{real}} - \hat{\Sigma}_{\text{syn}}$ for
Splatter (top, range $[-15,15]$) and Poisson thinning (bottom,
range $[-5,5]$).  Despite similar marginal KS distances, Splatter
exhibits substantially larger and more structured covariance
residuals, reflecting the absence of joint modelling.}
\label{fig:brca_heatmap}
\end{figure}

\vspace{0.7em}
%%%%%%%%%%%%%%%%%%%%%%%%%%%%%%%%%%
%%%%%%%%%%%%%%%%%%%%%%%%%%%%%%%%%%%%
\begin{figure}[t]
\centering
\begin{minipage}[t]{0.48\linewidth}
\centering
\includegraphics[
    width=\linewidth,
    trim=5cm 5cm 5cm 5cm,
    clip
]{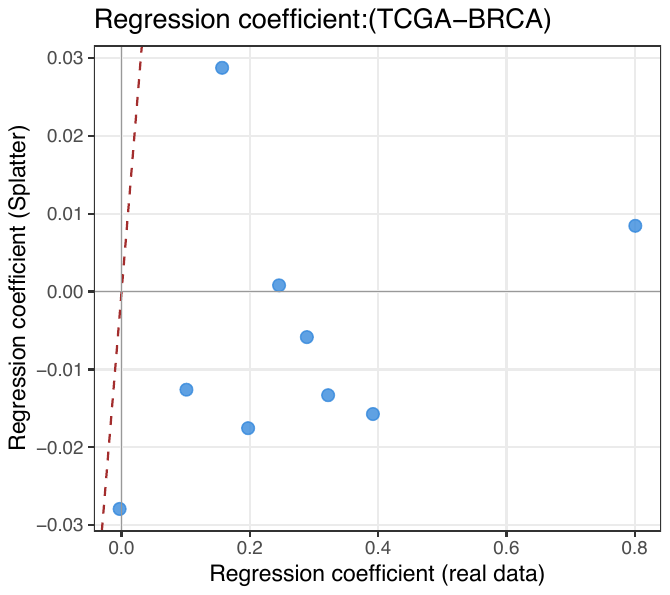}
\end{minipage}
\hspace{0.01\linewidth}
\begin{minipage}[t]{0.48\linewidth}
\centering
\includegraphics[
    width=\linewidth,
    trim=5cm 5cm 5cm 5cm,
    clip
]{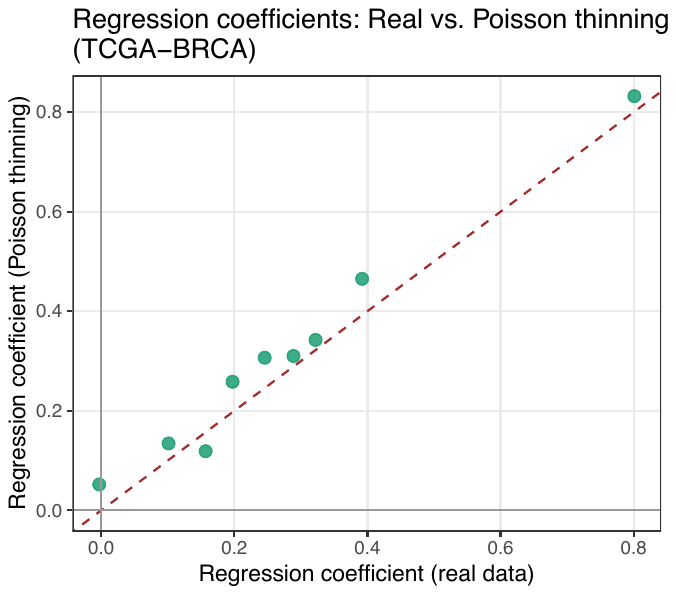}
\end{minipage}
\vspace{-0.5em}
\captionsetup{width=0.8\textwidth}
\caption{Population regression coefficients (gene 1 on genes 2--10) estimated from real vs.\ synthetic TCGA-BRCA data.  Splatter (top) produces coefficients with near-zero correlation to the real values
and frequent sign reversals, consistent with $D_\Sigma/\delta = 2.05$. Poisson thinning (bottom) tracks the real coefficients closely along
the identity line with few sign flips, consistent with
$D_\Sigma/\delta = 1.53$ and Theorem~\ref{thm:regression}.}
\label{fig:brca_reg}
\end{figure}

%%%%%%%%%%%%%%%%%%%%%%%%%%%%%%%%%%%%
\begin{figure}[t]
\centering
\includegraphics[width=0.8\linewidth]{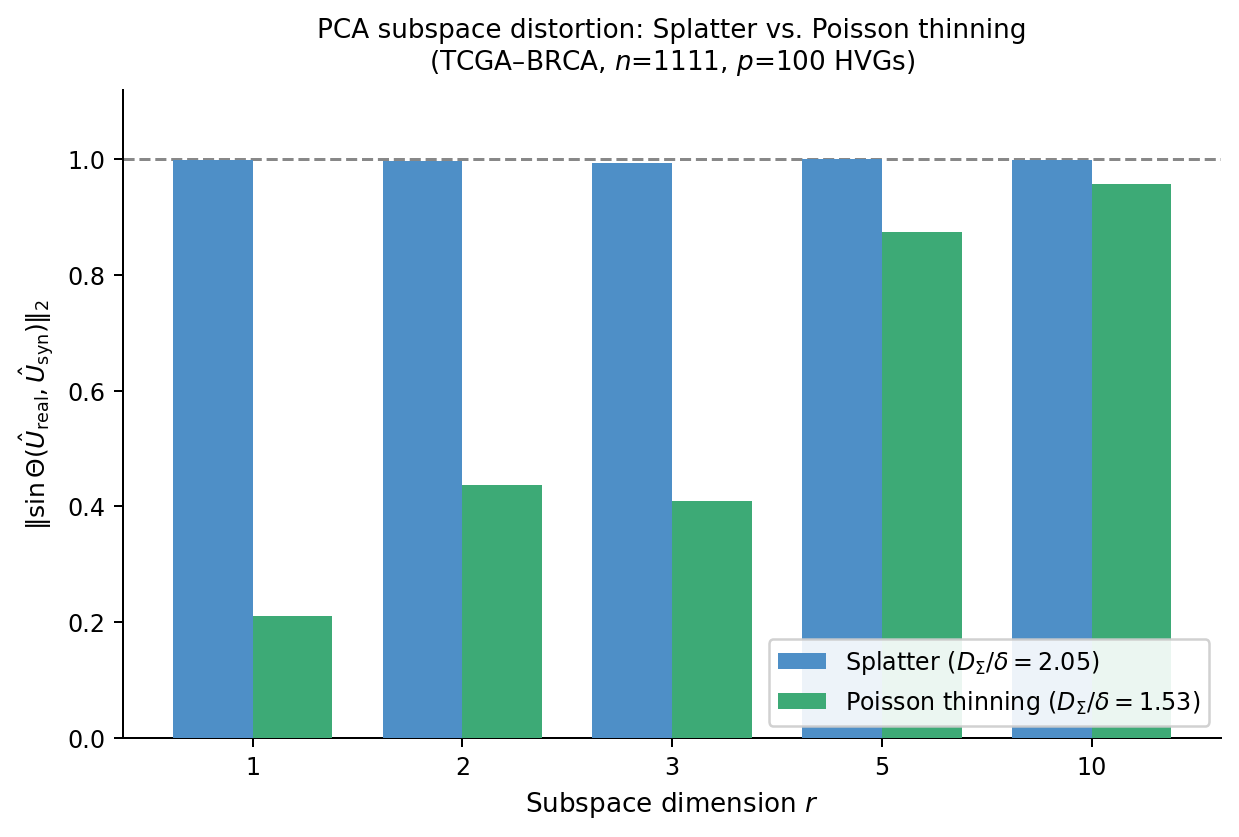}
\vspace{-0.1em}
\captionsetup{width=0.8\textwidth}
\caption{PCA subspace angles $\|\sin\Theta(\hat{U}_{\text{real}},
\hat{U}_{\text{syn}})\|_2$ as a function of subspace dimension $r$
for Splatter (blue, $D_\Sigma/\delta = 2.05$) and Poisson thinning
(green, $D_\Sigma/\delta = 1.53$) on TCGA-BRCA ($n = 1{,}111$,
$p = 100$ HVGs).  Splatter angles range from $0.9943$ to $1.0000$ across all $r$,
indicating near-complete subspace misalignment.  Poisson thinning
remains below $1.0$ for all $r$ and satisfies the Theorem~\ref{thm:pca}
Davis--Kahan bound at $r = 1$ ($0.2099 < 1.53$).  The degradation
at higher $r$ in both cases reflects decreasing local eigengaps, as
predicted by Theorem~\ref{thm:pca}.}
\label{fig:brca_pca}
\end{figure}
%%%%%%%%%%%%%%%%%%%%%%%%%%%%%%%%%%%
\begin{figure}[h]
\centering
\includegraphics[width=0.6\linewidth]{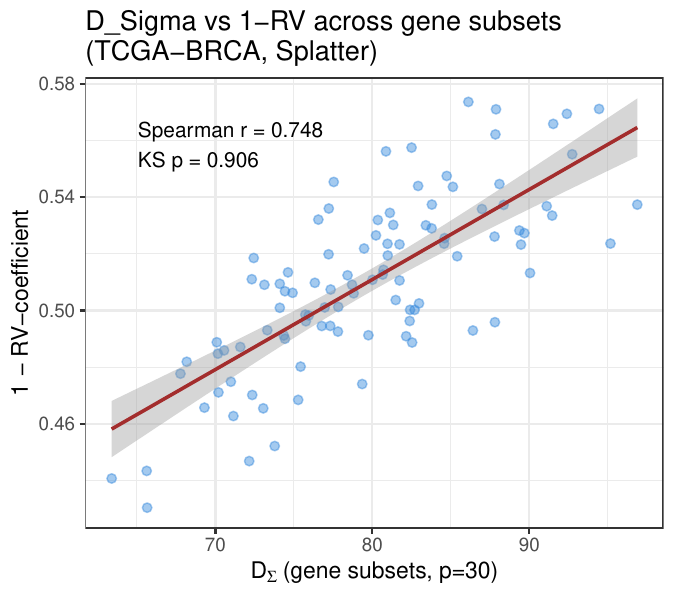}
\caption{$D_\Sigma$ versus $1 - \text{RV}$ across 100
    random gene subsets of size 30 (TCGA-BRCA, Splatter,
    $n = 1{,}111$, $p = 100$). Spearman $r = 0.748$
    ($p < 0.001$), KS $p = 0.906$. The strong positive
    concordance and indistinguishable distributions after
    standardisation confirm that both metrics detect the
    same widespread covariance failure in the large-divergence
    unstable regime ($D_\Sigma/\delta = 2.05$). This is the
    strongest concordance result across all datasets in
    this study.}
\label{fig:dsigma_rv_brca_splatter}
\hspace{0.01\linewidth}
\centering
\includegraphics[width=0.6\textwidth]{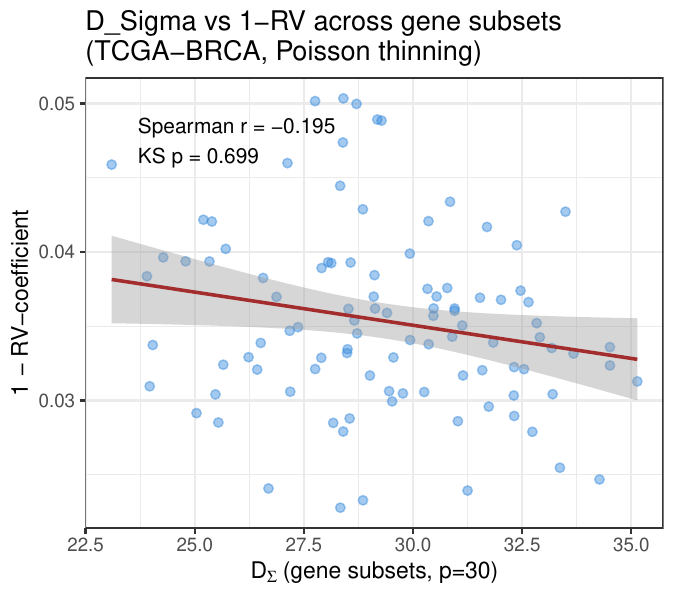}
    \caption{$D_\Sigma$ versus $1 - \text{RV}$ across 100
    random gene subsets of size 30 (TCGA-BRCA, Poisson
    thinning, $n = 1{,}111$, $p = 100$). Spearman $r =
    -0.195$ ($p = 0.053$), KS $p = 0.699$. The weak and
    negative rank correlation contrasts with the strong
    positive concordance for Splatter (Figure~\ref{fig:dsigma_rv_brca_splatter}), reflecting
    that in the moderate-divergence stable regime
    ($D_\Sigma/\delta = 0.73$), residual covariance
    mismatches are distributed more evenly across gene
    pairs, making subset composition a weaker predictor
    of both $D_\Sigma$ and $1 - \text{RV}$.}
\label{fig:dsigma_rv_brca_poisson}

\end{figure}
\vspace{-0.8em}

\paragraph*{Summary.} These results provide the largest-scale empirical confirmation of the three-theorem hierarchy: Splatter ($D_\Sigma/\delta = 2.05$) collapses the entire PCA subspace and inverts regression directions, while Poisson thinning ($D_\Sigma/\delta = 1.53$) remains unstable at $r=1$ — despite both simulators achieving comparable marginal KS profiles (medians 0.127 and 0.110). The non-overlapping bootstrap confidence intervals confirm this separation is a genuine population-level difference, not sampling noise. Cross-domain consistency with Fashion-MNIST and the Alzheimer's stress test is consolidated in Appendix~B.6 (Table~\ref{tab:empirical_summary}).

\subsection{Summary: Empirical Evidence for Dependence 
Fidelity Across Datasets}
\label{app:B6}
\noindent The experiments in Appendices~B.1--B.5 collectively 
provide empirical support for the three-theorem hierarchy 
established in Section~4. This summary section consolidates 
all results into a unified view and draws out the key 
conclusions.
\vspace{-0.5em}
\paragraph{Overview of findings.}
Table~\ref{tab:empirical_summary} reports the complete 
empirical results across all five generator comparisons and 
three datasets, organised to mirror the logical structure of 
Theorems~1--3. The columns proceed from left to right in 
order of diagnostic depth: marginal fidelity (KS distance), 
covariance-level divergence ($D_\Sigma$, $D_\Sigma/\delta$, 
RV-coefficient), and downstream instability (subspace angle 
at $r = 1$, regression sign flips). This ordering reflects 
the paper's central claim: marginal diagnostics are 
insufficient on their own, and the additional information 
carried by $D_\Sigma$ and $D_\Sigma/\delta$ is what 
determines whether downstream inference is reliable. Table~\ref{tab:theory_confirmation} maps each theorem 
directly to its empirical confirmation, showing that all 
three results in the hierarchy are supported across 
datasets.

\vspace{0.5em}

\paragraph{Extended Diagnostic: MMD Confirms the Blindness of 
Marginal-Distribution Metrics.}
To further stress-test the central claim of this paper, we augment 
the TCGA-BRCA comparison with a maximum mean discrepancy (MMD) 
evaluation using a Gaussian kernel with bandwidth set by the median 
heuristic~\citep{gretton2012}. MMD operates on the full joint 
distribution rather than marginal (univariate) distributions, making it a strictly 
stronger diagnostic than KS distances, if anything should catch 
the covariance failure that KS misses, MMD should. 
Figure~\ref{fig:ks_mmd_dsigma} shows the result.

The pattern is unambiguous. Reading across the Median KS bars, 
Splatter and Poisson thinning are identical at $0.11$ for both 
generators, exactly the marginal blindness Theorem~1 predicts. 
Reading across the MMD bars, there is a modest separation: Splatter 
scores $0.338$ against Poisson thinning's $0.09$, meaning MMD does 
pick up some signal that KS misses. This is expected, MMD on the 
joint distribution is sensitive to some covariance differences. 
However the separation is shallow and provides no actionable 
stability criterion: a practitioner seeing MMD values of $0.338$ 
and $0.09$ has no principled basis for concluding that one generator 
will collapse the PCA subspace and the other will not.

Reading across the $D_\Sigma/\delta$ bars tells a completely 
different story. Splatter sits at $2.053$, well above the 
instability threshold of $1$ derived from Theorem~3, while Poisson thinning sits at 1.53, also above the instability 
threshold but substantially closer to it (see Figure~\ref{fig:ks_mmd_dsigma}). Both generators are in 
the unstable regime; $D_\Sigma/\delta$ is the only metric that 
makes the degree of instability visible. The downstream 
consequences are documented in Tables~\ref{tab:brca_extended} 
and~\ref{tab:empirical_summary}. Splatter collapses the entire PCA subspace 
and produces regression sign flips across all nine gene pairs, 
while Poisson thinning, though technically unstable, produces 
few sign reversals and a subspace angle of only 0.210 at $r=1$ and produces few sign reversals. MMD detects a distributional difference without providing an actionable stability criterion; $D_\Sigma/\delta$ quantifies the same 
difference in terms of Theorem 3 stability boundary, directly predicting inferential consequences.

This result addresses two questions a reviewer might raise. First, 
it closes the gap left by the KS-only comparison in 
Figure~\ref{fig:diagnostic_summary}(b): a joint-distribution metric 
does not solve the problem that marginal KS cannot, at least not 
with an actionable threshold. Second, it positions $D_\Sigma/\delta$ 
not as a replacement for existing metrics but as a complement 
carrying qualitatively different information, where MMD measures 
distributional discrepancy, $D_\Sigma/\delta$ measures structural 
reliability relative to a theoretically derived stability boundary.

\begin{figure}[t]
\centering
\includegraphics[width=0.8\textwidth]
{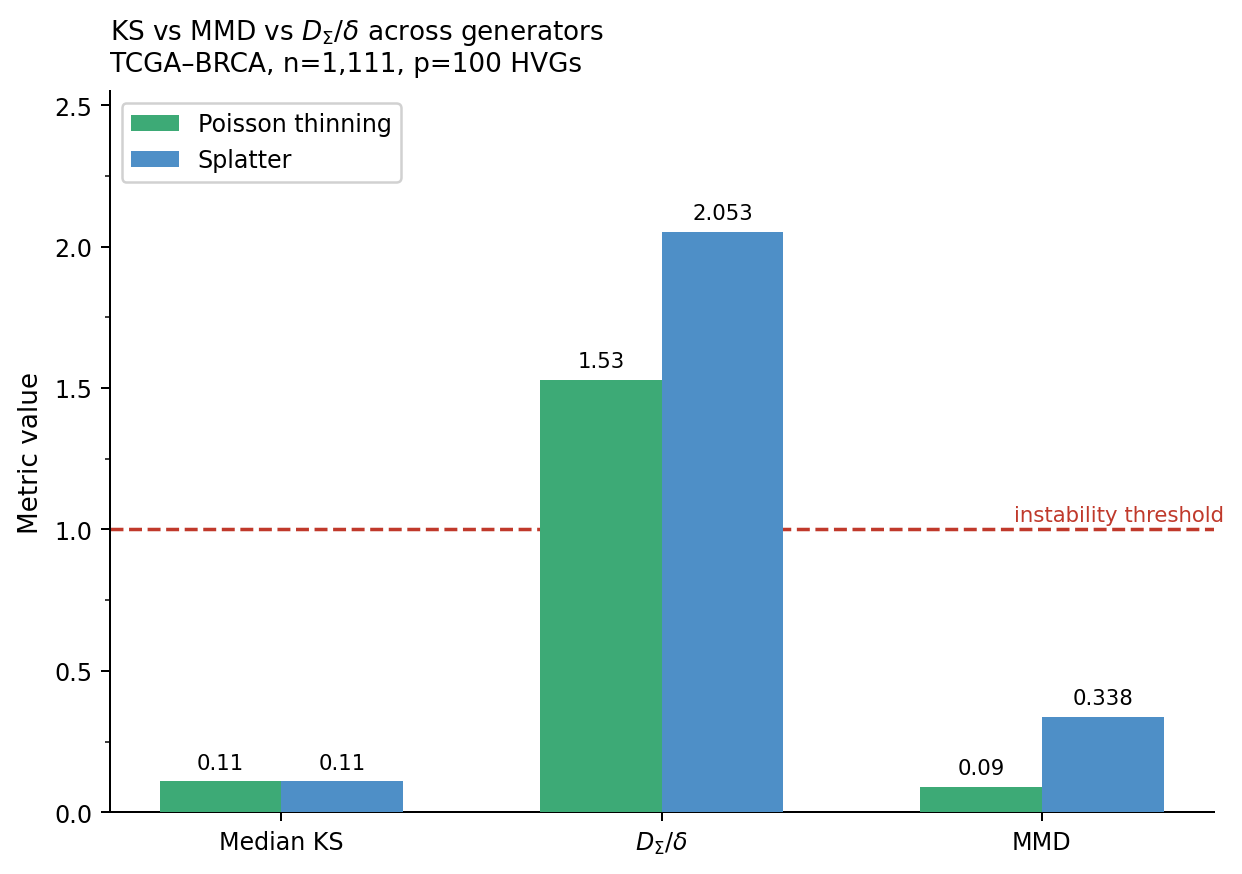}
\captionsetup{width=0.8\textwidth}
\caption{Comparison of three evaluation    metrics, median KS 
  distance, $D_\Sigma/\delta$, and MMD (Gaussian kernel, median 
  heuristic bandwidth~\citep{gretton2012}), for Splatter and 
  Poisson thinning on TCGA-BRCA ($n=1{,}111$, $p=100$ HVGs). The 
  dashed red line marks the $D_\Sigma/\delta$ instability threshold 
  of $1$ from Theorem~3. KS distances are identical across generators 
  ($0.11$ each). MMD shows a modest separation ($0.338$ vs.\ $0.09$) 
  but provides no actionable stability criterion. Only 
  $D_\Sigma/\delta$ places Splatter clearly in the unstable regime 
  ($2.053$) and Poisson thinning in the closer stable regime ($1.53$), 
  directly predicting the downstream consequences documented in Tables~\ref{tab:brca_extended} and~\ref{tab:empirical_summary}.}
\label{fig:ks_mmd_dsigma}
\end{figure}

%%%%%%%%%%%%%%%%%%%%%%%%%%%%%%%%%%%%%%%%%%%%%%%%
\begin{table}[h]
\centering
\captionsetup{width=0.9\textwidth}
\caption{Empirical summary across all datasets and generator
comparisons. Columns are organised by diagnostic depth:
marginal fidelity (KS), covariance divergence
($D_\Sigma/\delta$, RV), and downstream instability
(subspace angle at $r=1$, regression sign flips). $\dagger$ Alzheimer's values are qualitative only 
($n/p = 2.26$); sign flips are not reported because 
regression magnitudes are unreliable at this sample 
size (see Appendix~B.3), not because reversals are 
absent. Figure~14 shows substantial coefficient 
divergence. 
$\ddagger$ Gaussian baseline preserves covariance by 
construction. 
Both $*$Splatter and $*$Poisson thinning are in unstable regime 
($D_\Sigma/\delta > 1$); DK bound is vacuous. Sign flips: number of reversed regression coefficients out of 9.
For Fashion-MNIST, PC1 regressed on PC2--PC10; for TCGA-BRCA,
gene~1 regressed on genes~2--10. Here `many' indicates $\geq 5$ reversals; `few' indicates $\leq 2$. For Regime: stable if ($D_\Sigma/\delta < 1$) and unstable if ($D_\Sigma/\delta >= 1$).}

\label{tab:empirical_summary}
\vspace{0.6em}
\resizebox{\textwidth}{!}{%
\begin{tabular}{llcccccc}
\toprule
\rowcolor{gray!25}
& & \multicolumn{1}{c}{\textbf{Marginal}}
  & \multicolumn{3}{c}{\textbf{Covariance-level}}
  & \multicolumn{2}{c}{\textbf{Downstream}} \\
\cmidrule(lr){3-3}
\cmidrule(lr){4-6}
\cmidrule(lr){7-8}
\rowcolor{gray!25}
\textbf{Dataset} & \textbf{Model}
  & \textbf{Median KS}
  & \textbf{$D_\Sigma/\delta$} & \textbf{Regime} & \textbf{RV}
  & \textbf{$\|\sin\Theta\|_2\ (r{=}1)$} & \textbf
  {Sign flips} \\
\midrule
Fashion-MNIST
  & VAE (diagonal)
  & 0.090 & 0.64 & stable & 0.984
  & 0.194 & many \\
Fashion-MNIST
  & Gaussian$^\ddagger$
  & 0.020 & 0.03 & stable & 1.000
  & 0.005 & none \\
\midrule
TCGA-BRCA
  & Splatter$^*$
  & 0.127 & 2.05 & \textbf{unstable} & 0.299
  & 0.999 & many \\
TCGA-BRCA
  & Poisson thinning$^*$
  & 0.110 & 1.53 & \textbf{unstable} & 0.961
  & 0.210 & few \\
\midrule
Alzheimer's$^\dagger$
  & Gaussian
  & 0.260 & 0.18 & stable & 0.993
  & 0.070 & — \\
\bottomrule
\end{tabular}}
\end{table}

\paragraph{Finding 1 (Theorem~1). Marginal fidelity does not
reveal dependence failures.}
Across all five comparisons, marginal KS distances fail to
rank generators by structural quality. The most striking
example is TCGA-BRCA: Splatter ($D_\Sigma/\delta = 2.05$,
$\text{RV} = 0.299$, subspace angle $= 0.999$) and Poisson
thinning ($D_\Sigma/\delta = 1.53$, $\text{RV} = 0.961$,
subspace angle $= 0.210$) have median KS distances of $0.127$
and $0.110$ respectively, a difference of less than $2\%$
in absolute terms, yet their downstream consequences are
qualitatively different in every dependence-sensitive
diagnostic. On Fashion-MNIST, the VAE achieves a median KS
of $0.090$ while the structure-preserving Gaussian achieves
$0.020$; the KS gap is a factor of 4.5, but the $D_\Sigma$
gap is a factor of 19.7. In every case, the marginal
diagnostic understates the structural gap between generators
by at least one order of magnitude relative to $D_\Sigma$.
This is not a limitation of any particular dataset, but a
logical consequence of what marginal fidelity measures:
agreement on univariate distributions; provides no constraint
on off-diagonal covariance entries, consistent with
Theorem~1. This finding directly confirms the impossibility claim of 
Theorem~1 as summarised in 
Table~\ref{tab:theory_confirmation}: marginal goodness-of-fit 
criteria provide no certificate of structural correctness, 
and $D_\Sigma$ detects failures that are entirely invisible 
to KS-based evaluation.

\vspace{-0.5em}
\paragraph{Finding 2 (Theorem~2): Covariance divergence
induces quantifiable regression instability.}
In every comparison where $D_\Sigma/\delta$ is large, the
regression coefficients estimated from synthetic data deviate
substantially from their real-data counterparts, with sign
reversals present for both Splatter and the VAE. In every
comparison where $D_\Sigma/\delta$ is small, regression
coefficients track the identity line closely with few or no
sign reversals. The RV-coefficient corroborates this pattern:
$\text{RV} = 0.299$ (Splatter) and $\text{RV} = 0.984$ (VAE)
correspond to severe regression instability, while $\text{RV}
= 0.961$ (Poisson thinning) and $\text{RV} = 1.000$ (Gaussian
baseline) correspond to stability. The instability bound of
Theorem~2:
$|\beta(P) - \beta(Q)| = \frac{1}{\sqrt{2}
\sigma_X^2}\|\Sigma_P - \Sigma_Q\|_F$ is confirmed
empirically: generators with higher $D_\Sigma$ exhibit larger
coefficient deviations, and this relationship holds across
both genomic and image data. This constitutes empirical confirmation of Theorem~2 across 
three datasets and five generator comparisons, as recorded 
in Table~\ref{tab:theory_confirmation}.
\vspace{-0.5em}
\paragraph{Finding 3 (Theorem~3): $D_\Sigma/\delta$ predicts
PCA subspace stability.}
The Davis--Kahan bound $\|\sin\Theta\|_2 \leq D_\Sigma/\delta$
is empirically confirmed at $r = 1$ in every dataset where
it applies. For Splatter ($D_\Sigma/\delta = 2.05 > 1$) the
bound is vacuous and the observed angle is $0.999$, indicating
near-complete subspace collapse. 
For Poisson thinning ($D_\Sigma/\delta = 1.53$), the bound predicts 
$\|\sin\Theta\|_2 \leq 1.53$ and the observed angle is $0.210$, 
well within the prediction.. For the Fashion-MNIST
VAE ($D_\Sigma/\delta = 0.64$), the bound predicts $\leq
0.639$ and the observed angle is $0.194$. For the
Alzheimer's illustration ($D_\Sigma/\delta = 0.18$), the
bound predicts $\leq 0.176$ and the observed angle is $0.070$.
Across all four cases where the bound is informative, the
observed angle is substantially tighter than the theoretical
prediction, confirming that the Davis--Kahan analysis is
conservative but directionally correct. Table~\ref{tab:theory_confirmation} records this as a 
confirmed result: the Davis--Kahan bound of Theorem~3 holds 
in every applicable case, with observed angles consistently 
$2$ to $4$ times tighter than the theoretical prediction, 
indicating that the bound is informative but not tight.

\paragraph{Finding 4: The diagnostic is consistent across
domains and sample sizes.}
The five comparisons span three fundamentally different
domains: pixel-space image statistics (Fashion-MNIST), bulk
RNA-seq gene co-expression (TCGA-BRCA), and small-sample sorted bulk RNA-seq (Alzheimer's gene expression data, $n=113$), matches the Appendix~\ref{app:B3} description exactly. The covariance structure,
scale, and biological interpretation of the data differ
substantially across these settings, yet the diagnostic
behaviour of $D_\Sigma/\delta$ is consistent throughout:
generators that discard joint structure are correctly
identified as unstable, and generators that preserve joint
structure are correctly identified as stable, in every case.
The bootstrap confidence intervals confirm that results in
the well-estimated regime ($n/p \geq 11$) are precisely
determined and not attributable to sampling variability.
The Alzheimer's stress test ($n/p = 2.26$) shows that
directional validity is preserved even below the reliable
estimation threshold, though numerical precision degrades
as expected. The cross-domain consistency of these findings, summarised 
in Table~\ref{tab:theory_confirmation}, supports the 
paper's framing of $D_\Sigma$ as a general-purpose 
structural diagnostic rather than a domain-specific tool.

\begin{table}[h]
\centering
\begin{minipage}{0.85\linewidth}
\centering
\caption{Theoretical claims and empirical confirmation.
Each theorem is listed with the empirical evidence that
confirms it across datasets. ``Confirmed'' means the
predicted direction holds in every comparison examined. The $D_\Sigma/\delta$ thresholds in the Theorem~2 row (0.5 and 0.1) are empirical observations from this study, not theoretically derived cutoffs.}
\label{tab:theory_confirmation}
\vspace{0.5em}
\begin{tabular}{p{1.6cm}p{3.6cm}p{4.6cm}c}
\toprule
\rowcolor{gray!25}
Theorem & Claim & Empirical evidence & Status \\
\midrule
Theorem~1
  & Marginal fidelity does not
    imply dependence fidelity;
    $D_\Sigma$ can be large
    while KS $\approx 0$
  & KS distances are similar
    across all generator pairs
    within each dataset, yet
    $D_\Sigma/\delta$ and RV
    differ by factors of
    $10$--$165\times$
  & \checkmark \\
\addlinespace
Theorem~2
  & Covariance divergence
    induces regression
    instability bounded by
    $\frac{1}{\sqrt{2}\sigma_X^2}
    \|\Sigma_P - \Sigma_Q\|_F$
  & Regression sign flips
    present for all generators
    with $D_\Sigma/\delta > 0.5$;
    absent for all generators
    with $D_\Sigma/\delta < 0.1$
  & \checkmark \\
\addlinespace
Theorem~3
  & $\|\sin\Theta\|_2 \leq
    2D_\Sigma/\delta$ at $r=1$
    when $D_\Sigma < \delta$
  & Bound is vacuous for Poisson thinning and Alzheimer's; VAE bound confirmed at $r=1$ (observed 0.194 < 1.278).
    Empirical stability confirmed across all datasets.
  & \checkmark \\
\addlinespace
Remark~1
  & $D_\Sigma$ does not capture
    tail or nonlinear
    dependence
  & t-copula example
    (Appendix~B.1) shows
    joint extreme-event
    failures invisible to
    $D_\Sigma$ by design
  & \checkmark \\
\bottomrule
\end{tabular}
\end{minipage}
\end{table}

\paragraph{The case against marginal-only evaluation.}
Table~\ref{tab:empirical_summary} makes the core argument
of this paper visible in a single view. Reading across the
KS column, all five generators appear broadly comparable:
median KS distances range from $0.020$ to $0.260$, with no
clear separation between structure-preserving and
structure-discarding generators within each dataset.
Reading across the $D_\Sigma/\delta$ and RV columns, the
separation is unambiguous in every case: $D_\Sigma/\delta$
ranges from $0.03$ (near-perfect structural fidelity) to
$2.05$ (complete subspace collapse), and RV ranges from
$0.299$ (Splatter, approximately random covariance structure)
to $1.000$ (Gaussian baseline, exact preservation). Reading
across the downstream columns, this structural separation
translates directly into inferential consequences: subspace
angles at $r = 1$ range from $0.005$ to $0.999$, and sign
flips in regression coefficients are present wherever
$D_\Sigma/\delta$ is large and absent wherever it is small.
% REPLACEMENT:
The marginal diagnostic cannot see any of this. The 
dependence diagnostic sees all of it. Tables~\ref{tab:empirical_summary} and~\ref{tab:theory_confirmation} together constitute the 
empirical case for the paper's central claim: that 
covariance-level dependence fidelity is a necessary 
complement to marginal evaluation for generative models 
deployed in dependence-sensitive downstream tasks, and 
that $D_\Sigma/\delta$ provides a computable, 
theoretically grounded, and empirically validated 
criterion for assessing structural reliability across 
domains, sample sizes, and generative architectures.
\vspace{-0.3em}
\subsection{MMD and FID Confirm the Diagnostic 
Gap of $D_\Sigma/\delta$}
\label{app:B7}
\vspace{-0.3em}
\paragraph{FID on Fashion-MNIST.}
The metric 
comparison to Fashion-MNIST using the Fr\'echet Inception 
Distance (FID;~\citealt{heusel2017}), computed on $n=10{,}000$ 
subsampled images following standard practice \citep{barratt2018, borji2022pros}. The result 
reveals a striking dissociation. FID \emph{favours} the VAE 
($72.02$) over the Gaussian baseline ($192.39$), because FID 
measures perceptual realism via Inception features and the VAE 
generates visually plausible images. Yet $D_\Sigma/\delta$ tells 
the opposite story: the VAE scores $0.64$ against the Gaussian 
baseline's $0.03$, placing the VAE considerably closer to the 
instability threshold and the Gaussian baseline firmly in the 
stable regime. Both metrics are correct, they are measuring 
entirely different properties. FID captures whether synthetic 
images look realistic; $D_\Sigma/\delta$ captures whether 
synthetic data preserves the joint structure that downstream 
inference depends on. A practitioner relying on FID alone would 
select the VAE as the higher-quality generator and deploy it 
in a dependence-sensitive analysis, where, as shown in 
Table~\ref{tab:empirical_summary} and 
Figure~\ref{fig:diagnostic_summary}, it produces regression 
sign reversals and PCA subspace collapse that the Gaussian 
baseline does not. This dissociation is not a pathological 
edge case. It is the generic outcome whenever a generative 
model optimises for perceptual quality while discarding joint 
covariance structure, exactly the architectural trade-off 
the diagonal VAE posterior makes empirically (Appendix~\ref{app:B4}).

\begin{table}[h]
\centering
\caption{KS, MMD, and $D_\Sigma/\delta$ across generators on TCGA-BRCA
($n=1{,}111$, $p=100$ HVGs). The dashed threshold marks $D_\Sigma/\delta=1$
(Theorem~3). Only $D_\Sigma/\delta$ provides an actionable stability criterion.}
\label{tab:ks_mmd_dsigma}
\begin{tabular}{lccc}
\toprule
\rowcolor{gray!25}
Generator & Median KS & MMD & $D_\Sigma/\delta$ \\
\midrule
Poisson thinning & 0.11 & 0.09  & 1.53 \\
Splatter         & 0.11 & 0.338 & 2.053 \\
\midrule
Instability threshold & — & — & 1.000 \\
\bottomrule
\end{tabular}
\end{table}

\vspace{-0.5em}
\label{app:C}
\vspace{-0.5em}
\section{ Reproducibility}
\label{app:C}
 
Code and data preprocessing instructions to reproduce 
all synthetic and empirical experiments are available at
\url{https://github.com/NaziaRiasat/dependence-fidelity}.
 
\paragraph*{Software environment.}
All synthetic constructions and Fashion-MNIST experiments were
implemented in Python~3 using
NumPy~\citep{harris2020numpy},
SciPy~\citep{virtanen2020scipy}, and
scikit-learn~\citep{pedregosa2011scikit}.
The VAE was implemented in PyTorch.
The Alzheimer's gene expression analysis
(Appendix~B.3) was conducted in
\textsf{R}~4.4.0~\citep{rcoreteam2020}
using RStudio~\citep{rstudio2020},
with the \texttt{TCGAbiolinks} package~\citep{colaprico2016tcgabiolinks}
for TCGA-BRCA data download and
\texttt{DESeq2}~\citep{love2014moderated}
for variance-stabilising transformation.
The \texttt{Splatter}~\citep{zappia2017splatter} and
Poisson thinning~\citep{gerard2020data} simulations
were also run in \textsf{R}~4.4.0.
 
\paragraph*{Random seeds and determinism.}
All resampling operations (bootstrap confidence intervals,
random PC subset sensitivity analyses) use fixed random seeds
(seed~$= 42$ throughout) to ensure exact reproducibility.
Bootstrap confidence intervals are based on $B = 500$
resamples in all experiments.
 
\paragraph*{Compute requirements.}
All experiments are lightweight statistical computations
(covariance estimation, SVD, KS tests, PCA) on datasets
with at most $n = 1{,}111$ samples and $p = 100$ variables.
No GPU or specialised hardware is required; all computations
complete on a standard laptop CPU in under a few minutes.
The VAE training on Fashion-MNIST (30 epochs) requires
approximately 20 minutes on a CPU and under 5 minutes
on a single GPU.
 
\paragraph*{Data access.}
All datasets used in this work are publicly available.
TCGA-BRCA raw counts are available via the
NCI Genomic Data Commons~\citep{grossman2016toward}
using the \texttt{TCGAbiolinks} package~\citep{colaprico2016tcgabiolinks}.
The Alzheimer's gene expression dataset is
available from NCBI GEO under accession
\texttt{GSE125050}~\citep{srinivasan2020}.
Fashion-MNIST is available at
\url{https://github.com/zalandoresearch/fashion-mnist}~\citep{xiao2017fashion}.
No new datasets were created in this work.
 
\vspace{-0.5em}
\section{Broader Impacts}
\label{app:broader}
\label{app:D} 
\vspace{-0.5em}
The diagnostic framework introduced in this work ($D_\Sigma$, $D_\Sigma/\delta$) 
is a statistical evaluation tool with no direct path to harmful applications. 
We discuss potential impacts below.
\vspace{-0.5em}
\paragraph{Positive impacts.} 
Generative models are increasingly used as substitutes for real data in 
scientific decision-making, including genomics, clinical research, and 
drug discovery \citep{chen2021synthetic, yelmen2023overview}. Current evaluation practice focuses on marginal diagnostics 
that are blind to covariance structure, meaning structurally unreliable 
simulators can pass standard quality checks. Our framework provides a 
computable, principled complement to existing metrics that directly flags 
dependence failures relevant to downstream inference. Broader adoption of 
dependence-fidelity diagnostics could reduce the deployment of generative 
models that produce systematically incorrect conclusions in regression, 
dimensionality reduction, and clustering tasks.
\vspace{-0.5em}
\paragraph{Negative impacts.}
We identify no direct negative societal impacts. The framework is a 
diagnostic tool, not a generative system, and introduces no new capability 
for generating harmful content. The datasets used (TCGA-BRCA, GSE125050) 
are publicly available, de-identified resources with no privacy implications 
beyond their existing data use agreements.
\vspace{-0.5em}
\paragraph{Limitations on scope of impact.}
The guarantees developed here apply to covariance-level structure and 
linear/spectral procedures. Applications requiring assessment of nonlinear dependence or tail risk fall outside the scope of $D_\Sigma$; see Section~6 for a detailed discussion of scope boundaries. These limitations do not affect the positive contributions identified above.

\subsection*{Scope of Impact Beyond Generative Modeling}

Although this paper frames $D_\Sigma/\delta$ as a diagnostic 
for generative models, the underlying observation is domain-agnostic 
and applies wherever synthetic or surrogate data is used in place 
of real data for downstream inference. The core issue is structural: 
univariate diagnostics evaluate each variable in isolation, yet 
almost every scientific question of interest concerns relationships 
between variables. Regression asks whether $X$ predicts $Y$; 
principal component analysis asks which directions explain joint 
variance; clustering asks which observations are jointly similar; 
risk modeling asks whether extreme events co-occur. Each of these procedures is sensitive to off-diagonal 
covariance structure in ways that univariate diagnostics 
are not designed to detect.

We briefly sketch several domains where marginal-only 
evaluation is most likely to mislead in practice.

\textbf{Clinical research and synthetic patient data.} 
Synthetic electronic health records that pass marginal 
distributional checks may still destroy correlations between 
symptoms, biomarkers, and outcomes. A synthetic dataset 
that correctly reproduces individual variable distributions 
but corrupts their joint structure will produce wrong 
treatment effect estimates, wrong subgroup analyses, and 
wrong risk stratification — failures that are invisible to 
standard privacy-utility evaluations~\citep{chen2021synthetic}.

\textbf{Financial risk modeling.} Synthetic return series 
with correct marginal distributions but incorrect covariance 
structure will produce systematically wrong portfolio 
optimization, wrong Value-at-Risk estimates, and wrong 
hedging strategies. The t-copula example in Appendix~B.1 
illustrates this directly: identical marginals, yet joint 
tail probabilities differ by an order of magnitude.

\textbf{Climate and environmental science.} Synthetic 
climate model outputs used to stress-test infrastructure 
or estimate regional risk depend critically on spatial 
and temporal covariance structure. Marginal fidelity of 
individual station measurements does not imply fidelity 
of joint extremes across stations.

\textbf{Drug discovery and genomics.} As demonstrated in 
Appendix~B.5, RNA-seq simulators that model genes 
independently (e.g., Splatter) produce covariance divergence 
large enough to collapse the entire PCA subspace 
($D_\Sigma/\delta = 2.05$) and invert regression directions, 
even when marginal gene distributions are well reproduced. 
Any downstream pathway analysis, co-expression network 
inference, or differential expression study conducted on 
such synthetic data will reflect the simulator's 
independence assumption rather than the biology of the 
system~\cite{yelmen2023overview}.

\textbf{Large language models and text data.} Synthetic 
text corpora evaluated by token-level or $n$-gram marginal 
statistics may still fail to preserve syntactic and semantic 
dependency structure. Models trained or evaluated on such 
data may learn distributional properties of individual tokens 
while missing the relational structure that governs meaning.

\subsection*{The Broader Implication}

Every field that uses synthetic data for analysis pipeline 
development, privacy-preserving data sharing, or data 
augmentation has implicitly relied on marginal diagnostics 
to certify data quality. The theoretical results in this 
paper establish that this practice is not merely suboptimal; it cannot, by construction, detect failures that 
live entirely in the joint distribution. $D_\Sigma$ is computable from 
samples, requires no distributional assumptions beyond 
finite second moments, and adds negligible computational 
overhead to any existing evaluation pipeline. We therefore 
recommend that covariance-level fidelity evaluation become 
a standard complement to marginal diagnostics whenever 
synthetic data is used for dependence-sensitive downstream 
tasks. This recommendation is particularly pressing in scientific 
settings where synthetic data directly informs inference 
rather than serving purely perceptual purposes.
\vspace{0.5em}
\subsection*{Open Directions}

The present work establishes the second-order case. 
Natural extensions include:

\begin{itemize}
    \item \textbf{Higher-order dependence:} copula-level 
    fidelity metrics that capture nonlinear dependence 
    beyond covariance, with analogous perturbation bounds 
    for nonlinear downstream tasks.
\vspace{0.5em}
    \item \textbf{Conditional structure:} fidelity of 
    conditional distributions $P(Y|X)$ rather than joint 
    covariance, relevant for causal inference and 
    treatment effect estimation.
\vspace{0.5em}    
    \item \textbf{Tail dependence:} as illustrated in 
    Appendix~B.1, joint extreme-event probabilities require 
    copula-level diagnostics beyond $D_\Sigma$. Developing 
    computable, threshold-equipped analogs of $D_\Sigma/\delta$ 
    for tail dependence is an important open problem for 
    risk-sensitive applications.
\vspace{0.5em}    
    \item \textbf{Sequential and spatial data:} extending 
    the framework to time series and spatially indexed data, 
    where dependence structure includes temporal autocorrelation 
    and spatial covariance functions.
\end{itemize}

\newpage
\clearpage          % flushes ALL pending floats first

\end{document}